\documentclass[lettersize,journal]{IEEEtran}
\usepackage[ruled]{algorithm2e}
\usepackage{booktabs}
\usepackage{array}
\usepackage{stfloats}
\usepackage{multirow}
\usepackage{graphicx}
\usepackage{subfigure}
\usepackage{bm}
\usepackage{enumitem}
\setlist[enumerate]{leftmargin=1em}
\usepackage{amsmath}
\usepackage{todonotes}
\usepackage{setspace}
\usepackage{multirow}
\usepackage{multicol}
\usepackage{footmisc}
\usepackage[hidelinks]{hyperref}
\usepackage{color}

\usepackage{ifpdf}
\ifpdf
  \usepackage{epstopdf}
\fi

\usepackage{amsthm}
\usepackage{amssymb}
\usepackage{amsmath}
\usepackage{cite}
\usepackage[T1]{fontenc} %%show bf

\newcommand{\mfk}{\mathfrak}

\newcommand{\QC}{\mathcal{Q}}
\newcommand{\YC}{\mathcal{Y}}
\newcommand{\XC}{\mathcal{X}}
\newcommand{\SC}{\mathcal{S}}
\newcommand{\LC}{\mathcal{L}}
\newcommand{\RC}{\mathcal{R}}
\newcommand{\UC}{\mathcal{U}}
\newcommand{\VC}{\mathcal{V}}

\newcommand{\AC}{\mathcal{A}}
\newcommand{\BC}{\mathcal{B}}
\newcommand{\CC}{\mathcal{C}} 
\newcommand{\IC}{{\mathcal{I}}}
\newcommand{\HC}{{\mathcal{H}}}

\newcommand{\SCigma}{{\mathit{\Sigma}}}

\newcommand{\Am}{{\mathbf{A}}}
\newcommand{\Bm}{{\mathbf{B}}}
\newcommand{\Cm}{{\mathbf{C}}}

\newcommand{\sizeC}{\in\mathbb{C}}
\newcommand{\size}{\in\mathbb{R}}

\newcommand{\norm}[1]{\lVert#1\rVert}

\DeclareMathOperator{\fro}{\mathsf{F}}
\DeclareMathOperator{\op}{\mathsf{op}}

\DeclareMathOperator*{\tr}{tr}
\DeclareMathOperator*{\GL}{GL}

\DeclareMathOperator*{\argmin}{\arg\min}

\DeclareMathOperator*{\To}{T}
\DeclareMathOperator*{\ifft}{ifft}
\DeclareMathOperator*{\fft}{fft}
\DeclareMathOperator*{\rank}{rank}

\def\dist{\mathrm{dist}}
\def\distzero{\dist(\mathcal{L}_0,\mathcal{R}_0;\mathcal{L}_\star,\mathcal{R}_\star)}
\def\distk{\dist(\mathcal{L}_k,\mathcal{R}_k;\mathcal{L}_\star,\mathcal{R}_\star)}
\def\distsquarek{\dist^2(\mathcal{L}_k,\mathcal{R}_k;\mathcal{L}_\star,\mathcal{R}_\star)}
\def\distkplusone{\dist(\mathcal{L}_{k+1},\mathcal{R}_{k+1};\mathcal{L}_\star,\mathcal{R}_\star)}
\def\distsquarekplusone{\dist^2(\mathcal{L}_{k+1},\mathcal{R}_{k+1};\mathcal{L}_\star,\mathcal{R}_\star)}
\def\supp{\mathrm{supp}}

\newtheorem{theorem}{Theorem}

\newtheorem{lemma}{\textbf{Lemma}}

\newtheorem{assumption}{\textbf{Assumption}}

\newtheorem*{proof*}{\textbf{Proof}}
\newtheorem*{remark*}{\textbf{Remark}}
\newtheorem{definition}{\textbf{Definition}}
\newtheorem*{remark}{Remark}
\let\oldremark\remark
\renewcommand{\remark}{\oldremark\normalfont}

\ifCLASSINFOpdf

\else

\fi

\begin{document}
\title{Learnable Scaled Gradient Descent for Guaranteed Robust Tensor PCA }
\author{Lanlan~Feng,
Ce~Zhu,~\IEEEmembership{Fellow,~IEEE,} %
Yipeng~Liu,~\IEEEmembership{Senior Member,~IEEE,}
Saiprasad Ravishankar,~\IEEEmembership{Senior Member,~IEEE,}
Longxiu~Huang
\thanks{This research is supported by the National Natural Science Foundation of China (NSFC) under Grant 62020106011, Grant 62171088, Grant W2412085, and Grant 62450131. Ce Zhu and Yipeng Liu are the corresponding authors.}
\thanks{Lanlan Feng (\href{mailto:lanfeng@std.uestc.edu.cn}{lanfeng@std.uestc.edu.cn}), Ce Zhu (\href{mailto:eczhu@uestc.edu.cn}{eczhu@uestc.edu.cn}) and Yipeng Liu (\href{mailto:yipengliu@uestc.edu.cn}{yipengliu@uestc.edu.cn}) are with the School of Information and Communication Engineering, University of Electronic Science and Technology of China (UESTC), Chengdu, 611731, China.}
\thanks{Saiprasad Ravishankar (\href{mailto:ravisha3@msu.edu}{ravisha3@msu.edu}) is with the Department of Computational Mathematics, Science and Engineering and the Department of Biomedical Engineering, Michigan State University (MSU), East Lansing, MI 48824, USA.}
\thanks{Longxiu Huang (\href{mailto:huangl3@msu.edu}{huangl3@msu.edu}) is with the Department of Computational Mathematics, Science and Engineering and the Department of Mathematics, MSU, East Lansing, MI 48824, USA.}
}

\markboth{ IEEE Journal, vol. xx, no. xx, Month Year}%
{Shell \MakeLowercase{\textit{et al.}}: Bare Demo of IEEEtran.cls for IEEE Journals}

\maketitle
\begin{abstract}
Robust tensor principal component analysis (RTPCA) aims to separate the low-rank and sparse components from multi-dimensional data, making it an essential technique in the signal processing and computer vision fields.  
Recently emerging tensor singular value decomposition (t-SVD) has gained considerable attention for its ability to better capture the low-rank structure of tensors compared to traditional matrix SVD.
 However, existing methods often rely on the computationally expensive tensor nuclear norm (TNN), which limits their scalability for
 real-world tensors.
To address this issue, we explore an efficient scaled gradient descent (SGD) approach within the t-SVD framework for the first time, and propose the RTPCA-SGD method.
Theoretically, we rigorously establish the recovery guarantees of RTPCA-SGD under mild assumptions, demonstrating that with appropriate parameter selection, it achieves linear convergence to the true low-rank tensor at a constant rate, independent of the condition number.
To enhance its practical applicability, we further propose a learnable self-supervised deep unfolding model, which enables effective parameter learning.
Numerical experiments on both synthetic and real-world datasets demonstrate the superior performance of the proposed methods while maintaining competitive computational efficiency, especially consuming less time than RTPCA-TNN.

\end{abstract}

\begin{IEEEkeywords}
Tensor singular value decomposition, tensor principal component analysis, scaled gradient descent, recovery guarantees, deep unfolding.
\end{IEEEkeywords}

\section{Introduction}
\IEEEPARstart{R}{obust} principal component analysis (RPCA) \cite{candes2011robust} is a foundational technique for separating a low-rank matrix and a sparse matrix from two-dimensional data, making it highly effective across a variety of applications, including surveillance\cite{sobral2017matrix}, anomaly detection\cite{shyu2003novel}, and more\cite{bouwmans2018applications}.
However, with the advancements in technology and the improvements in data acquisition, real-world data has become increasingly complex, often exhibiting high-dimensional characteristics.
Conventional matrix-based RPCA approaches often struggle to maintain the natural multidimensional relationships present in such data, leading to potential information loss.
\begin{figure}
     \centering
     \includegraphics[width=0.9\linewidth]{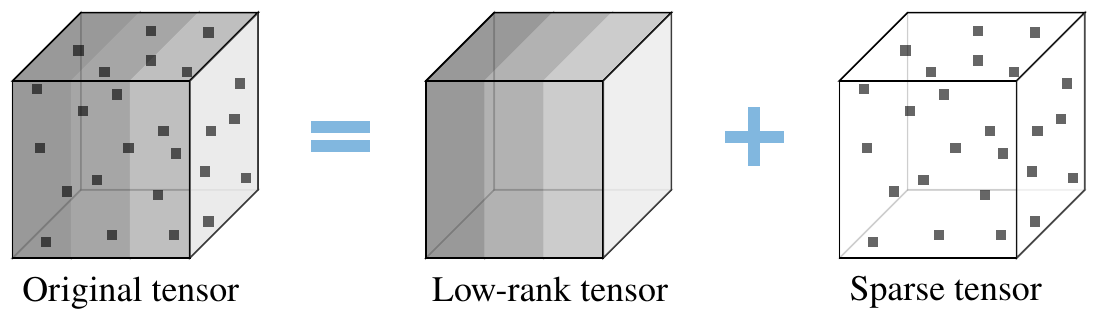}
     \vspace{-0.5em}
     \caption{Illustration of RTPCA model.}
     \label{fig: RTPCA}
     \vspace{-0.8em}
 \end{figure}
 
 Tensors\cite{kolda2009tensor}, as higher-order extensions of vectors and matrices, offer powerful data structures that capture interactions across multiple dimensions, making them essential in many fields such as data science, machine learning, and signal processing \cite{sidiropoulos2017tensor}. 
 Their multidimensional framework reveals unique properties that can be harnessed for a wide range of analytical and processing tasks such as tensor completion, tensor principal component analysis, and tensor regression \cite{liu2022tensor}. Robust tensor principal component analysis (RTPCA)\cite{lu2016tensor,zhang2014novel,lu2019tensor} extends RPCA to the tensor version, aiming to recover the low-rank tensor $\mathcal{X}_\star$ and sparse tensor $\mathcal{S}_\star$ from the observed data $\mathcal{Y}$, as shown in Fig.~\ref{fig: RTPCA}. Mathematically, it is formulated as 
 \begin{equation}
     \mathcal{Y} = \mathcal{X}_\star +\mathcal{S}_\star.
     \label{eq: rtpca0}
 \end{equation}
  
The performance of RTPCA varies significantly depending on the chosen tensor decomposition method. Unlike matrices, which typically rely on singular value decomposition (SVD), there are a variety of decomposition techniques for tensors, including canonical polyadic (CP) decomposition~\cite{harshman1970foundations,de2008tensor}, Tucker decomposition~\cite{tucker1966some,mu2014square}, tensor singular value decomposition (t-SVD)\cite{lu2016tensor, zhang2014novel,lu2019tensor}, and tensor networks (e.g., tensor tree, tensor train, tensor ring)\cite{ballani2014tree, oseledets2011tensor, zhao2016tensor}.
Among these, t-SVD has demonstrated particular effectiveness in many image processing applications due to its ability to capture low-rank structures in the Fourier domain for tensors~\cite{hao2013facial}. This paper focuses on t-SVD, which introduces a specialized tensor multiplication operation known as the t-product (detailed in Definition \ref{def: tproduct}). 
Building upon the t-product, t-SVD (detailed in Definition \ref{def: t-SVD}) can factorize an order-3 tensor into two orthogonal tensors and an \textit{f}-diagonal tensor. 
Accordingly, its tensor rank called tubal rank has been defined (detailed in  Definition \ref{def: tubal rank}), and the tensor nuclear norm (TNN) (detailed in  Definition \ref{def: TNN}) has been developed for low-tubal-rank approximation.
Under the t-SVD framework, Lu \textit{et al.} \cite{lu2016tensor,lu2019tensor} explored the corresponding  RTPCA model and provided a theoretical guarantee for exact recovery. The RTPCA model can be formulated as
    \begin{equation}
		\min\limits_{\mathcal{X}, \mathcal{S}} ||\mathcal{X}||_* + \lambda||\mathcal{S}||_{1},\ \ 
		\text{s.t. } \mathcal{Y} = \mathcal{X} + \mathcal{S}.
		\label{RTPCA_TNN}
	\end{equation}
     where $\|.\|_*$ denotes the TNN and $\|.\|_1$  is  the $\ell_1$ norm.
While some researchers have investigated different sparse constraints to tackle various forms of sparsity across diverse applications \cite{Zhou_2_1norm, zhang2014novel,sun2019lateral}, mainstream research has primarily concentrated on enhancing low-rank constraints\cite{kernfeld2015tensor, liu2017fourth, xu2019fast, song2020robust,jiang2023dictionary,hu2016twist,liu2018improved,zheng2019mixed,jiang2020multi,mu2020weighted,kong2018t,gao2020enhanced,wu2022low,wang2021generalized,zheng2020tensor,qin2022low, feng2023multiplex,liu2024revisiting}.
However, the aforementioned methods rely on TNN and require computationally intensive full t-SVD computation in each iteration, resulting in considerable overhead, particularly when applied to large-scale datasets.
To tackle this challenge, decomposition-based approaches have been studied for tensor completion \cite{zhou2017tensor,du2021unifying,liu2024low} and RTPCA \cite{wang2020faster}, factorizing low-rank tensors into smaller components connected by the t-product. However, these methods often lack theoretical recovery guarantees or are limited to symmetric positive semi-definite (PSD) tensors.

Recently, scaled gradient descent (ScaledGD or SGD) \cite{tong2021accelerating, dong2023fast} has emerged for its advantage of low per-iteration
cost and strong theoretical guarantee for exact recovery. Motivated by this, this paper develops an efficient ScaledGD method within the t-SVD framework for the RTPCA problem with guaranteed recovery, termed RTPCA-SGD. Specifically, it first factorizes the low-rank component into two smaller factors connected by the t-product. These two factors are updated by the gradient descent algorithm, and  a scaling factor is incorporated into it to eliminate the dependence of convergence rate on the condition number $\kappa$.
Moreover, inspired by the success of the deep-unfolding technology for the RPCA/RTPCA methods \cite{cai2021learned, dong2023deep}, we further propose a learnable self-supervised deep unfolding model termed RTPCA-LSGD to facilitate effective parameter learning.  
The key contributions of this paper are summarized as follows:
\begin{itemize}
    \item To the best of our knowledge, this is the first work to explore the ScaledGD approach within the t-SVD framework. The proposed RTPCA-SGD method is computationally more efficient compared to RTPCA-TNN, as it eliminates the need for full t-SVD computation in each iteration. The complexity of our method per iteration is $O(I_1 I_2 I_3 \log \left(I_3\right) + \lceil \frac{I_3+1}{2} \rceil I_1I_2R)$, which is significantly lower than that of RTPCA-TNN when $R \ll \min(I_1, I_2)$.  Notably, the proposed method is designed to handle general asymmetric cases, and can be seamlessly applied to symmetric PSD tensors without loss of generality.
    % not limited to symmetric PSD tensors 
    \item Theoretically, we establish rigorous exact recovery guarantees for RTPCA-SGD, as detailed in Theorem~\ref{thm:main_theorem}. Compared to the matrix-based RPCA setting \cite{tong2021accelerating,cai2021learned} and the Tucker-based approach \cite{dong2023fast}, our analysis relies on weaker assumptions for both the low-rank and sparse components. With appropriate parameter selection, the proposed algorithm is shown to converge linearly to the true low-rank tensor at a constant rate, unaffected by the condition number.
    \item To further enhance the applicability and performance of our method in practical settings, we introduce a scalable, self-supervised deep unfolding model called RTPCA-LSGD. This model effectively learns the parameters of the proposed RTPCA-SGD algorithm without requiring ground truth data. 
    \item Numerical experiments validate the effectiveness of our methods on both synthetic and real-world datasets. Notably, RTPCA-SGD outperforms state-of-the-art methods in recovery accuracy while maintaining comparable computational efficiency, especially consuming less time than RTPCA-TNN. Furthermore, RTPCA-LSGD, leveraging deep unfolding for learnable parameter optimization, delivers additional performance enhancements.
\end{itemize}

The rest of this paper is organized as follows: Section \ref{sec:related works} reviews the related works. Section \ref{sec: Notations and Preliminaries} introduce some basic notations and preliminaries. In Section \ref{sec: methods}, we present the proposed RTPCA-SGD and RTPCA-LSGD methods in detail. Section \ref{sec: Theoretical Results} outlines the theoretical results for the proposed methods. In section \ref{section_experiment}, we conduct some experiments to verify our results in theory and apply the proposed methods to video denoising and background initialization tasks.
 Finally, we give the conclusions and discuss future directions in Section \ref{sec: conclusions}.

\section{Related works}\label{sec:related works}
Tensor decompositions can be classified into several categories, including CP decomposition~\cite{harshman1970foundations,de2008tensor}, Tucker decomposition~\cite{tucker1966some,mu2014square}, t-SVD~\cite{lu2016tensor, zhang2014novel,lu2019tensor}, and tensor networks (e.g., tensor tree, tensor train, tensor ring)\cite{ballani2014tree, oseledets2011tensor, zhao2016tensor}.
CP decomposition represents a multi-dimensional tensor as the sum of rank-one tensors, with the minimal number of such tensors defining the CP rank. However, accurately estimating the CP rank is computationally challenging and considered NP-hard~\cite{de2008tensor, zhao2015bayesian}.
 Tucker decomposition factorizes a tensor into a core tensor along with a set of factor matrices for each mode, with the ranks of these factor matrices organized into a vector known as the Tucker rank. 
 The extreme imbalance in dimensions (rows and columns) for its unfolding matrices will often result in poor performance \cite{mu2014square}. 
This paper mainly focuses on the t-SVD, which has proven to be highly effective in various image processing applications. Its strength lies in capturing low-rank structures within the Fourier domain for tensors.
The t-SVD framework factorizes an order-3 tensor into two orthogonal tensors and an \textit{f}-diagonal tensor, connected by the t-product. 
This operation involves circular convolution between two tubal fibers, simplifying to a dot product in the Fourier domain~\cite{kilmer2011factorization, zhang2014novel, lu2016tensor}. Based on the t-product,  t-SVD can be efficiently computed by a combination of fast Fourier transform (FFT) along the 3rd mode and multiple matrix SVDs along the 1st and 2nd modes, Leveraging the conjugate symmetry of the Fourier transform, the number of required matrix SVDs is effectively halved, significantly reducing computational complexity~\cite{lu2019tensor, feng2020robust}.
Under the t-SVD framework, theoretical guarantee for the exact recovery has been provided for the TNN-based RTPCA \cite{lu2019tensor} problem and tensor completion problem \cite{zhang2016exact}.
The low-rank constraint has been widely extended by incorporating alternative invertible transforms~\cite{kernfeld2015tensor, liu2017fourth, xu2019fast, song2020robust,jiang2023dictionary}, alleviating rotation sensitivity \cite{hu2016twist,liu2018improved,zheng2019mixed}, developing non-convex versions \cite{jiang2020multi,mu2020weighted,kong2018t,gao2020enhanced,wu2022low,wang2021generalized}, enabling high-order generalizations\cite{zheng2020tensor,qin2022low, feng2023multiplex,liu2024revisiting}, fusing low-rankness and smoothness \cite{wang2023guaranteed}, etc.  
However, these mainstream methods are all based on TNN and suffer from heavy computational overhead due to the t-SVD computation required in each iteration, making them inefficient, particularly for large-scale datasets.

To mitigate this challenge, decomposition-based approaches have gained considerable attention by factorizing low-rank tensors into smaller components linked by t-product.
Zhou \textit{et al.} \cite{zhou2017tensor} introduced tensor bi-factor factorization for tensor completion and employed the alternating minimization algorithm for solution, effectively bypassing the computational burden of SVDs.  Building on this, Du \textit{et al.} \cite{du2021unifying} integrated tensor bi-factor factorization with TNN regularization into a unified framework for tensor completion. These two methods have proved that the proposed algorithm converges to a Karush-Kuhn-Tucker (KKT) point.
For RTPCA, Wang \textit{et al.} \cite{wang2020faster}  proposed triple 
factorization, and established the sub-optimality of the proposed non-convex augmented Lagrangian method. However, these methods lack theoretical recovery guarantees.
Liu  \textit{et al.} \cite{liu2024low} extended the Burer-Monteiro method to the tensor version, followed by solving tensor completion problem through factorized gradient descent (FGD). Theoretical guarantees have been provided for the recovery performance of the FGD method,  demonstrating a linear convergence rate. 
However, this approach is limited to symmetric positive PSD tensors, and the convergence rate depends linearly on the condition number $\kappa$ of the low-rank component.

Recently, ScaledGD or SGD was introduced for the matrix RPCA problem \cite{tong2021accelerating}, incorporating a scaling factor in gradient descent steps to eliminate the convergence rate's dependence on the condition number $\kappa$, while maintaining the low per-iteration cost of gradient descent. Notably, it is designed for general matrices rather than being limited to symmetric PSD matrices. 
Subsequently, Dong \textit{et al.}\cite{dong2023fast} extended ScaledGD to the RTPCA problem under the Tucker decomposition and provided the recovery guarantee theory. RPCA-SGD and Tucker-SGD face practical challenges due to their sensitivity to hyperparameter settings, such as learning rates and thresholds.
To tackle this issue, the deep unfolding technique has been applied to RPCA problem in diverse areas such as ultrasound imaging \cite{solomon2019deep}, background subtraction \cite{van2021deep,joukovsky2023interpretable}, infrared small target detection\cite{wu2024rpcanet}.
It transforms iterative algorithms into deep neural network architectures, enabling parameter learning through backpropagation and enhancing the original algorithm's performance.
Cai \textit{et al.} \cite{cai2021learned} designed a supervised deep unfolded architecture for RPCA-SGD to learn a set of parameters, which can be extended to accommodate an infinite number of RPCA iterations. 
Herrera \textit{et al.} \cite{herrera2020denise} proposed an unsupervised model only for PSD
low-rank matrices. Dong \textit{et al.} \cite{dong2023deep} then considered the self-supervised model for Tucker-SGD, which expunges the need for ground truth labels.

\section{Notation and Preliminaries}
\label{sec: Notations and Preliminaries}
Firstly, some basic notations about this work are given briefly.  A scalar, a vector, a matrix, and a tensor are denoted by  $a$, $\mathbf{a}$, $\mathbf{A}$, and $\mathcal{A}$, respectively.   $\mathbb{R}$ represents the field of real numbers while $\mathbb{C}$ is the field of complex numbers. 

 For a matrix $ \mathbf{A}$, its spectral norm is defined as $\lVert \mathbf{A} \rVert_2 = \text{max}_{i}\mathit{\sigma}_{i}$, where $\mathit{\sigma}_{i}$ are the singular values of $ \mathbf{A}$.
 $\norm{\mathbf{A}}_{2,\infty}$ represents the largest $\ell_2$ norm of the rows. The rank is defined as the number of the nonzero singular values of $ \mathbf{A} $, and is denoted as $\operatorname{rank}(\mathbf{A})$.

For an order-3 tensor $\mathcal{A} \size^{I_1 \times I_2 \times I_3}$, $\mathcal{A}(i_1,:,:)$, $\mathcal{A}(:,i_2,:)$ and $\mathcal{A}(:,:,i_3)$ represent the ${i_1}$-th horizontal, $ i_2 $-th lateral and $ i_3 $-th frontal slices, respectively. The $ i_3 $-th frontal slice can be abbreviated as $\mathbf{A}^{(i_3)}$. Moreover, $\mathcal{A}(:, i_2, i_3)$ , $\mathcal{A}(i_1, :, i_3 )$  and $\mathcal{A}(i_1, i_2, : )$ denotes the row fiber, column and tubal fiber of $\mathcal{A}$, repectively. 
Additionally, its $(i_1,i_2,i_3)$-th element is denoted as $a_{i_1,i_2, i_3}$ or $[\mathcal{A}]_{i_1,i_2,i_3}$.
We use $\widehat{\mathcal{A}} = \fft (\mathcal{A}, [], 3)  \size^{I_1 \times I_2 \times I_3} $ to represent the fast Fourier transform (FFT) along the 3-rd mode of $\mathcal{A}$. In the same way, $\mathcal{A} = \ifft (\widehat{\mathcal{A}}, [], 3)$ stands for the inverse operation.
$ \operatorname{conj}(\mathcal{A})$ is defined as the complex conjugate of $\mathcal{A}$, which takes the complex conjugate of each element. 
We denote 
$\lceil t \rceil$ as the nearest integer greater than or equal to $t$.
Moreover, $a \lor b$ represents the larger value of $a$ and $b$.

Given a tensor $\mathcal{A} \size^{I_1 \times I_2 \times I_3}$, its block circulant matrix is
$$
\operatorname{bcirc}(\mathcal{A})=\left[\begin{array}{cccc}
\Am^{(1)} & \Am^{\left(I_3\right)} & \cdots & \Am^{(2)} \\
\Am^{(2)} & \Am^{(1)} & \cdots & \Am^{(3)} \\
\vdots & \vdots & \ddots & \vdots \\
\Am^{\left(I_3\right)} & \Am^{\left(I_3-1\right)} & \cdots & \Am^{(1)}
\end{array}\right]\in \mathbb{R}^{I_1I_3 \times I_2I_3}.
$$
Moreover, we define the block diagonal matrix $\widehat{\Am}$ as 
$$
\widehat{\Am}=\operatorname{bdiag}(\widehat{\mathcal{A}})=\left[\begin{array}{llll}\widehat{\Am}^{(1)} & & & \\ & \widehat{\Am}^{(2)} & & \\ & & \ddots & \\ & & & \widehat{\Am}^{\left(I_3\right)}\end{array}\right] \in \mathbb{C}^{I_1I_3 \times I_2I_3}.
$$

\begin{definition}\label{def: tproduct}[\textbf{T-product}]
Given two tensors $\AC \size^{I_1 \times I_2 \times I_3}$ and $\BC \size^{I_2 \times J \times I_3}$, the t-product between them is defined as 
\begin{equation}\label{equ: t-product}
   \CC = \AC * \BC \size^{I_1 \times J \times I_3}, 
\end{equation}
which can be computed by 
\begin{equation}
    \CC(i,j,:) = \sum_{k=1}^{I_2} \AC(i,k,:) \bullet \BC(k,j,:),
\end{equation}
where $\bullet$ denotes the circular convolution between two tubal fiber vectors. 
\end{definition}
It has been proven in detail in \cite{lu2019tensor, feng2020robust} that this t-product operator can be calculated in the Fourier domain as 
\begin{equation}
   \widehat{\Cm}^{(i_3)} =\widehat{\Am}^{(i_3)}\widehat{\Bm}^{(i_3)},
\end{equation}
In addition, according to the conjugate symmetry property of the Fourier transform, 
we  only need to compute $\lceil \frac{I_3+1}{2}\rceil$ matrix multiplications for t-product \cite{lu2019tensor} as follows
\begin{equation} \label{fft}
     \widehat{\mathbf{C}}^{(i_3)}=
        \begin{cases}
\widehat{\mathbf{A}}^{(i_3)}\widehat{\mathbf{B}}^{(i_3)}, \ & i_3=1,\cdots, \lceil \frac{I_3+1}{2}\rceil,\\
        \operatorname{conj} (\widehat{\mathbf{C}}^{(I_3-i_3+2)}), \ & i_3=\lceil \frac{I_3+1}{2}\rceil+1,\cdots, I_3.
        \end{cases}
\end{equation}

\begin{definition}[\textbf{Tensor Frobenius norm}] \cite{lu2019tensor}
Given an order-$3$ tensor  $\mathcal{A} \sizeC^{I_1 \times I_2 \times I_N}$, its Frobenius norm is defined as $\lVert \mathcal{A} \rVert_\mathrm{F} =\sqrt{\sum_{i_1,i_2,i_3} |a_{i_1,i_2, i_3}|^2 }$. We have the following property
$\norm{\AC}_{\fro} = \frac{1}{\sqrt{I_3}}\norm{\widehat{\AC}}_{\fro} = \frac{1}{\sqrt{I_3}}\norm{\widehat{\Am}}_{\fro}$.
\end{definition}

\begin{definition}[\textbf{Tensor $\ell_{2,\infty}$ norm}] \label{def: l2inf}The tensor $\ell_{2,\infty}$ norm of an order-$3$ tensor $\mathcal{A}\size^{I_1 \times I_2 \times I_3}$ is defined as $\lVert \mathcal{A} \rVert_\mathrm{2,\infty} = \operatorname{max}_{i_1}\sqrt{\sum_{i_2 = 1}^{I_2}\sum_{i_3 = 1}^{I_3}a_{i_1,i_2, i_3}^{2}}$. And we can have $\norm{\AC}_\mathrm{2,\infty} = \frac{1}{\sqrt{I_3}}\norm{\widehat{\AC}}_\mathrm{2,\infty} = \norm{\operatorname{bcirc}(\mathcal{A})}_\mathrm{2,\infty}$.
\end{definition}

\begin{definition}[\textbf{Tensor $\ell_{1,\infty}$ norm}]\label{def: l1inf} The tensor $\ell_{1,\infty}$ norm of an order-$3$ tensor $\mathcal{A}\size^{I_1 \times I_2 \times I_3}$ is defined as $\lVert \mathcal{A} \rVert_\mathrm{1,\infty} = \operatorname{max}_{i_1}\sum_{i_2 = 1}^{I_2}\sum_{i_3 = 1}^{I_3}|a_{i_1,i_2, i_3}|$. 
\end{definition}

\begin{definition}[\textbf{Tensor infinity norm}] \label{def: inf}The tensor infinity norm of an order-$3$ tensor $\mathcal{A}\size^{I_1 \times I_2 \times I_3}$ is defined as $\lVert \mathcal{A} \rVert_\mathrm{\infty} = \operatorname{max}_{i_1,i_2,i_3}|a_{i_1,i_2,i_3}|$. 
\end{definition}

\begin{definition}[\textbf{Conjugate transpose}] \cite{lu2016tensor}
The conjugate transpose $\AC^{\To} \sizeC^{I_2 \times I_1 \times I_3}$ of $\AC \sizeC^{I_1 \times I_2 \times I_3}$  is achieved by firstly conjugating transpose each frontal slice and then reversing the order of frontal slices from 2 to $I_3$.
\end{definition}

\begin{definition}[\textbf{Identity tensor}]\cite{kilmer2011factorization}    
An order-3 tensor $\IC \size ^{I \times I \times I_3}$  is called as an identity tensor if its first frontal slice is an identity matrix and the other ones are zeros. 
\end{definition}

\begin{definition}[\textbf{Inverse tensor}]\cite{kilmer2011factorization}    
For an order-3 tensor $\AC \size ^{I \times I \times I_3}$, $\mathcal{A}^{-1}\in \mathbb{R}^{I \times I \times I_3}$ is said to be the inverse tensor if it satisfies $\AC * \AC^{-1} = \IC,  \AC^{-1} * \AC = \IC$. 
\end{definition}

\begin{definition}[\textbf{Orthogonal tensor}]\cite{kilmer2011factorization}
   A tensor $\AC$ is called orthogonal when it satisfies
   \begin{equation}
       \AC^{\To} * \AC = \AC * \AC^{\To} =\IC,
   \end{equation}
   where $\IC$ is the identity tensor, whose first frontal slice is the identity matrix and the rest are zero.
\end{definition}

\begin{definition}[\textbf{F-diagonal tensor}]\cite{kilmer2011factorization} A tensor $\AC$ is called f-diagonal when each frontal slice $\AC^{(i_3)}, i_3 = 1, \cdots, I_3$ is a diagonal matrix.
   
\end{definition}  

\begin{definition}[\textbf{T-SVD}]\cite{kolda2009tensor} \label{def: t-SVD}
For a three-way tensor $\AC \size^{I_1 \times I_2 \times I_3}$, its tensor singular value decomposition (t-SVD) can be represented as 
\begin{equation}\label{equ: t-SVD}
   \AC = \UC*\mathit{\Sigma}*\VC^{\mathrm{T}}, 
\end{equation}
where $\UC \size^{I_1 \times I_1 \times I_3}$ and $\VC \size^{I_2 \times I_2 \times I_3}$ are orthogonal tensors while $\mathit{\Sigma} \size^{I_1 \times I_2 \times I_3}$ is a f-diagonal tensor. 
\end{definition}
Based on the t-product, this decomposition can be obtained by computing matrix SVDs in the Fourier domain. Then we can acquire
\begin{equation}
    \widehat{ \mathbf{A} }^{(i_3)} = \widehat{ \mathbf{U} }^{(i_3)} \widehat{ \mathbf{\Sigma} }^{(i_3)} \widehat{\mathbf{V}}^{(i_3)\mathrm{T}}, i_3 = 1, \cdots, I_3.
\end{equation}
Taking advantage of the property (\ref{fft}), we can obtain the detailed t-SVD calculation process in Algorithm \ref{Alg: T-SVD}.  The number of SVD to be computed is $ \lceil\frac{I_3+1}{2}\rceil$ for a tensor with size $I_{1}\times I_{2}\times I_{3}$.  Figure \ref{fig: t-svd} also gives a decomposition graph for better understanding.
Compared with matrix SVD, t-SVD can better excavate low-rank structures in multi-way tensors. 

\begin{algorithm}[!t]
    \caption{T-SVD for order-3 tensor}\label{Alg: T-SVD}
    \LinesNumbered % display line number
    \KwIn{$ \AC \size^{I_{1}\times I_{2}\times I_{3}} $.}
    $\widehat{\AC} = \fft (\AC, [], 3)$,\\
    \For{$i_3 = 1,\cdots, \lceil\frac{I_3+1}{2}\rceil$}
    {
    $[\widehat{ \mathbf{U} }^{(i_3)}, \widehat{ \mathbf{\Sigma} }^{(i_3)},  
    \widehat{\mathbf{V}}^{(i_3)}  ] $ = SVD$(\widehat{ \mathbf{A} }^{(i_3)})$;}
    \For{$i_3 = \lceil\frac{I_3+1}{2}\rceil +1, \cdots, I_3$}
{$\widehat{ \mathbf{U} }^{(i_3)} = \operatorname{conj} (\widehat{
    \mathbf{U}}^{(I_3-i_3+2)})$;\\
    $\widehat{ \mathbf{\Sigma} }^{(i_3)} = \widehat{ \mathbf{\Sigma} }^{(I_3-i_3+2)}$;\\
      $\widehat{ \mathbf{V} }^{(i_3)} = \operatorname{conj} (\widehat{ \mathbf{V} }^{(I_3-i_3+2)})$;
  }
  $\UC = \ifft(\widehat{\UC},[],3)$, $\mathit{\Sigma} = \ifft(\widehat{\mathit{\Sigma}},[],3)$,
  $\VC=\ifft(\widehat{\VC},[],3)$.\\
  \KwOut{$\UC, \mathit{\Sigma}, \VC$.}
 \end{algorithm}

 \begin{figure}
    \includegraphics[width=1\linewidth]{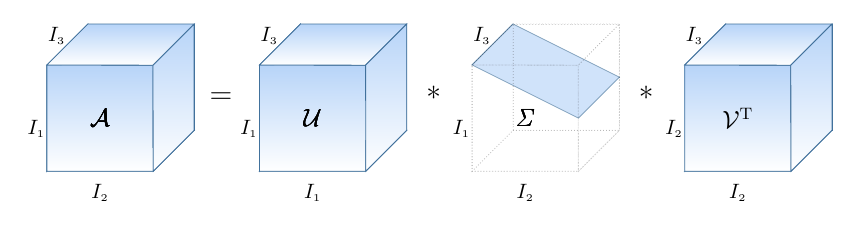}
    \vspace{-0.8cm} 
    \caption{Illustration of t-SVD framework.} \label{fig: t-svd}
    \vspace{-0.2cm}
\end{figure}

\begin{definition}[\textbf{Tensor tubal rank and tensor multi-rank}]\cite{zhang2016exact,lu2019tensor} \label{def: tubal rank}
Given a tensor $\AC \size^{I_1 \times I_2 \times I_3 }$ and its t-SVD $\AC = \UC * \mathit{\Sigma} * \VC^{\To}$, the tubal rank $\rank_t (\AC)$ is defined as the number of nonzero singular tube fibers of $\mathit{\Sigma}$, which is formulated as 
\begin{equation}
\begin{split}
    \operatorname{rank}_t (\AC) &= \# \{i, \mathit{\Sigma}(i,i,:) \neq \mathbf{0} \} \\
   & = \# \{i, \mathit{\Sigma}(i,i,1) \neq 0 \}.
    \end{split}
\end{equation}
In addition,  its multi-rank is defined as a vector $\rank_m (\AC) = \mathbf{R} = (R_1, R_2, \cdots, R_{I_3})$, whose $i_3$-th entry  satisfies $R_{i_3} = \rank(\widehat{\Am}^{(i_3)}), i_3 = 1, 2, \cdots, I_3$. Then we can denote $\norm{\mathbf{R}}_1$ as the sum of the multi-rank, and tubal rank as $\norm{\mathbf{R}}_{\mathrm{\infty}} $ or $R$. $\sigma_{\min}(\AC) = \min (\widehat{\mathit{\Sigma}})$ is defined as the minimum non-zero singular value of all frontal slices of $\widehat{\AC}$ in the Fourier domain, and we represent $\sigma_{\min}(\AC) = \sigma_{\min}(\widehat{\Am})$. Similarly, we define $\sigma_{\max}(\AC) = \sigma_{\max}(\widehat{\Am}) = \max (\widehat{\mathit{\Sigma}})$.
If the tubal rank of tensor $\AC$ is $R$, it will have the skinny t-SVD $\AC = \UC * \mathit{\Sigma} * \VC^{\To}$, which satisfies $\UC \size^{I_1 \times R \times I_3}, \mathit{\Sigma} \size^{R \times R \times I_3}, \VC \size^{I_2 \times R \times I_3}, \UC^ {\To} * \UC = \VC^{\To} * \VC = \IC$.
\end{definition}

\begin{definition}\label{def: TNN}[\textbf{ Tensor nuclear norm}]\cite{zhang2016exact,lu2019tensor} 
Meanwhile, given an order-3 tensor $\mathcal{A}$ with the tubal rank $R$, its tensor nuclear norm (TNN) is defined as 
\begin{equation}
 \begin{split}
    \norm{\AC}_*  
    & = \frac{1}{I_3}  \norm{\widehat{\AC}}_* = \frac{1}{I_3}  \sum_{i_3 = 1}^{I_3} \norm{\widehat{\Am}^{(i_3)}}_* \\
     & = \frac{1}{I_3}  \sum_{i_3 = 1}^{I_3} \sum_{r = 1}^{R} \widehat{\mathbf{\Sigma}}(r, r, i_3).
    \end{split}
\end{equation}
\end{definition}

\begin{definition}[\textbf{Tensor spectral norm}]\cite{lu2019tensor} For a tensor $\AC \size ^{I_1 \times I_2 \times I_3} $, its tensor spectral norm is denoted as $\norm{\AC}_2 = \norm{\operatorname{bcirc}(\mathcal{A})}_2 = \norm{\widehat{\mathbf{A}}}_2$. Note that TNN is the dual norm of the tensor spectral norm.
\end{definition}

\begin{definition}[\textbf{Condition number}]
For $\AC \size ^{I_1 \times I_2 \times I_3}$, its condition number is defined as 
$\kappa(\mathcal{A}) = \frac{\sigma_{\max}(\mathcal{A})}{\sigma_{\min}(\mathcal{A})}$.
\end{definition}

\begin{definition}[\textbf{Tensor operator norm}]
% A tubal-rank r 
Given a operator $\LC \size ^{I_1 \times I_2 \times I_3}$, its tensor operator norm is defined as $\|\mathcal{L}\|_{\op} = \sup_{\norm{\mathcal{A}}_{\fro} \leq1} \norm{\mathcal{L}(\mathcal{A})}_{\fro}
$. If $\mathcal{L}(\mathcal{A}) = \mathcal{L}*\mathcal{A}$, tensor operator norm is equivalent to tensor spectral norm.
\end{definition}

\begin{definition}[\textbf{Standard tensor basis}]\cite{zhang2016exact} Standard tensor basis $\mathring {\mathfrak{e}}_i$ is a tensor with size $I_1 \times 1 \times I_3$, whose $(i,1,1)$-th entry is 1 and the rest is 0. 
\end{definition}

\section{Methods}\label{sec: methods}
\subsection{Model}
Given the observed data  $\mathcal{Y} \in 
\mathbb{R}^{I_{1}\times I_{2} \times I_{3}}$, RTPCA aims to separate the intrinsic low-rank component  $\XC_{\star}\in 
\mathbb{R}^{I_{1}\times I_{2} \times I_{3}}$ and the underlying sparse component $\SC_{\star}\in 
\mathbb{R}^{I_{1}\times I_{2} \times I_{3}}$ from their mixture $\YC=\XC_{\star}+\SC_{\star}$.
When we denote the tensor  $\XC\in\mathbb{R}^{I_{1}\times I_{2} \times I_{3}}$ with the tubal rank $R$ in a factored form $\XC=\LC *\RC^{\top}$, where $\LC\in 
\mathbb{R}^{I_{1}\times R \times I_{3}},\RC\in 
\mathbb{R}^{I_{2}\times R \times I_{3}}$, we consider the following problem: 
\begin{equation}
\begin{aligned}
\min_{\LC,\RC,\SC}& \ \frac{1}{2}\ \left\Vert \LC * \RC^{\top}+\SC-\YC\right\Vert_{\fro}^{2}\\
\operatorname{s.t.} \ & \ \mathcal{S} \  \text{is an} \ \alpha-\text{sparse tensor}.
\end{aligned}
\end{equation}
Note that we employ two small factors here, which are different from the classic t-SVD form as shown in  Definition \ref{def: t-SVD}, \textit{i.e.}, $\XC = \UC*\mathit{\Sigma}*\VC^{\mathrm{T}}$.
In real-world scenarios, $\alpha$ is typically unknown and challenging to optimize. Then we formulate the following non-convex optimization model:
\begin{equation}
\begin{aligned}
\min_{\LC,\RC,\SC}& \ \frac{1}{2}\ \left\Vert \LC * \RC^{\top}+\SC-\YC\right\Vert_{\fro}^{2}\\
\operatorname{s.t.} \ & \ \operatorname{supp}\left(\mathcal{S}\right) \subseteq \operatorname{supp}\left(\mathcal{S}_{\star}\right),
\end{aligned}
\end{equation}
where $\operatorname{supp}(\cdot)$ denotes the support set, which is the set of indices where the elements of a tensor are non-zero.   

\subsection{Algorithm}
\begin{algorithm*}
    \caption{RTPCA-SGD under the t-SVD format}\label{Alg: LGRTPCA1}
    \KwIn{Observed tensor $\YC\size^{I_1 \times I_2 \times I_3}$, the underlying tubal rank $R$, step sizes $\eta_{k} = \eta$, the thresholding values $\zeta_{0}, \zeta_{1}$, the decay rate $\tau$, and the iteration number $K$.}
     \textbf{Initialization:} $\SC_{0} = \operatorname{T}_{\zeta_0}(\YC)$; $[\LC_{0},\RC_{0}]=\mathcal{D}_{R}(\YC - \SC_{0})$.\\
     \textbf{for} $k = 0, 1, \cdots, K-1$ \textbf{do}\\
     % Update S
     \quad $\SC_{k+1} = \operatorname{T}_{\zeta_{k+1}}(\YC - \LC_{k} * \RC_{k}^{\To})$, where $\zeta_{k+1} = \tau^k \zeta_1$;\\
     % Update L and R
     \quad$\mathcal{L}_{k+1} = \mathcal{L}_{k} - \eta_{k+1}(\mathcal{L}_{k} * \mathcal{R}_{k}^{\To} + \mathcal{S}_{k+1} - \mathcal{Y})*\mathcal{R}_k*(\mathcal{R}_k^{\To}*\mathcal{R}_k)^{-1}$;\\
     % Update R
     \quad$\mathcal{R}_{k+1} = \mathcal{R}_{k} - \eta_{k+1}(\mathcal{L}_{k} * \mathcal{R}_{k}^{\To} + \mathcal{S}_{k+1} - \mathcal{Y})^{\To}*\mathcal{L}_{k}*(\mathcal{L}_{k}^{\To}*\mathcal{L}_{k})^{-1}$;\\
     \textbf{end for}  \\
     \KwOut{Recovered low-rank tensor $\XC_K = \LC_K * \RC^{\To}_K$ and sparse tensor $\SC_K$.}
\end{algorithm*} 

To address the abovementioned problem, our RTPCA-SGD algorithm contains two parts: a) Spectral initialization and b) Scaled gradient updates. We first present the top-$R$ tensor approximation operator in Definition \ref{def: top-R} and the soft-thresholding operator in Definition \ref{def: soft_thres}. 
\begin{definition}[\textbf{Top-$R$ tensor approximation}]\label{def: top-R}Given a tensor $\mathcal{A}\in\mathbb{R}^{I_1\times I_2\times I_3}$, its truncated t-SVD with the parameter $R$ is represented as $\mathcal{A} = \mathcal{U}_{R}*\mathit{\Sigma}_{R}*\mathcal{V}_{R}^{\To}$, where $\mathcal{U}_{R}\in\mathbb{R}^{I_1 \times R \times I_3}$ and $\mathcal{V}_{R}\in\mathbb{R}^{I_2 \times R \times I_3}$  repectively denote  the top-$R$ left  and right singular tensor of $\mathcal{A}$, and $\mathit{\Sigma}_{R} \in\mathbb{R}^{R \times R \times I_3}$ is the top-$R$ f-diagonal tensor. Then
its top-$R$ tensor approximation is defined as 
\begin{equation}
     [\LC,\RC] = \mathcal{D}_{R}(\mathcal{A}),
\end{equation}
where $\mathcal{L} = \mathcal{U}_{R}*{\mathit{\Sigma}_{R}}^{\frac{1}{2}}$ and  $\mathcal{R} = \mathcal{V}_{R}*{\mathit{\Sigma}_{R}}^{\frac{1}{2}}$.
\end{definition}

\begin{definition}[\textbf{Soft thresholding operator}]\label{def: soft_thres} Given a tensor $\mathcal{A}\in\mathbb{R}^{I_1\times I_2\times I_3}$ and a thresholding value $\zeta$, its soft thresholding operator, denoted as $\operatorname{T}_{\zeta}(\mathcal{A})$, is defined such that each entry is computed by
\begin{equation}
\begin{aligned}
\left[\operatorname{T}_\zeta(\mathcal{A})\right]_{i_1, i_2, i_3} &=\operatorname{sign}\left(a_{i_1, i_2, i_3}\right) \cdot \max \left(0,\left|a_{i_1, i_2, i_3}\right|-\zeta\right)\\
\end{aligned}
\end{equation}
\end{definition}

\paragraph{\textbf{Spectral initialization}} \label{alg: Spectral initialization}
For the sparse tensor, we initialize it using the soft thresholding operator with parameter \(\zeta_0\), defined as \(\SC_{0} = \operatorname{T}_{\zeta_0}(\YC)\), to eliminate apparent outliers. For the low-rank component, we initialize its two factors using the top-\(R\) tensor approximation: \([\LC_{0}, \RC_{0}] = \mathcal{D}_{R}(\YC - \SC_{0})\). It is clear that the choice of the initial threshold parameter \(\zeta_0\) may influence the recovery quality. Such initialization form is designed to ensure recovery guarantees, as presented in Theorem~\ref{thm:initial} of Section~\ref{sec: Theoretical Results} later.

\paragraph{\textbf{Scaled gradient updates}} We iteratively update the sparse component via a soft-thresholding operator and update two smaller factors of the low-rank component via ScaledGD.
We now discuss the key details of our algorithm as follows:
\begin{enumerate}
    \item \textbf{Update the sparse tensor $\mathcal{S}$}: We choose a simple yet effective soft thresholding operator for updating the sparse component at the $k+1$-th iteration as follows:
    \begin{equation}
        \SC_{k+1} = \operatorname{T}_{\zeta_{k+1}}(\YC - \LC_{k} * \RC_{k}^{\To}).
    \end{equation}
Soft-thresholding has already been applied as a proximal operator of $\ell_1$ norm for tensors in some TNN-based RTPCA algorithms \cite{lu2016tensor,lu2019tensor,feng2020robust}.
This paper shows that with a set of appropriate thresholding values with $\zeta_{k+1} \geq\|\mathcal{X}_\star-\mathcal{X}_{k}\|_\infty$,  $\operatorname{T}_{\zeta_{k+1}}$ will be effectively a projection operator onto the support of $\mathcal{S}_\star$, which is later detailed 
in Lemma~\ref{lm:sparity} of Section~\ref{sec: Theoretical Results}.
\item \textbf{Update  the low-rank tensor $\mathcal{X}$}: To avoid  the computational cost of the full t-SVD  on $\mathcal{X}$, we factorize it into two smaller factors connected by the t-product, \textit{i.e.}, $\XC=\LC *\RC^{\top}$ with $\LC\in 
\mathbb{R}^{I_{1}\times R \times I_{3}}$ and $\RC\in 
\mathbb{R}^{I_{2}\times R \times I_{3}}$. Define the objective function as $f(\mathcal{L},\mathcal{R}):=\frac{1}{2}\left\Vert \LC * \RC^{\top}+\SC-\YC\right\Vert_{\fro}^{2}$, then the gradients with respect to $\mathcal{L}$ and $\mathcal{R}$ can be computed by
\begin{equation}
\begin{aligned}
       \nabla_{\mathcal{L}} f_k&=\left(\mathcal{L}_k* \mathcal{R}_k^{\top}+\mathcal{S}-\mathcal{Y}\right)* \mathcal{R}_k,\\
       \nabla_{\mathcal{R}} f_k&=\left(\mathcal{L}_k *\mathcal{R}_k^{\top}+\mathcal{S}-\mathcal{Y}\right)^{\top}*\mathcal{L}_k.
\end{aligned}
\end{equation}
It has been observed that the standard gradient descent method suffers from ill-conditioning issues  \cite{tong2021accelerating,cai2021learned,dong2023fast}. To overcome this, we introduce the scaled terms $(\mathcal{R}_k^{\To}*\mathcal{R}_k)^{-1}$ and $(\mathcal{L}_{k}^{\To}*\mathcal{L}_{k})^{-1}$
to update the tensor factors as follows:
\begin{equation}
\begin{aligned}
\label{eq: update Xfactor}
    \mathcal{L}_{k+1} &= \mathcal{L}_{k} - \eta_{k+1}\nabla_{\mathcal{L}} f_k*(\mathcal{R}_k^{\To}*\mathcal{R}_k)^{-1},\\
   \mathcal{R}_{k+1} &= \mathcal{R}_{k} - \eta_{k+1}\nabla_{\mathcal{R}} f_k*(\mathcal{L}_{k}^{\To}*\mathcal{L}_{k})^{-1}.
   \end{aligned}
\end{equation}
where $\eta_{k+1}$ is the step size at the $(k + 1)$-th iteration, which needs to be carefully selected to ensure good recovery.
\end{enumerate}
Finally, we summarize the proposed algorithm RTPCA-SGD in Algorithm~\ref{Alg: LGRTPCA1}.

\subsection{Parameter learning}
The exact recovery guarantee for Algorithm~\ref{Alg: LGRTPCA1} will be given in Theorem~\ref{thm:main_theorem} of Section~\ref{sec: Theoretical Results}, which indicates that four parameters need to be appropriately chosen, \textit{i.e.}, the thresholding values $\zeta_0$, $\zeta_1$, the decay rate $\tau$, and the gradient step size $\eta$.
In particular, $\zeta_{k+1} = \tau \zeta_k = \tau^k \zeta_1, k\geq 1$.
To address the challenge of parameter selection in practical applications, we propose a learnable self-supervised deep unfolding method called RTPCA-LSGD. This model efficiently learns the four parameters of the RTPCA-SGD algorithm, enhancing both applicability and performance.
Our approach regards each iteration of Algorithm~\ref{Alg: LGRTPCA1} as a neural network layer, where $\zeta_0$, $\zeta_1$, $\tau$, and $\eta$ are treated as learnable parameters. 
Softplus activations are employed to ensure $\zeta_{0}, \zeta_{1}, \eta > 0$, while the sigmoid function is used to enforce \(0 < \tau < 1\).
As illustrated in Fig.~\ref{fig: deep_unfolding}, the network consists of a feed-forward layer and a recurrent layer, which are respectively designed to model the spectral initialization and the ScaleGD iterative steps of Algorithm~\ref{Alg: LGRTPCA1}.
We employ a self-supervised learning loss (SLL) \cite{herrera2020denise, dong2023deep} for parameter optimization, encouraging the sparsity of the sparse tensor. The SLL is defined as:
 \begin{equation}
l(\LC,\RC) = \frac{1}{\|\mathcal{Y}\|_{\mathrm{F}}^2}\left\|\mathcal{Y}- \LC * \RC^{\To} \right\|_1.
\end{equation}
Here we only use the observed data $\YC$ and the updated terms of the low-rank component. This self-supervised approach eliminates the need for ground truth data, making it highly suitable for real-world applications.
\begin{figure} 
    \includegraphics[width=1\linewidth]{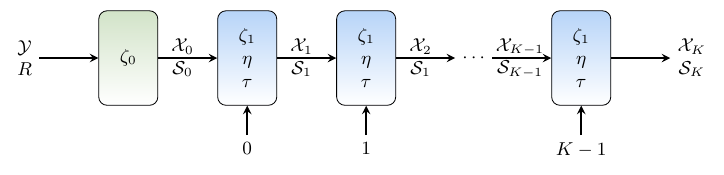}
    \caption{Network architecture. The observed tensor $\mathcal{Y}$ and the tubal rank $R$ serve as inputs to the first layer, performing the spectral initialization step described in Algorithm~\ref{Alg: LGRTPCA1}. The subsequent recurrent layers implement the iterative updates for $\XC_{k} = \LC_{k} * \RC_{k}^{\top}$ and $\mathcal{S}_{k}$, following Algorithm~\ref{Alg: LGRTPCA1}. These layers require the iteration number $k$, as it determines the \((k + 1)\)-th thresholding value parameter, defined as $\zeta_{k+1} = \tau^{k} \zeta_{1}$ for $k \geq 0$. }
\label{fig: deep_unfolding}
\end{figure}

\subsection{Computational complexity}
\label{sec: Computational complexity}
For the TNN-based RTPCA problem \cite{lu2019tensor}, as formulated in Eq.~\eqref{RTPCA_TNN}, the alternating direction method of multipliers (ADMM) is typically employed for its solution. The primary computational cost arises from the proximity operator of the TNN at each iteration. 
Specifically, the per-iteration computational complexity is
$O\left(I_1 I_2 I_3 \log \left(I_3\right) + \lceil \frac{I_3+1}{2} \rceil I_1 I_2 \min(I_1, I_2)\right)$.
In contrast, the proposed method outlined in Algorithm~\ref{Alg: LGRTPCA1} primarily incurs computational cost from updating the two small factor matrices of the low-rank component, as shown in Eq.~\eqref{eq: update Xfactor}. 
The per-iteration computational complexity is
$O(I_1 I_2 I_3 \log \left(I_3\right) + \lceil \frac{I_3+1}{2} \rceil I_1I_2R)$  when the tubal rank $R \ll \min(I_1, I_2)$, resulting in a much lower overall cost compared to RTPCA-TNN.

\section{Theoretical Results}\label{sec: Theoretical Results} 
In this section, we present the recovery guarantee for RTPCA-SGD and prove that with appropriate parameter selection, RTPCA-SGD is able to recover the low-rank component and sparse component under mild assumptions.
We first define two assumptions for the RTPCA problems:
\begin{assumption}[\textbf{Tensor $\mu$-incoherence conditions of low-rank tensor}] \label{as:incoherence} $\mathcal{X}_{\star} \in \mathbb{R}^{I_1 \times I_2 \times I_3}$ satisfies the tensor $\mu$-incoherence conditions and its skinny t-SVD with tubal-rank $R$ is represented as $\mathcal{X}_{\star} = \UC_{\star} * \mathit{\Sigma}_{\star} * \VC^{\To}_{\star}$, which means
\begin{equation}
\begin{aligned}\label{eq: tic}
\norm{\UC_{\star}}_\mathrm{2,\infty} =     \max_{i_1 = 1,2, \cdots, I_1} \norm{\UC_{\star}^{\To} * \mathring{\mathfrak{e}}_{i_1}}_{\fro}   &\leq \sqrt{\frac{\mu R}{I_1 }},\\
\norm{\VC_{\star}}_\mathrm{2,\infty} =     \max_{i_2 = 1,2, \cdots, I_2} \norm{\VC_{\star}^{\To}  * \mathring{\mathfrak{e}}_{i_2}}_{\fro}    &\leq \sqrt{\frac{\mu R}{I_2 }},
\end{aligned}
\end{equation}
for some constant $\mu \geq 1$. Here $\mathring{\mathfrak{e}}_i$ denotes the standard tensor basis.
\end{assumption}

In the Fourier domain, $\widehat{ \mathbf{X} }_{\star}^{(i_3)} = \widehat{ \mathbf{U} }_{\star}^{(i_3)} \widehat{ \mathbf{\Sigma} }_{\star}^{(i_3)} \widehat{\mathbf{V}}_{\star}^{(i_3)}, i_3 = 1, \cdots, I_3$. 
For each frontal slices $\widehat{\mathbf{X}}_{\star}^{(i_3)}, i_3 = 1, \cdots, I_3$, the matrix incoherence conditions are represented as
\begin{equation}
\begin{aligned}\label{eq: mic}
    \max_{i_1 = 1,2, \cdots, I_1} \norm{[\widehat{\mathbf{U}}_{\star}^{(i_3)}]^{\To} \mathbf{e}_{i_1}}_{\fro} &\leq \sqrt{\frac{\mu R}{I_1}}, \\
    \max_{i_2 = 1,2, \cdots, I_2} \norm{[\widehat{\mathbf{V}}_{\star}^{(i_3)}]^{\To} \mathbf{e}_{i_2}}_{\fro} &\leq \sqrt{\frac{\mu R}{I_2}},
\end{aligned}
\end{equation}
for some constant $\mu \geq 1$. Here $\mathbf{e}_{i}$ denotes the standard matrix basis with the $i$-th entries as 1 and the rest as 0.

\begin{remark}
As demonstrated in \cite{zhang2016exact}, the tensor incoherence conditions in Assumption \ref{as:incoherence} can be derived from the matrix incoherence condition, but not vice versa, \textit{i.e.}, \eqref{eq: tic} $\Leftarrow$ \eqref{eq: mic} and \eqref{eq: tic} $ \not \Rightarrow $ \eqref{eq: mic}.
This indicates that tensor incoherence conditions are weaker than those in the matrix setting \cite{cai2021learned} and the Tucker decomposition-based scenario \cite{dong2023fast}.
\end{remark}

\begin{assumption}[\textbf{$\alpha$-t-sparsity of tensor}] \label{as:sparsity}$\mathcal{S}_{\star} \in \mathbb{R}^{I_1 \times I_2 \times I_3}$  is an  $\alpha_t$-sparse tensor, which implies that  at most $\alpha_t$ fraction of non-zero element in each slices $\mathcal{S}(i_1,:,:)$, $\mathcal{S}(:,i_2,:)$, $\mathcal{S}(:,:,i_3)$.
\end{assumption}

To clarify its difference, we recall the $\alpha$-sparsity assumed in the RPCA-SGD \cite{tong2021accelerating} and Tucker-SGD \cite{dong2023fast}.

\begin{assumption}
[\textbf{$\alpha$-m-sparsity in RPCA-SGD} \cite{tong2021accelerating}] \label{as:RPCA-sparsity}$\mathbf{S}_{\star} \in \mathbb{R}^{I_1 \times I_2}$  is an  $\alpha_{m}$-sparse matrix, which implies that  at most $\alpha_m$ fraction of non-zero element in each rows  $\mathbf{S}(i_1,:)$ and columns $\mathbf{S}(:,i_2)$.
\end{assumption}

\begin{assumption}
[\textbf{$\alpha$-f-sparsity in Tucker-SGD} \cite{dong2023fast}] \label{as:Tucker-sparsity}$\mathcal{S}_{\star} \in \mathbb{R}^{I_1 \times I_2 \times I_3}$  is an  $\alpha_f$-sparse tensor, which implies that  at most $\alpha_f$ fraction of non-zero element in each fibers $\mathcal{S}(i_1,i_2,:)$, $\mathcal{S}(:,i_2,i_3)$, $\mathcal{S}(i_1,:,i_3)$.
\end{assumption}

\begin{remark}
Assumption \ref{as:sparsity} imposes the sparsity on the order-$2$ slices rather than on the order-$1$ vectors or fibers, which implies a weaker sparsity condition on the sparse component compared to the matrix case \cite{tong2021accelerating} and the Tucker decomposition case \cite{dong2023fast}. 
For further clarity, we present an illustrative example in Fig.~\ref{fig: sparsity_compare}.  Consider a tensor \(\mathcal{A} \in \mathbb{R}^{5 \times 5 \times 5}\) where all entries in the \((1,1)\)-th column fiber \(\mathcal{A}(:, 1, 1)\) are non-zero, and the remaining entries are 0. In this case, \(\mathcal{A}\) is an \(\alpha_t\)-sparse tensor with \(\alpha_t = 20\% \). 
In contrast, it is  an \(\alpha_f\)-sparse tensor with \(\alpha_f = 100\% \) in the Tucker-ScaledGD setting. Clearly, \(\alpha_t < \alpha_f\), demonstrating that Assumption \ref{as:sparsity} about the sparse component in this paper is inherently weaker than Assumption \ref{as:Tucker-sparsity} in the Tucker case.
In addition, as also discussed in \cite{cai2024robust}, by unfolding \(\mathcal{A}\) along the 1st mode, we obtain a matrix \(\mathbf{A} \in \mathbb{R}^{5 \times 25}\), where only the  1st  column \(\mathbf{A}(:,1)\) contains non-zero entries, making \(\mathbf{A}\) an \(\alpha_m\)-sparse matrix with \(\alpha_m = 100\%\). Clearly, \(\alpha_t < \alpha_m\).
This example demonstrates that unfolding a sparse tensor along a specific mode can significantly worsen the sparsity of the resulting matrix. We also observe that unfolding along a different mode (e.g., the 2nd or 3rd mode) might yield a matrix with no worse sparsity. However, practical applications often lack prior knowledge of the outlier pattern, making it challenging to identify the optimal unfolding mode. Therefore, the tensor sparsity condition in Assumption \ref{as:sparsity} is inherently weaker than the matrix sparsity condition in Assumption \ref{as:RPCA-sparsity}.

\end{remark}

 \begin{figure}
    \includegraphics[width=1\linewidth]{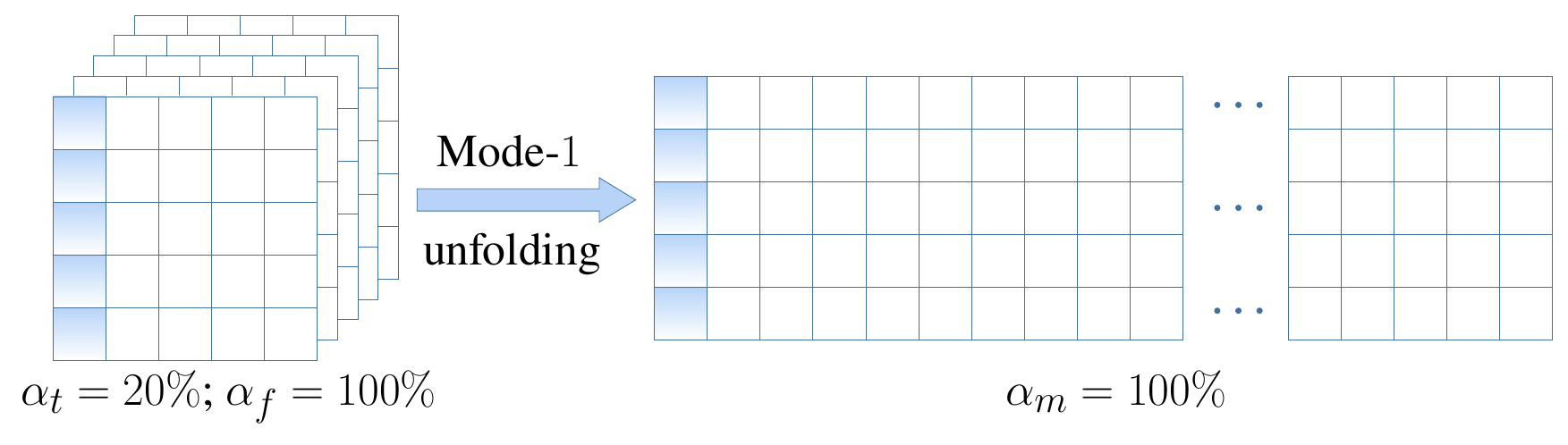}
    \vspace{-0.5cm}
    \caption{Comparison of tensor sparsity and matrix sparsity. Blue boxes denote outlier entries, while white boxes indicate zero entries. The matrix on the right is obtained by unfolding the tensor on the left along the 1st mode.} \label{fig: sparsity_compare}
    \vspace{-0.2cm}
\end{figure}

With these two weaker assumptions in place, we are now prepared to present our main theorem.

\begin{theorem}[\textbf{Guaranteed recovery}]\label{thm:main_theorem}
    Suppose that $\mathcal{X}_{\star}$ with tubal-rank $R$ satisfies tensor $\mu$-incoherence conditions (\textit{i.e.}, Assumption \ref{as:incoherence}), $\mathcal{S}_{\star}$ is an $\alpha$-sparse tensor (\textit{i.e.}, Assumption \ref{as:sparsity}) with $\alpha \leq \frac{1}{10^4 \mu R^{1.5} {I_3}^{1.5} \kappa}$. If we set the thresholding values $\left\|\mathcal{X}_{\star}\right\|_{\infty} \leq \zeta_0 \leq 2\left\|\mathcal{X}_{\star}\right\|_{\infty}$ 
    and $\zeta_{k+1} = \tau \zeta_k, k\geq 1$ with  $\zeta_{1} := \frac{3}{\sqrt{I_1I_2}} \mu R \sigma_{\min}(\mathcal{X}_\star) $, 
$\eta_k=\eta \in\left[\frac{1}{4}, \frac{2}{3}\right]$, where $\tau = 1-0.6\eta$, the iterates of RTPCA-SGD satisfy
   \begin{equation} 
   \begin{aligned}
   \label{eq: con3}
    \left\|\mathcal{X}_k-\mathcal{X}_{\star}\right\|_{\mathrm{F}} &\leq  \frac{0.03}{\sqrt{I_3}} {\tau}^{k}  \sigma_{\min}\left(\mathcal{X}_{\star}\right),\cr
\|\mathcal{X}_k -\mathcal{X}_\star\|_\infty &\leq \frac{3}{\sqrt{I_1I_2}} \mu R \tau^k \sigma_{\min}(\mathcal{X}_\star), \cr
\|\mathcal{S}_k-\mathcal{S}_\star\|_\infty &\leq \frac{6}{\sqrt{I_1I_2}} \mu R \tau^{k-1} \sigma_{\min}(\mathcal{X}_\star),
\end{aligned}
    \end{equation}
    and \begin{equation}\operatorname{supp}\left(\mathcal{S}_k\right) \subseteq \operatorname{supp}\left(\mathcal{S}_{\star}\right)
    \end{equation}
\end{theorem}
\begin{remark} 
The above theoretical results provide recovery guarantees for RTPCA-SGD. We can summarize  three important consequences as follows:
\begin{itemize}
    \item \textbf{Claim 1}: Algorithm~\ref{Alg: LGRTPCA1} with appropriate parameter choices can exactly recover the low-rank and sparse tensors as long as the fraction $\alpha$ of corruptions and the tubal rank $R$ are sufficiently small.
    \item \textbf{Claim 2}: It converges linearly at a constant rate on three error metrics, as shown in Eq.~\eqref{eq: con3}. \label{claim:2}
    \item \textbf{Claim 3}: The convergence rate is independent of $\kappa$. 
\end{itemize}
Here we choose the decay rate $\tau = 1 - 0.6\eta$ to simplify the proof. This may not be an optimal convergence rate.
To the best of our knowledge, this is the first time that a ScaledGD-based theoretical analysis for tensors under the t-SVD framework has been presented. The existing works are based on matrix SVD \cite{tong2021accelerating, cai2021learned} and Tucker decomposition \cite{dong2023fast}. 
Additionally, this paper focuses on the general asymmetric case, and our theoretical results can also be seamlessly applied to PSD tensors without loss of generality by replacing $\mathcal{R} = \mathcal{L}$.
\end{remark}

In preparation for proving Theorem~\ref{thm:main_theorem}, we first introduce a new tensor distance metric in Definition \ref{de: em} and optimal alignment tensor existence lemma in Lemma \ref{lm: Q_existence}. 
Next, Lemma~\ref{lemma:matrix2factor} connects this tensor distance metric to the error metric of the low-rank component.
On the other hand, we also present Lemma~\ref{lm:sparity}, which confirms the effectiveness of our selected thresholding values for the sparse component. 
Additionally, we present Theorem~\ref{thm:local convergence} which establishes local linear convergence and Theorem~\ref{thm:initial} which provides guaranteed initialization.
Detailed proofs for Lemma~\ref{lm: Q_existence},~\ref{lemma:matrix2factor}, \ref{lm:sparity}  and Theorem~\ref{thm:local convergence},~\ref{thm:initial} can be found in the supplementary material. Finally, we conclude by proving Theorem~\ref{thm:main_theorem}.

\begin{definition}[\textbf{Tensor distance metric}]\label{de: em}
Given a tensor $\XC_\star \size^{I_1 \times I_2 \times I_3,}$ with tubal-rank $R$, and its skinny t-SVD represented as $\XC_\star = \UC_\star * \mathit{\Sigma}_\star *\VC_\star^{\To}$,  we define the low-rank factors of the ground truth as $\LC_\star:=\UC_\star*\mathit{\Sigma}_\star^{\frac{1}{2}}$ and $\RC_\star:=\VC_\star*\mathit{\Sigma}_\star^{\frac{1}{2}}$. 
Moreover, the set of invertible tensors
in $\mathbb{R}^{R\times R \times I_3}$ is denoted by $\GL(R)$.
For theoretical analysis, we define  the tensor distance metric $\dist(\LC,\RC;\LC_\star,\RC_\star)$ for the decomposed tubal rank-$R$ tensors as follows:
\begin{equation}\label{eq: em}
\resizebox{1\linewidth}{!}{$\displaystyle \inf_{\QC\in \GL(R)} \left(\|(\LC*\QC-\LC_\star)*\mathit{\Sigma}_\star^{\frac{1}{2}}\|_{\mathrm{F}}^{2} + \| (\RC*\QC^{-\top}-\RC_\star)*\mathit{\Sigma}_\star^{\frac{1}{2}} \|_{\mathrm{F}}^{2}\right)^{\frac{1}{2}}$}.
\end{equation}
\end{definition}
Notice that the optimal alignment tensor $\QC$ exists and is invertible if $\LC$ and $\RC$ are sufficiently close to $\LC_\star$ and $\RC_\star$. In particular, one can have the following lemma. 

\begin{lemma}\label{lm: Q_existence}
For any $\LC\in\mathbb{R}^{I_1\times R \times I_3}$ and $\RC\in \mathbb{R}^{I_2 \times R \times I_3}$, if 
\begin{align}\label{eq:Q_existence_condition}
     \dist(\LC,\RC;\LC_\star,\RC_\star) \leq \frac{\epsilon}{\sqrt{I_3}} \sigma_{\min}\left(\mathcal{X}_{\star}\right),
\end{align}
for $0<\epsilon<1$, then the optimal alignment tensor $\QC$ between $[\LC,\RC]$ and $[\LC_\star,\RC_\star]$ exists and is invertible. 
\end{lemma}

\begin{lemma}\label{lemma:matrix2factor} For any $\LC_{k}\in\mathbb{R}^{I_1\times R \times I_3}$ and $\RC_{k}\in \mathbb{R}^{I_2 \times R \times I_3}$,
if 
$\distk \leq \frac{\epsilon}{\sqrt{I_3}} \sigma_{\min}\left(\mathcal{X}_{\star}\right)$ and $0<\epsilon<1$, then one has
\begin{gather*}
    \| \mathcal{L}_k*\mathcal{R}_k^{\top} - \mathcal{X}_\star \|_{\fro} \leq \left(1+\frac{\epsilon}{2}\right)\sqrt{2}\  \distk.
\end{gather*}
\end{lemma}

\begin{lemma} \label{lm:sparity}
At the $(k+1)$-th iteration of Algorithm~\ref{Alg: LGRTPCA1},  setting the threshold value as $\zeta_{k+1} \geq\|\mathcal{X}_\star-\mathcal{X}_{k}\|_\infty$ yields
\begin{gather*}
    \supp(\mathcal{S}_{k+1})\subseteq \supp(\mathcal{S}_\star),\\
    \|\mathcal{S}_\star-\mathcal{S}_{k+1}\|_\infty \leq \|\mathcal{X}_\star-\mathcal{X}_{k}\|_\infty + \zeta_{k+1} \leq 2 \zeta_{k+1}.
\end{gather*}
\end{lemma}

\begin{theorem}[\textbf{Local linear convergence}] \label{thm:local convergence}
Suppose that $\mathcal{X}_\star=\mathcal{L}_\star*\mathcal{R}_\star^\top$ is a tubal rank-$R$ tensor with $\mu$-incoherence and $\mathcal{S}_\star$ is an $\alpha$-sparse tensor with $\alpha\leq\frac{1}{10^4\mu R^{1.5}{I_3}^{1.5}}$. Let $\mathcal{Q}_k$ be the optimal alignment tensor between $[\mathcal{L}_k,\mathcal{R}_k]$ and $[\mathcal{L}_\star,\mathcal{R}_\star]$. If the initial guesses obey the conditions
\begin{gather*}
 \distzero \leq \frac{\varepsilon}{\sqrt{I_3}}  \sigma_{\min}(\mathcal{X}_\star), \cr
\sqrt{I_1}\|(\mathcal{L}_0*\mathcal{Q}_0-\mathcal{L}_\star)*\mathit{\Sigma}_\star^{\frac{1}{2}}\|_{2,\infty}  \lor \cr \sqrt{I_2} \|(\mathcal{R}_0*\mathcal{Q}_0^{-\top}-\mathcal{R}_\star)*\mathit{\Sigma}_\star^{\frac{1}{2}}\|_{2,\infty}
    \leq \sqrt{\mu R} \sigma_{\min}(\mathcal{X}_\star)
\end{gather*}
with $\varepsilon:=0.02$, then by setting the thresholding values $\zeta_{k+1} = \tau \zeta_k$ for $k\geq 1$ and the fixed step size $\eta_k=\eta\in[\frac{1}{4},\frac{2}{3}]$, the iterates of Algorithm~\ref{Alg: LGRTPCA1} satisfy
\begin{gather*}
\distk\leq \frac{\varepsilon }{\sqrt{I_3}}\tau^k \sigma_{\min}(\mathcal{X}_\star), \cr
\sqrt{I_1}\|(\mathcal{L}_k*\mathcal{Q}_k-\mathcal{L}_\star)*\mathit{\Sigma}_\star^{\frac{1}{2}}\|_{2,\infty}  \lor \cr \sqrt{I_2} \|(\mathcal{R}_k * \mathcal{Q}_k^{-\top}-\mathcal{R}_\star)*\mathit{\Sigma}_\star^{\frac{1}{2}}\|_{2,\infty}  
\leq \sqrt{\mu R} \tau^k \sigma_{\min}(\mathcal{X}_\star), \cr
\|\mathcal{X}_\star-\mathcal{X}_k\|_\infty \leq \frac{3}{\sqrt{I_1I_2}} \mu R \tau^k \sigma_{\min}(\mathcal{X}_\star), 
\end{gather*}
where the convergence rate $\tau:=1-0.6\eta$.
\end{theorem}

\begin{theorem}[\textbf{Guaranteed initialization}] \label{thm:initial}
Suppose that $\mathcal{X}_\star=\mathcal{L}_\star*\mathcal{R}_\star^\top$ is a tubal rank-$R$ tensor with $\mu$-incoherence and $\mathcal{S}_\star$ is an $\alpha$-sparse tensor with $\alpha\leq\frac{c_0}{\mu R^{1.5}{I_3}^{1.5}\kappa}$ for some small positive constant $c_0\leq\frac{1}{38}$. Let $\mathcal{Q}_0$ be the optimal alignment tensor between $[\mathcal{L}_0,\mathcal{R}_0]$ and $[\mathcal{L}_\star,\mathcal{R}_\star]$. By setting the thresholding values $\left\|\mathcal{X}_{\star}\right\|_{\infty} \leq \zeta_0 \leq 2\left\|\mathcal{X}_{\star}\right\|_{\infty}$ and using the spectral initialization step, the initial guesses satisfy
\begin{gather*}
\distzero \leq \frac{14c_0}{\sqrt{I_3}} \sigma_{\min}(\mathcal{X}_\star), \cr
\sqrt{I_1}\|(\mathcal{L}_0*\mathcal{Q}_0-\mathcal{L}_\star)*\mathit{\Sigma}_\star^{\frac{1}{2}}\|_{2,\infty}  \lor \cr\sqrt{I_2} \|(\mathcal{R}_0*\mathcal{Q}_0^{-\top}-\mathcal{R}_\star)*\mathit{\Sigma}_\star^{\frac{1}{2}}\|_{2,\infty}
    \leq  \sqrt{\mu R} \sigma_{\min}(\mathcal{X}_\star).
\end{gather*}
%as long as $c_0\leq\frac{1}{35}$.
%with $\varepsilon:=0.02$.
\end{theorem}
Theorem~\ref{thm:initial} demonstrates that the spectral initialization step in Algorithm~\ref{Alg: LGRTPCA1}  will deduce the initial conditions necessary for Theorem~\ref{thm:local convergence}.

\begin{proof}[\textbf{Proof of Theorem~~\ref{thm:main_theorem}}]
Setting \( c_0 = 10^{-4} \) in Theorem~\ref{thm:initial}, its results will satisfy the required conditions of Theorem~\ref{thm:local convergence}, yielding
\begin{gather*}
    \distk \leq \frac{0.02}{\sqrt{I_3}} (1 - 0.6\eta)^k \sigma_{\min}(\mathcal{X}_\star), \\
    \sqrt{I_1}\|(\mathcal{L}_k *\mathcal{Q}_k - \mathcal{L}_\star)*\mathit{\Sigma}_\star^{\frac{1}{2}}\|_{2,\infty} \lor \cr
    \resizebox{1\linewidth}{!}{$\sqrt{I_2}\|(\mathcal{R}_k *\mathcal{Q}_k^{-\top} - \mathcal{R}_\star)*\mathit{\Sigma}_\star^{\frac{1}{2}}\|_{2,\infty} 
    \leq \sqrt{\mu R} (1 - 0.6\eta)^k \sigma_{\min}(\mathcal{X}_\star)$}, 
\end{gather*}
\begin{gather*}
    \|\mathcal{X}_\star - \mathcal{X}_k\|_\infty \leq \frac{3}{\sqrt{I_1 I_2}} \mu R \tau^k \sigma_{\min}(\mathcal{X}_\star),
\end{gather*}
for all \( k \geq 0 \). By Lemma~\ref{lm:sparity}, we further obtain
\begin{gather*}
    \|\mathcal{S}_\star - \mathcal{S}_k\|_\infty \leq \frac{6}{\sqrt{I_1 I_2}} \mu R \tau^{k-1} \sigma_{\min}(\mathcal{X}_\star).
\end{gather*}
Additionally, using Lemma~\ref{lemma:matrix2factor}, we have
\begin{gather*}
    \|\mathcal{X}_k - \mathcal{X}_\star\|_{\fro} \leq \left(1 + \frac{\epsilon}{2}\right) \sqrt{2} \distk
\end{gather*}
as long as \( \distk \leq \frac{\epsilon}{\sqrt{I_3}} \sigma_{\min}(\mathcal{X}_\star) \). With \( \epsilon = 0.02 \), we establish the three terms of the first claim.

For \( k \geq 1 \), the second claim $\operatorname{supp}\left(\mathcal{S}_k\right) \subseteq \operatorname{supp}\left(\mathcal{S}_{\star}\right)$ follows directly from Lemma~\ref{lm:sparity}. For \( k = 0 \), we initialize \( \mathcal{X}_{-1} = \bm{0} \), leading to
\[
\mathcal{S}_0 = \operatorname{T}_{\zeta_0}(\mathcal{Y}) = \operatorname{T}_{\zeta_0}(\mathcal{Y} - \mathcal{X}_{-1}),
\]
where \( \|\mathcal{X}_\star\|_\infty \leq \zeta_0 \leq 2 \|\mathcal{X}_\star\|_\infty = 2 \|\mathcal{X}_\star - \mathcal{X}_{-1}\|_\infty \). Applying Lemma~\ref{lm:sparity} again, we confirm the second claim for all \( k \geq 0 \).

This completes the proof.
\end{proof}

\section{Experiments} \label{section_experiment}
In this section, we first conduct experiments on synthetic data to validate our theoretical results. Specifically, we present the phase transition performance of the proposed RTPCA-SGD method under different condition numbers $\kappa$.
We then analyze several key aspects of the method, including its linear convergence rate, independence from the condition number, the impact of the decay rate parameter, and the effectiveness of the deep unfolding approach. Additionally, we compare the phase transition performance of the proposed RTPCA-SGD method with that of the convex RTPCA-TNN model \cite{lu2019tensor}, both of which have theoretical guarantees for exact recovery in the t-SVD-based RTPCA problem.
Finally, we evaluate our methods on real-world datasets including video denoising and background initialization tasks, demonstrating their practical effectiveness and computational efficiency.

\subsection{Synthetic data}
For the synthetic experiments, the observed data $\mathcal{Y} \in \mathbb{R}^{I_1 \times I_2 \times I_3}$ is constructed as $\mathcal{Y} = \mathcal{X}_\star + \mathcal{S}_\star$, where $\mathcal{X}_\star$ is a low-rank tensor with tubal rank $R$ and condition number $\kappa$, and $\mathcal{S}_\star$ is an $\alpha$-sparse tensor. The procedure for generating $\mathcal{X}_\star$ is outlined as follows: 
\begin{itemize}
\item Randomly generate two factor tensors, $\mathcal{L} \in \mathbb{R}^{I_1 \times R \times I_3}$ and $\mathcal{R} \in \mathbb{R}^{R \times I_2 \times I_3}$, with entries independently sampled from a standard normal distribution.
\item Construct a preliminary low-rank tensor $\mathcal{X}_0$ as $\mathcal{X}_0 = \mathcal{L} * \mathcal{R}$, resulting in a tensor with tubal rank $R$ but an undefined condition number.

\item Perform a skinny t-SVD on $\mathcal{X}_0$ to obtain orthogonal tensors $\mathcal{U} \in \mathbb{R}^{I_1 \times R \times I_3}$ and $\mathcal{V} \in \mathbb{R}^{I_2 \times R \times I_3}$.

\item Generate an $f$-diagonal tensor $\mathit{\Sigma} \in \mathbb{R}^{R \times R \times I_3}$ with the desired condition number $\kappa$. First, construct a Fourier-domain tensor $\widehat{\mathit{\Sigma}}$, setting the diagonal elements of each slice $\widehat{\mathbf{\Sigma}}^{(i_3)}$ (for $i_3 = 1, \dots, \lceil \frac{I_3 + 1}{2} \rceil$) to be linearly spaced between $\left(\frac{1}{2}\right)^{i_3-1}$ and $\frac{1}{\kappa}$. Then, generate the remaining slices using the conjugate symmetry property, \textit{i.e.}, $\widehat{\mathbf{\Sigma}}^{(i_3)} = \operatorname{conj}(\widehat{\mathbf{\Sigma}}^{(I_3 - i_3 + 2)})$. Finally, apply an inverse FFT along the third mode to obtain $\mathit{\Sigma}$.

\item Compute $\mathcal{X}_\star = \mathcal{U} * \mathit{\Sigma} * \mathcal{V}^{\mathrm{T}}$, yielding a low-rank tensor with tubal rank $R$ and condition number $\kappa$.
\end{itemize}

In addition, we randomly generate an $\alpha$-sparse tensor $\mathcal{S}_\star$ by sampling $\alpha$-fraction entries from a uniform distribution within the range $[-\theta, \theta]$, where $\theta = \frac{\norm{\mathcal{X}_\star}_{1}}{I_1 I_2 I_3}$ represents the mean entry-wise magnitude of $\mathcal{X}_\star$. Finally, the synthetic data is obtained as $\mathcal{Y} = \mathcal{X}_\star + \mathcal{S}_\star$.

To evaluate recovery accuracy, we use the relative standard error (RSE) metric: 
\begin{equation}
    \text{RSE}_{\mathcal{X}} = \frac{||\mathcal{X}-\mathcal{X}_\star||_{\mathrm{F}} }{||\mathcal{X}_\star||_{\mathrm{F}}},
\label{RSE}
\end{equation}
where $\mathcal{X}$ represents the low-rank estimate obtained from the method, and $\mathcal{X}_\star$ is the ground-truth of low-rank component.

\begin{figure}[!t]
% \tiny
% \scriptsize
% \footnotesize
\Huge
\setlength{\tabcolsep}{0.1pt}
\begin{center}
\resizebox{0.5\textwidth}{!}{\begin{tabular}{cccc}
\includegraphics[width=0.5\textwidth]{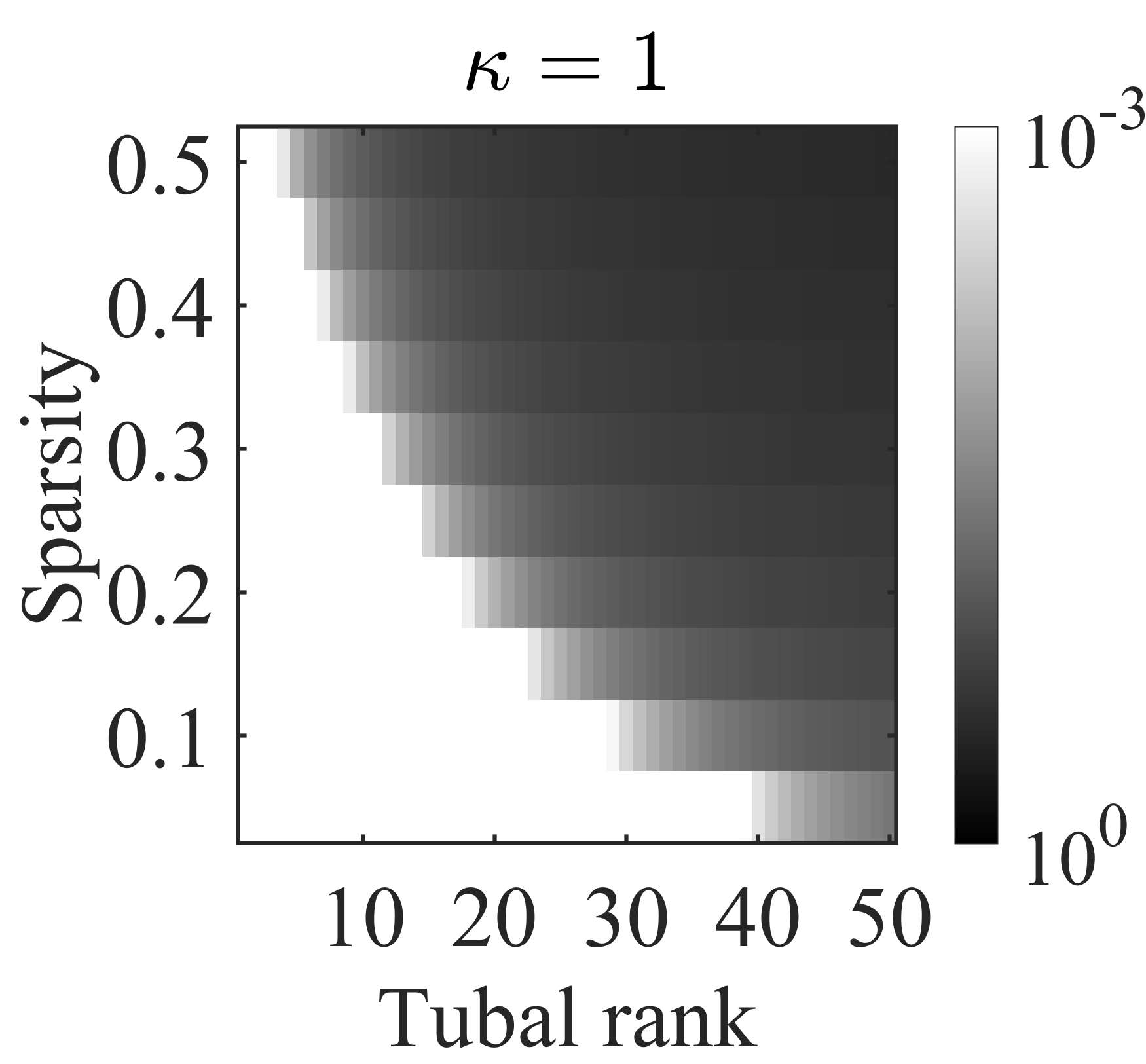}&
\includegraphics[width=0.5\textwidth]{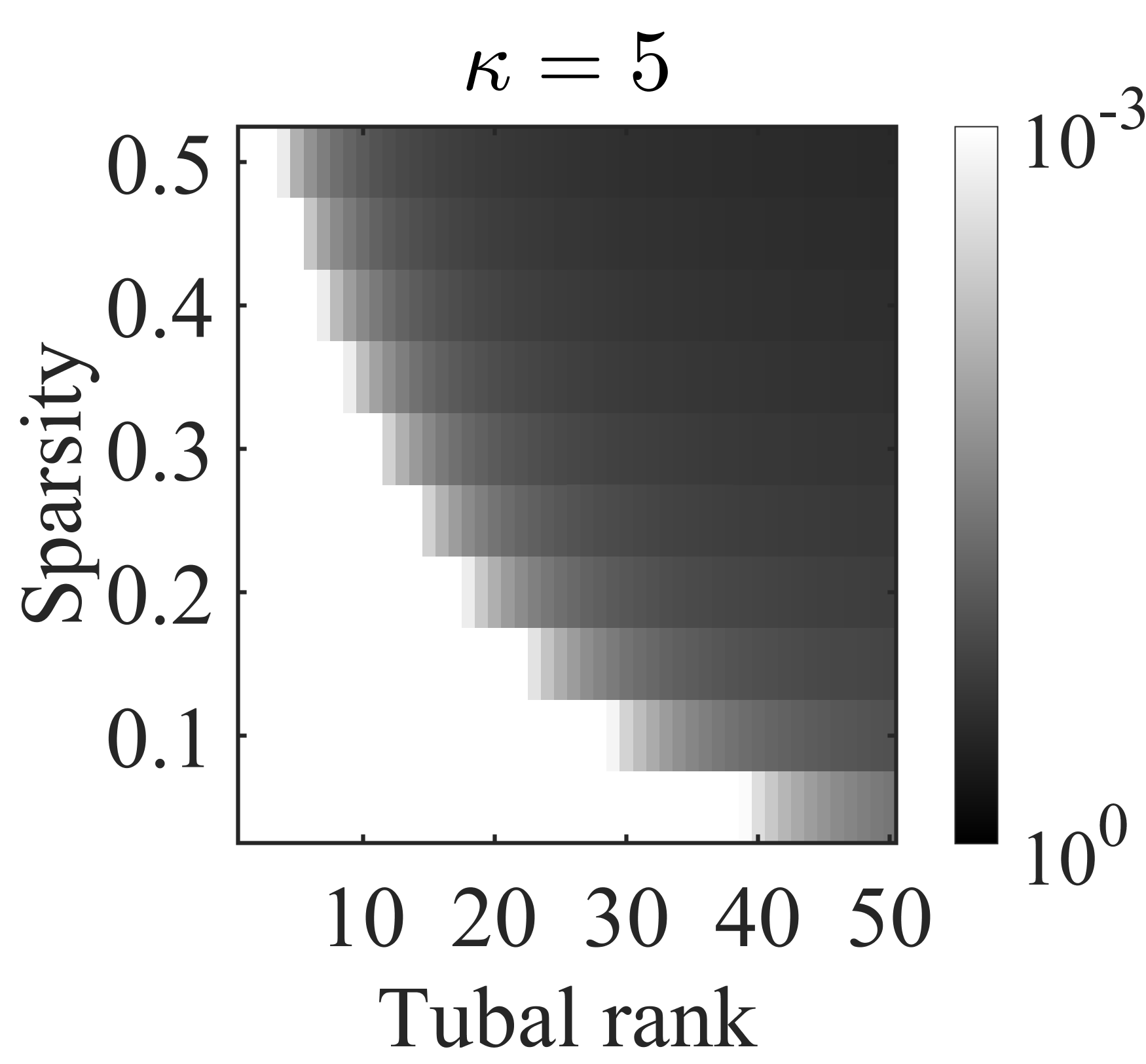}&
\includegraphics[width=0.5\textwidth]{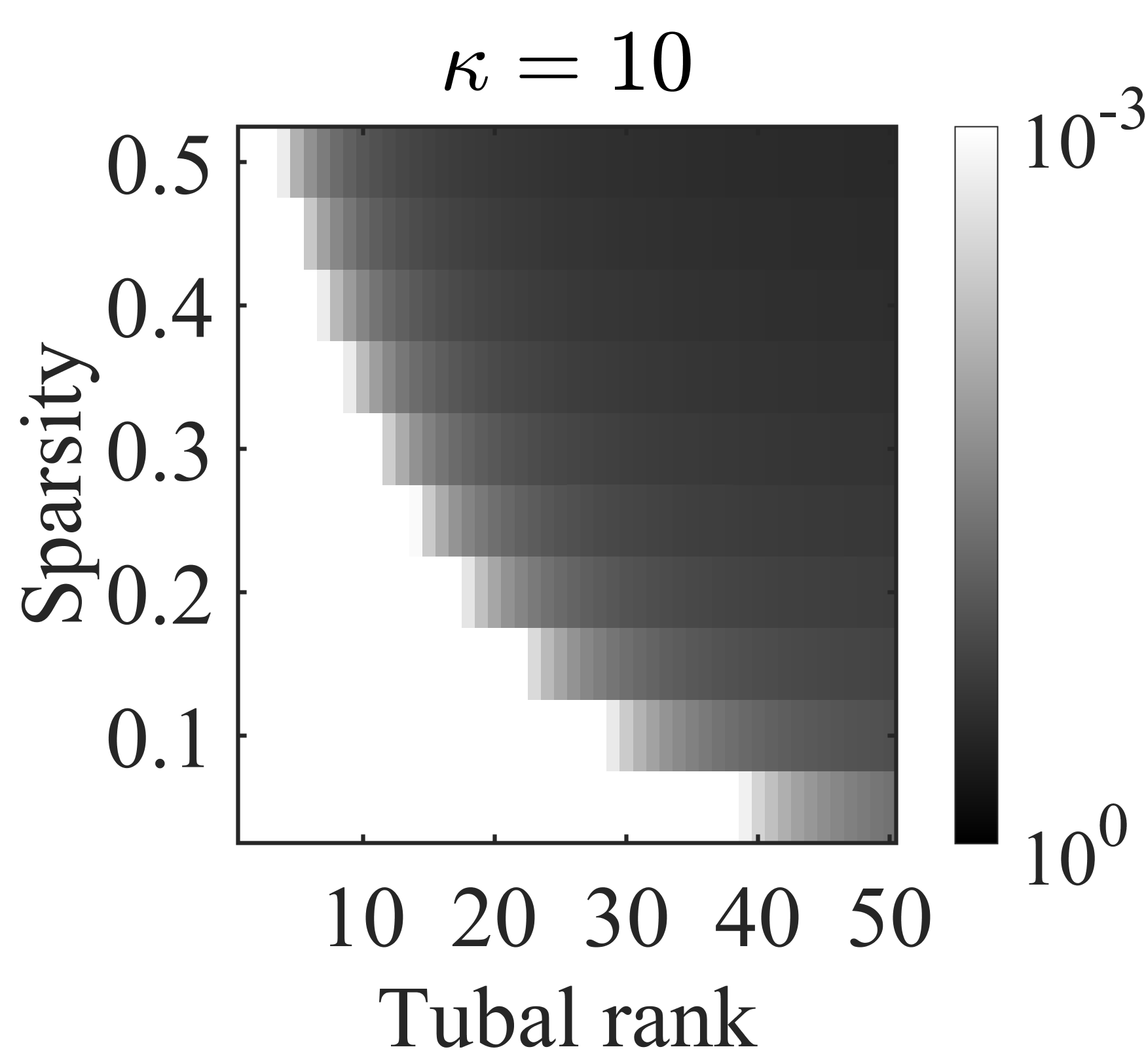}
 \\(a) &(b) &(c)
% \\(a) $\kappa = 1$&(b) $\kappa = 5$&(c) $\kappa = 10$
\end{tabular}}
% \vspace{-0.2cm}
\caption{Phase transition performance of  RTPCA-SGD under different condition numbers. Here $I_1 = I_2 = 100, I_3 = 50$.} 
\label{fig: Phasetransitions}
  \end{center}
  \vspace{-0.2cm}
\end{figure}

\subsubsection{Phase transition}
We begin with evaluating the phase transition performance of the proposed RTPCA-SGD and RTPCA-TNN methods with respect to tubal rank $R$ and sparsity level $\alpha$ under varying condition numbers. Specifically, we set $I_1 = I_2 = 100, I_3 = 50$, the iteration number $K = 100$, and run 10 random experiments for each pair of $(R, \alpha)$. The low-rank component is considered successfully recovered if $\text{RSE}_{\mathcal{X}} \leq 10^{-3}$. Fig.~\ref{fig: Phasetransitions} shows the median of $\text{RSE}_{\mathcal{X}}$ for each pair, where the whiter pixels indicate more successful recovery. It can be seen that exact recovery is achieved when both the tubal rank and sparsity are sufficiently small, which is consistent with Claim 1 of Theorem~\ref{thm:main_theorem}.

\begin{figure}[!t]
% \tiny
% \scriptsize
% \footnotesize
\Huge
\setlength{\tabcolsep}{0.1pt}
\begin{center}
\resizebox{0.495\textwidth}{!}{\begin{tabular}{cccc}
\includegraphics[width=0.49\textwidth]{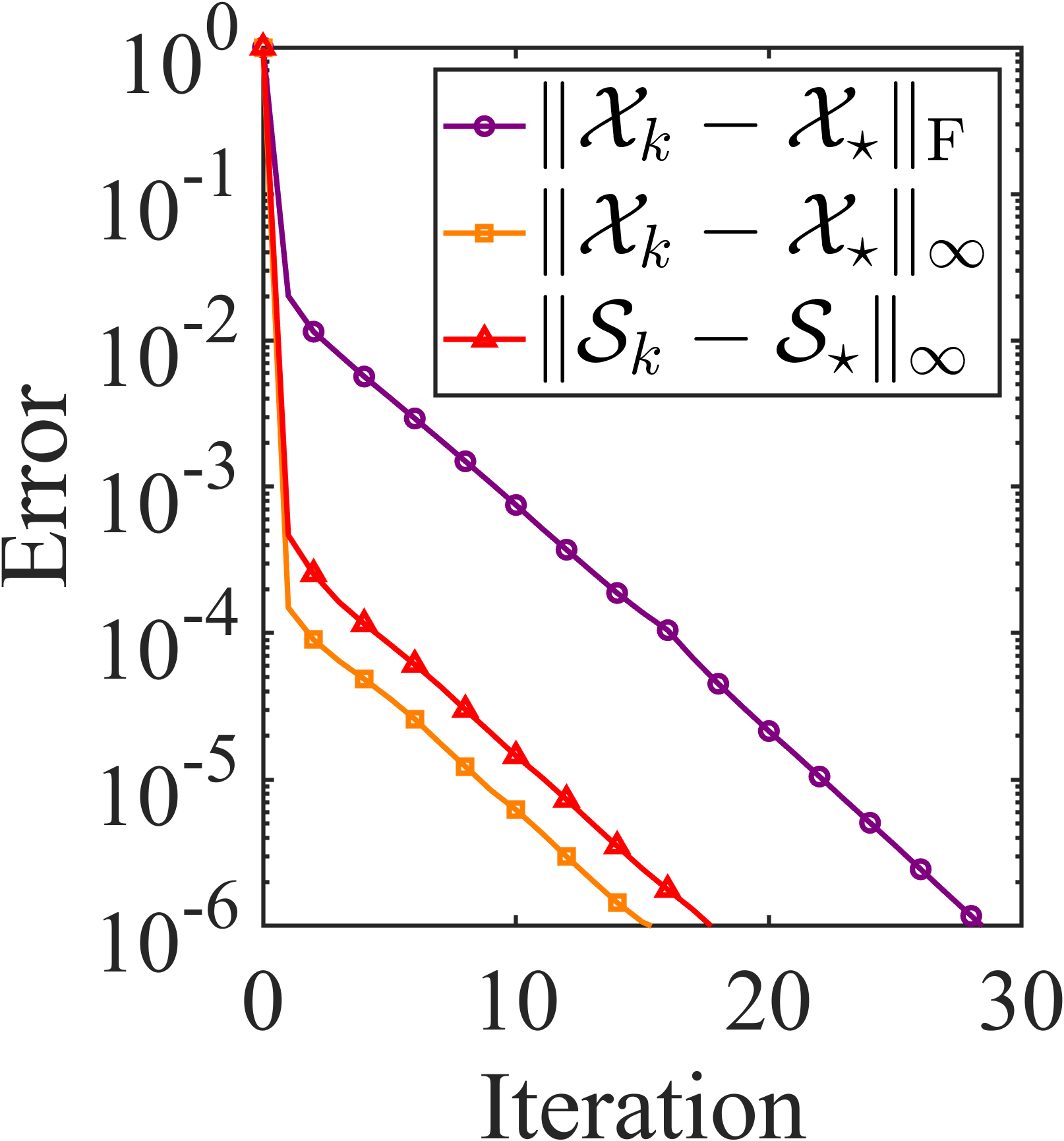}&
\includegraphics[width=0.475\textwidth]{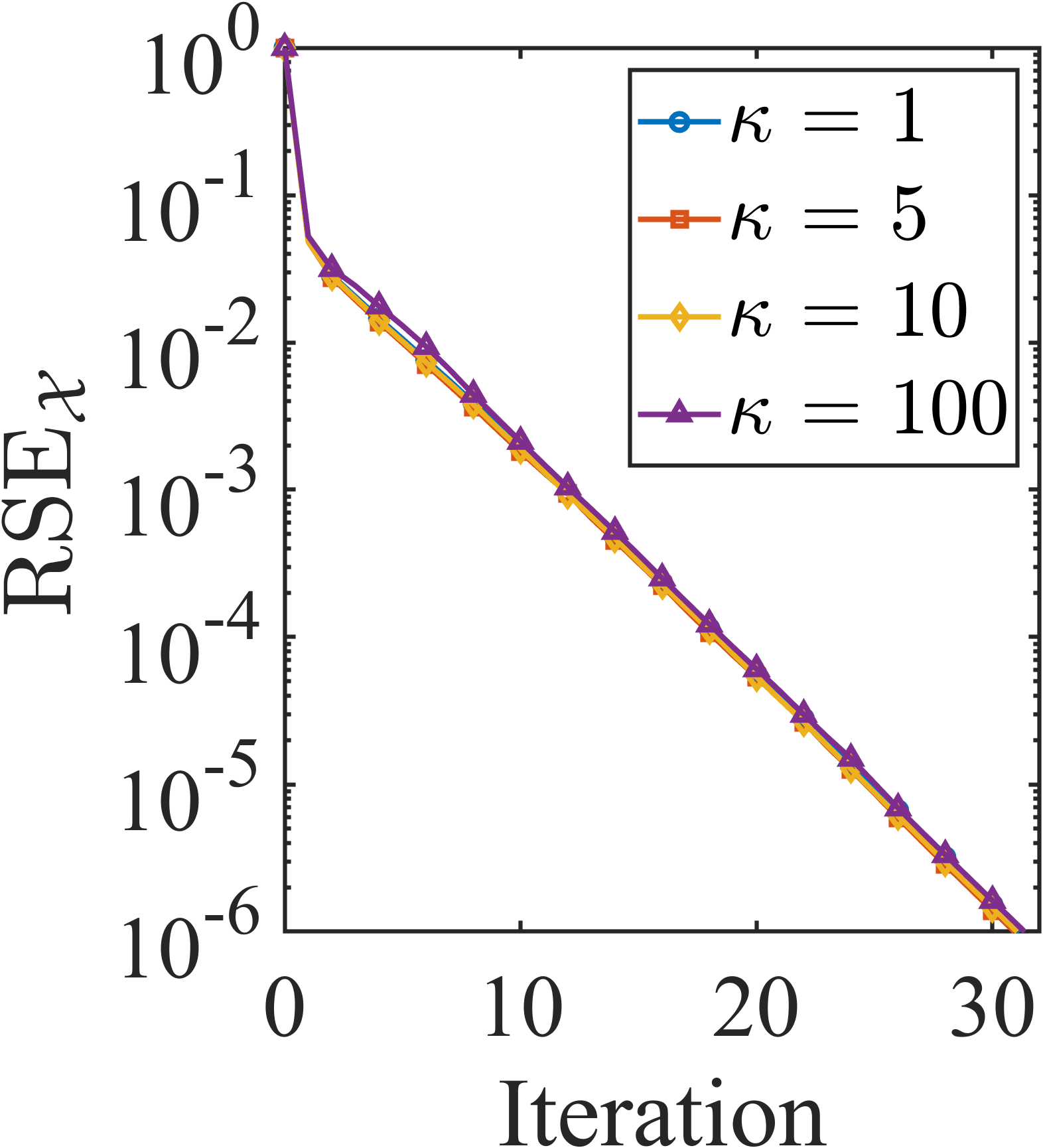}&
\ \includegraphics[width=0.5\textwidth]{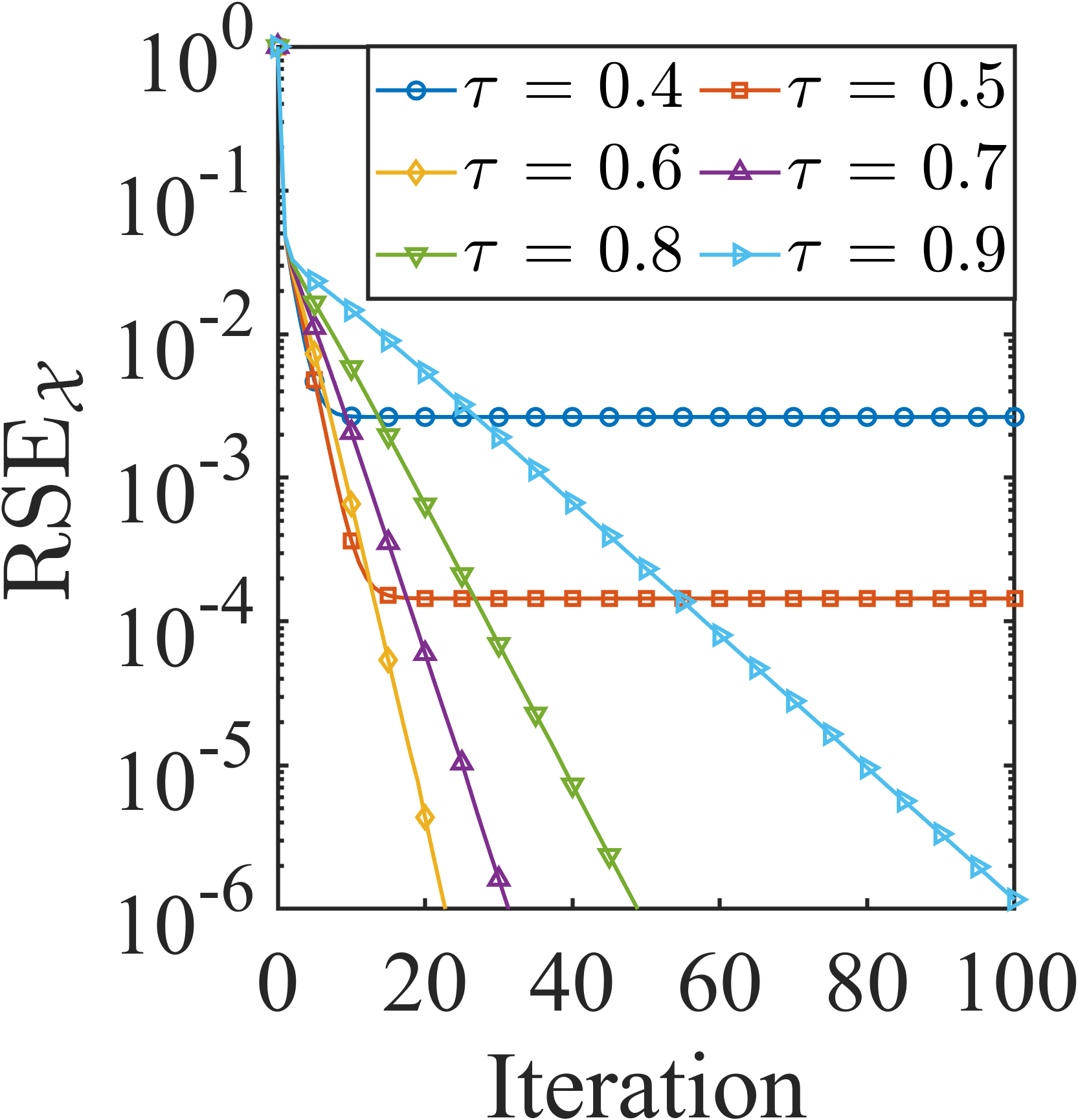}\\ \qquad (a) Different metrics & \qquad(b) Influence of  $\kappa$ & \qquad(c) Influence of $\tau$ 
\end{tabular}}
% \vspace{-0.2cm}
\caption{Convergence performance with respect to different error metrics, condition numbers and decay rates.} 
\label{fig:lcr}
  \end{center}
  \vspace{-0.1cm}
\end{figure}

\subsubsection{Linear convergence rate}
As shown in Eq.~\eqref{eq: con3}, Algorithm  \ref{Alg: LGRTPCA1} converges linearly at a constant rate with respect to three error metrics: $\left\|\mathcal{X}_k-\mathcal{X}_{\star}\right\|_{\mathrm{F}}, \left\|\mathcal{X}_k-\mathcal{X}_{\star}\right\|_{\infty}, \left\|\mathcal{S}_k-\mathcal{S}_{\star}\right\|_{\infty}$. To illustrate this, we set $I_1 = I_2 = 100$, $I_3 = 50$, $\alpha = 0.1$, $R = 5$, $\kappa = 5$ to conduct the experiments. The convergence curves for the different error metrics are plotted in Fig.~\ref{fig:lcr}(a). As observed, all three metrics exhibit linear convergence with the same convergence rate, thereby confirming  Claim 2 of Theorem~\ref{thm:main_theorem}.

\subsubsection{Independence of the condition number}
To validate Claim 3 in Theorem~\ref{thm:main_theorem} that the recovery performance of the proposed algorithm is independent of the condition number, we examine its phase transition performance and convergence behavior across various condition numbers, as shown in Fig.~\ref{fig: Phasetransitions} and in Fig.~\ref{fig:lcr}(b).
The results indicate that the exact recovery performance remains unaffected by the condition number, and the algorithm maintains a consistent linear convergence rate regardless of the condition number.

\subsubsection{Influence of the decay rate}
We further investigate the influence of the decay rate $\tau$ on the thresholding parameters. According to Theorem~\ref{thm:main_theorem}, $\tau$ also determines the convergence rate, with smaller values leading to faster convergence.
In our experiments, we set $I_1 = I_2 = 100$, $I_3 = 50$, $\alpha = 0.1$, $R = 5, \kappa = 5$. As illustrated in Fig.~\ref{fig:lcr}(c), the proposed method achieves exact recovery over a wide range of $\tau$ values, provided that $\tau$ is not too small. Moreover, larger values of $\tau$ result in slower convergence speeds, indicating a trade-off between accuracy and speed. 
Consequently, we choose $0.7 \leq \tau \leq 0.9$ for our experiments.

\subsubsection{Compared with RTPCA-TNN}
We also compare the proposed RTPCA-SGD method with RTPCA-TNN \cite{lu2019tensor}.
Fig.~\ref{fig: Phasetransition_com}(a) and Fig.~\ref{fig: Phasetransition_com}(b) shows their phase transition comparison when we set $I_1 = I_2 = 100, I_3 = 10, \kappa =5$.
We can see that the white/successful region of our method is larger than that of RTPCA-TNN, which shows the superior recovery performance of the proposed RTPCA-SGD.

\begin{figure}[!t]
% \tiny
% \scriptsize
% \footnotesize
\Huge
\setlength{\tabcolsep}{0.1pt}
\begin{center}
\resizebox{0.5\textwidth}{!}{\begin{tabular}{cccc}
\includegraphics[width=0.5\textwidth]{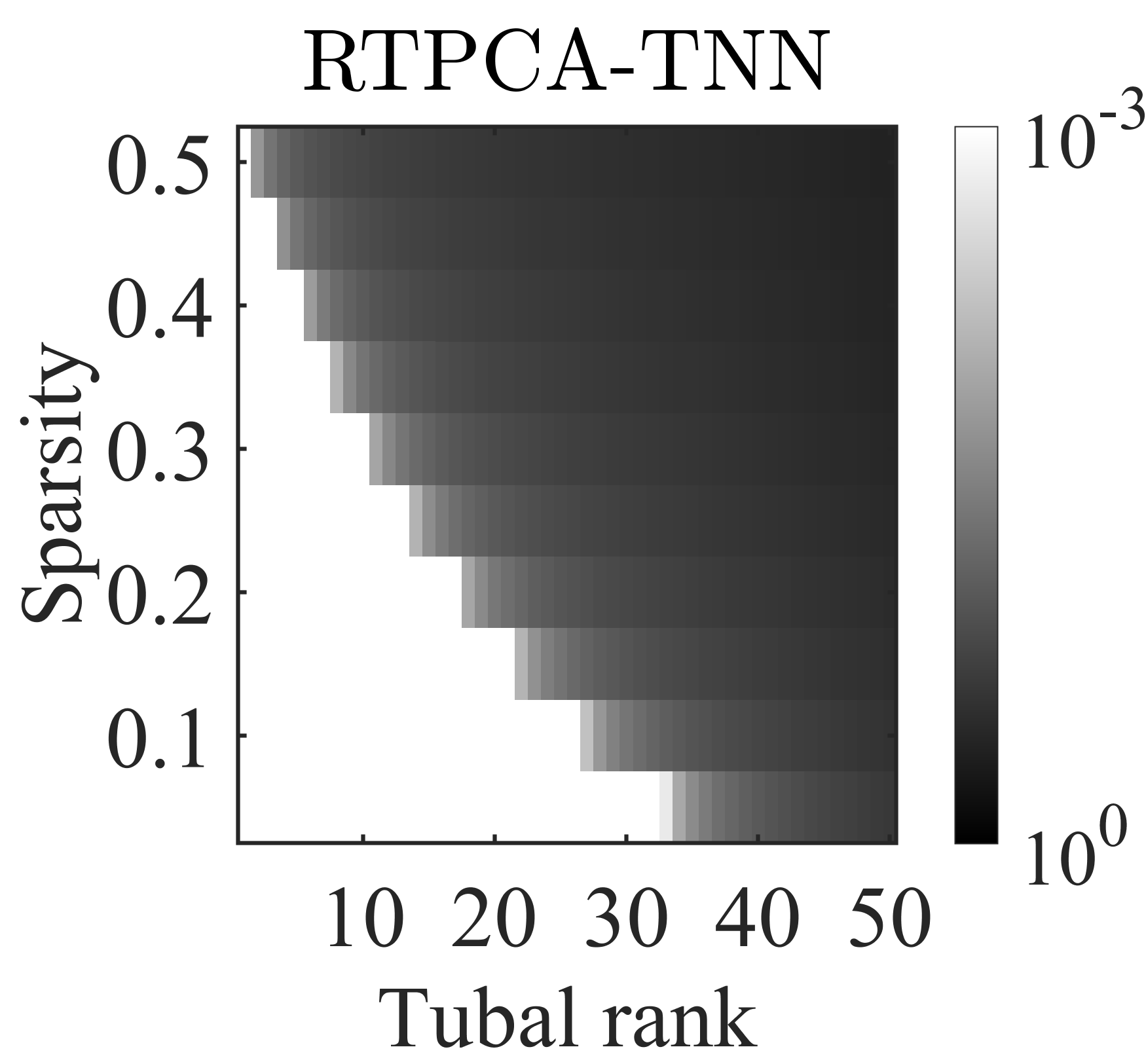}&
\includegraphics[width=0.5\textwidth]{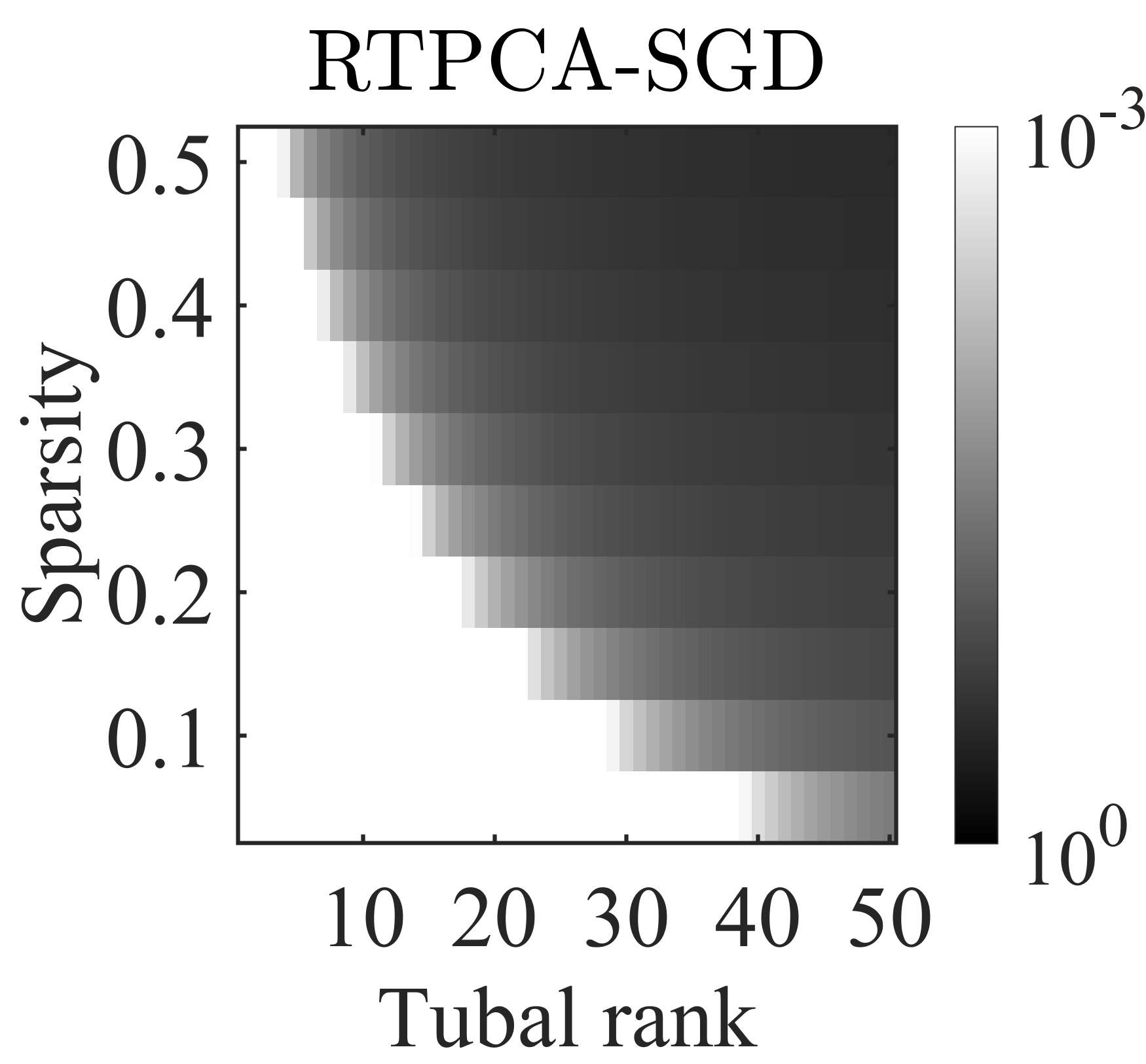}&
\includegraphics[width=0.5\textwidth]{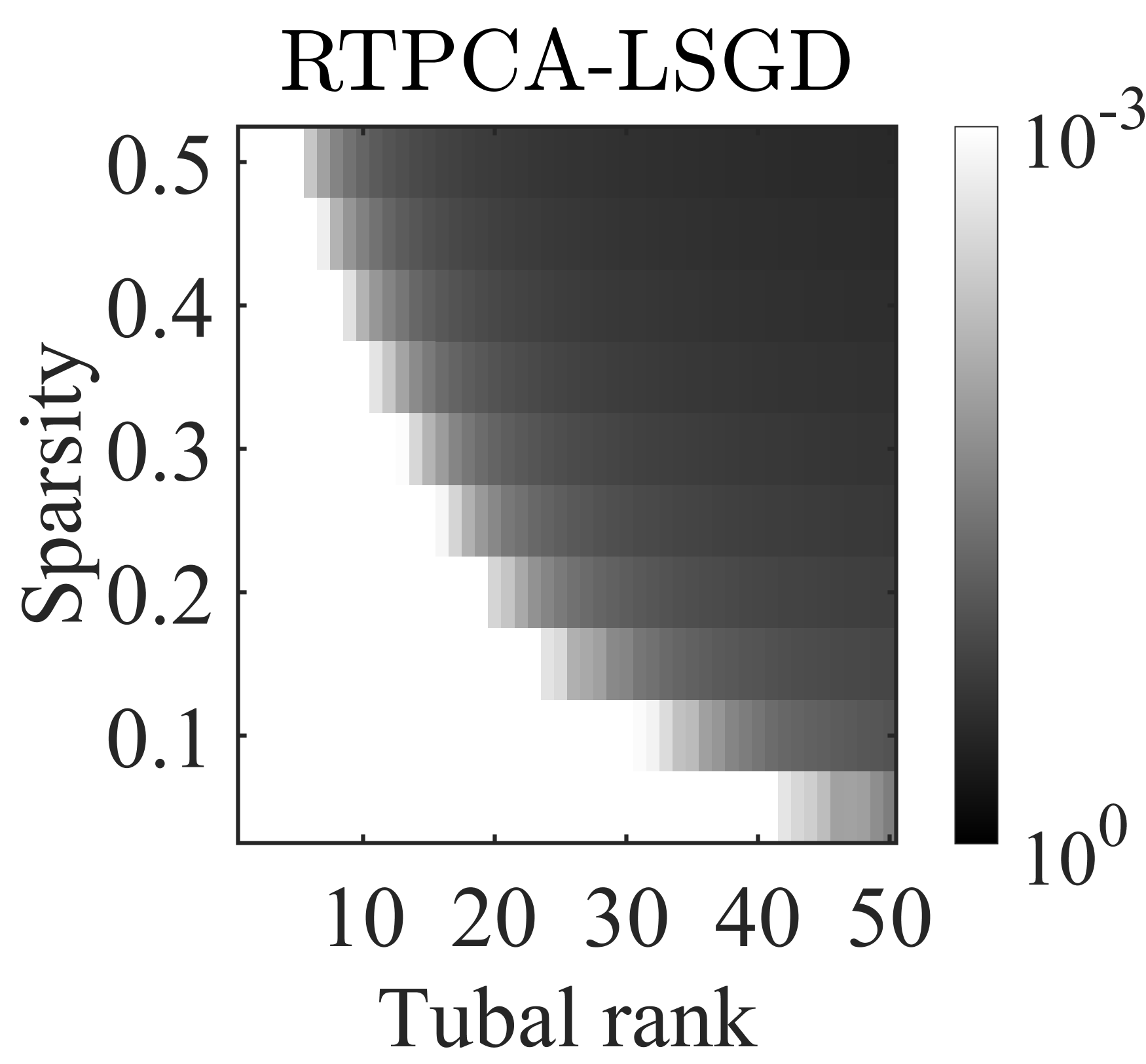}
 \\(a) &(b) &(c)
\end{tabular}}
% \vspace{-0.2cm}
\caption{Phase transition comparison of different methods. Here $I_1 = I_2 =  100, I_3 = 10$, $\kappa = 5$.} 
\label{fig: Phasetransition_com}
  \end{center}
  \vspace{-0.5cm}
\end{figure}

\subsubsection{Effectiveness of deep unfolding}
For this experiment, we set the same iteration number $K= 100$. As shown in Fig.~\ref{fig: Phasetransition_com}(b) and Fig.~\ref{fig: Phasetransition_com}(c), RTPCA-LSGD obtains better recovery performance than RTPCA-SGD. This result demonstrates the effectiveness of the deep unfolding approach for parameter optimization, further enhancing the recovery performance.

\subsection{Real data}
\subsubsection{Video denoising}
In this section, we evaluate the effectiveness of the proposed method through video denoising experiments.
Specifically, we use two video datasets with dimensions $144 \times 176 \times 30$, ``Akiyo" and ``News” \footnote{http://trace.eas.asu.edu/yuv/index.html}. The proposed RTPCA-SGD method is compared against three state-of-the-art approaches: RPCA-SGD~\cite{tong2021accelerating}, Tucker-SGD~\cite{dong2023fast}, and RTPCA-TNN~\cite{lu2019tensor}. 
For RTPCA-SGD, the parameters are set as $R = 3$, $K = 50$, $\tau = 0.8$, $\zeta_0 = 1$, $\zeta_1 = 1$.
All experiments are conducted on a computer equipped with an Intel i5-11400F CPU, 16GB of RAM, and an RTX 3060 Ti GPU.
Table~\ref{Tab: VDresults} presents the comparison results of all methods in terms of PSNR, RSE and CPU time on videos corrupted with  10\% and 30\% salt and pepper noise.
The results demonstrate that RTPCA-SGD achieves superior recovery accuracy compared to the state-of-the-art methods. Although RTPCA-SGD requires more CPU time than RPCA-SGD and Tucker-SGD due to the overhead of tensor computations, it remains faster than the classical RTPCA-TNN method because it eliminates the need for full t-SVD computations. This observation aligns with the computational complexity analysis presented in Section~\ref{sec: Computational complexity}.
Fig.~\ref{fig: VDresults} provides two visual examples, showing that the proposed method effectively recovers more facial and background details compared to other approaches.

\renewcommand{\arraystretch}{0.9}
\begin{table*}[!ht]
    \centering
    \small
    \caption{Comparison of different evaluation metrics on videos with different noise rates.}
    % \vspace{-0.1cm}
    \label{Tab: VDresults}
    \setlength{\tabcolsep}{4.1pt} % 设置列间距
    \begin{tabular}{@{}ccccccccccccc@{}}
    \toprule
    \multirow{3}{*}{Methods} & \multicolumn{6}{c}{Akiyo} & \multicolumn{6}{c}{News} \\ \cmidrule(l){2-7} \cmidrule(l){8-13}
    & \multicolumn{3}{c}{$\alpha=0.1$} & \multicolumn{3}{c}{$\alpha=0.3$} & \multicolumn{3}{c}{$\alpha=0.1$} & \multicolumn{3}{c}{$\alpha=0.3$} \\ \cmidrule(l){2-4} \cmidrule(l){5-7} \cmidrule(l){8-10} \cmidrule(l){11-13}
    & PSNR$\uparrow$ & RSE$\downarrow$ & Time(s)$\downarrow$ & PSNR$\uparrow$ & RSE$\downarrow$ & Time(s)$\downarrow$ & PSNR$\uparrow$ & RSE$\downarrow$ & Time(s)$\downarrow$ & PSNR$\uparrow$ & RSE$\downarrow$ & Time(s)$\downarrow$ \\ \midrule
    RPCA-SGD\cite{tong2021accelerating} & 39.7715 & 0.0387 & \textbf{0.61} & 31.2856 & 0.0863 & \textbf{0.66} & 32.3529 & 0.0682 & \textbf{0.57} & 27.9771 & 0.1216 & \textbf{0.58} \\
    Tucker-SGD\cite{dong2023fast} & 37.3051 & 0.0362 & 1.63 & 34.1089 & 0.0506 & 1.65 & 32.0452 & 0.0704 & 1.47 & 31.2059 & 0.0767 & 1.50 \\
    RTPCA-TNN\cite{lu2019tensor} & 36.7287 & 0.0371 & 6.78 & 34.0879 & 0.0488 & 6.53 & 32.1030 & 0.0691 & 6.38 & 29.1768 & 0.0964 & 6.49 \\
    RTPCA-SGD &  \textbf{41.8706} & \textbf{0.0206} & 2.83 &  \textbf{34.8449} &  \textbf{0.0451} & 2.80 &  \textbf{35.4187} & \textbf{0.0475} & 2.75 &  \textbf{31.9326} &  \textbf{0.0703} & 2.77 \\
    \bottomrule
    \end{tabular}
    % \vspace{-0.2cm}
\end{table*}

\begin{figure*}[!t]
% \tiny
% \scriptsize
\footnotesize
\setlength{\tabcolsep}{1pt}
\begin{center}
% \resizebox{\textwidth}{!}{\begin{tabular}{ccccccc}
\begin{tabular}{cccccc}
\includegraphics[width=0.163\textwidth]{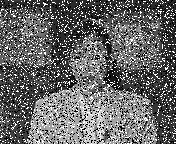}&
\includegraphics[width=0.163\textwidth]{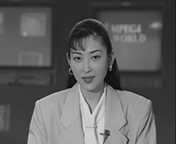}&
\includegraphics[width=0.163\textwidth]{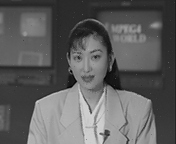}&
\includegraphics[width=0.163\textwidth]{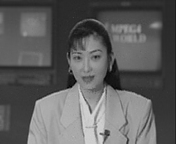}&
\includegraphics[width=0.163\textwidth]{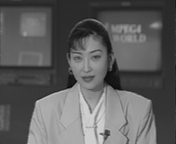}&
\includegraphics[width=0.163\textwidth]{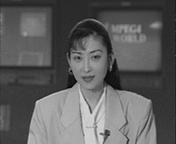}\\
\includegraphics[width=0.163\textwidth]{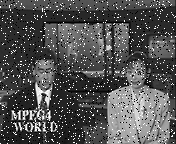}&
\includegraphics[width=0.163\textwidth]{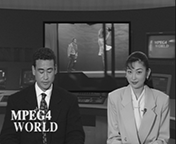}&
\includegraphics[width=0.163\textwidth]{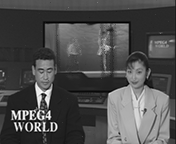}&
\includegraphics[width=0.163\textwidth]{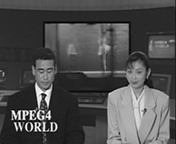}&
\includegraphics[width=0.163\textwidth]{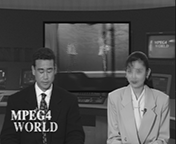}&
\includegraphics[width=0.163\textwidth]{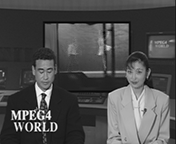}\\
Noisy & GT & RPCA-SGD\cite{tong2021accelerating}  & Tucker-SGD\cite{dong2023fast} & RTPCA-TNN\cite{lu2019tensor}  &   RTPCA-SGD 
\end{tabular}
% \vspace{-0.2cm}
\caption{Video denoising examples on the $11$-th frame of ``Akiyo'' with 30\% noise and the $24$-th frame of ``News'' with 10\%  noise. } 
\label{fig: VDresults}
  \end{center}
  % \vspace{-0.1cm}
\end{figure*}

\subsubsection{Background initialization}
We also evaluate the proposed methods on the background initialization task. 
In surveillance video sequences consisting of a series of frames, the background is typically modeled as a low-rank component, while the foreground target corresponds to a sparse component.
For RTPCA-SGD, the parameters are set as $R = 3$, $K = 50$, $\tau = 0.9$, $\zeta_0 = 1$, $\zeta_1 = 1$.
For this experiment, three scenes are selected from the classical SBI dataset\footnote{https://sbmi2015.na.icar.cnr.it/SBIdataset.html}\cite{bouwmans2017scene, jodoin2017extensive}: 
``HighwayI'' ($240\times320\times441$), ``IBMtest2''($256\times320\times91$), and ``HallAndMonitor'' ($240\times352\times297$).
We evaluate recovery quality using five widely adopted metrics from \cite{bouwmans2017scene}: Average Gray-level Error (AGE), Percentage of Error Pixels (pEPs), Percentage of Clustered Error Pixels (pCEPs), Multi-Scale Structural Similarity Index (MSSSIM), and Peak Signal-to-Noise Ratio (PSNR). 
For the first three metrics(AGE, pEPs and pCEPs), smaller values indicate better background recovery performance. In contrast, higher values of MSSSIM and PSNR reflect better recovery accuracy. Additionally, we compare the CPU time consumed by each method to assess computational efficiency. 
As shown in Table~\ref{tab: evaluation}, RTPCA-SGD achieves the best recovery accuracy across all scenes with competitive computational efficiency.
In particular, it consumes much less time than RTPCA-TNN thanks to the removal of full t-SVD computations at each iteration.
Figure~\ref{fig: BM} provides a visual and intuitive illustration of the recovered background images by different methods. 
Our method consistently removes foreground pixels effectively while preserving background information very well in all video sequences.

\begin{table*}
    	\centering
         \small
		\caption{Comparison of different evaluation metrics on background initialization.}
           \vspace{-0.1cm}
		\label{tab: evaluation}
             \setlength{\tabcolsep}{9.5pt} % 列间距
		\begin{tabular}{cccccccc}
  
			\toprule
			Scenes  & Methods & AGE$\downarrow$ & pEPs$\%$$\downarrow$ & pCEPs$\%$$\downarrow$ & MSSSIM$\uparrow$ & PSNR$\uparrow$ & CPU time (s)$\downarrow$\\
			\midrule
            \multirow{5}*{\begin{tabular}[c]{@{}c@{}}HighwayI\\($240\times320\times441$)\end{tabular}}  
            & RPCA-SGD\cite{tong2021accelerating}  &  7.5118 & 2.2206 &  1.0663 & 0.9664 &  29.2713&  \textbf{24.84}\\
            &Tucker-SGD\cite{dong2023fast}  &  2.6755 & 0.6551 &  0.3832 & 0.9779 &  36.8433&   65.37\\
            &RTPCA-TNN\cite{lu2019tensor}  &  12.3085 & 16.5169 &  13.1783 & 0.7001 &  20.6919&  361.40\\
            & {RTPCA-SGD  } &   \textbf{2.5416} &  \textbf{0.5415} &   \textbf{0.2533} &  \textbf{0.9857}  &   \textbf{37.1574} &86.78\\
             \midrule
            \multirow{5}*{\begin{tabular}[c]{@{}c@{}}IBMtest2\\($256\times320\times91$)\end{tabular}}  
           & RPCA-SGD\cite{tong2021accelerating}  &  6.2228 & 6.0283 &  3.3278 & 0.9790 &  28.8994&  \textbf{5.52}\\
            &Tucker-SGD\cite{dong2023fast}  &  4.6854 & 1.2920 &  0.7864 & 0.9838 &  32.0206&   12.39\\
            &RTPCA-TNN\cite{lu2019tensor}  &  7.5074 & 5.9249 &  4.4934 & 0.9322 &  23.6850&  54.56\\
            &{RTPCA-SGD  } &   \textbf{3.6856} &  \textbf{0.1162} &   \textbf{0.0011} & \textbf{0.9955}  &   \textbf{35.6996} &15.79\\
            \midrule
            \multirow{5}*{\begin{tabular}[c]{@{}c@{}} HallAndMonitor\\($240\times352\times297$)\end{tabular}}  
          & RPCA-SGD\cite{tong2021accelerating}  &  4.5542 & 3.5843 &  2.0127 & 0.9521 &  28.4215&  \textbf{18.37}\\
            &Tucker-SGD\cite{dong2023fast}  &  3.8604 & 2.0448 &  1.0983 & 0.9522 &  28.1574& 50.61\\
            &RTPCA-TNN\cite{lu2019tensor}  &  4.9731 & 3.3638 &  2.2364 & 0.9132 &  25.6906&  287.15\\
            &{RTPCA-SGD  } &  \textbf{3.2279} & \textbf{1.2190} &  \textbf{0.6554} &  \textbf{0.9628}  &   \textbf{29.2700} &78.41\\
			\bottomrule
		\end{tabular}
	\vspace{-0.2cm}
\end{table*}

\begin{figure*}[htbp]
% \scriptsize
\footnotesize
% \small
\setlength{\tabcolsep}{1pt}
\begin{center}
\begin{tabular}{cccccc} 
\includegraphics[width=0.163\textwidth]{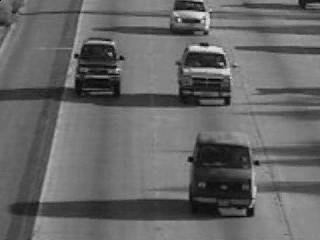}&
\includegraphics[width=0.163\textwidth]{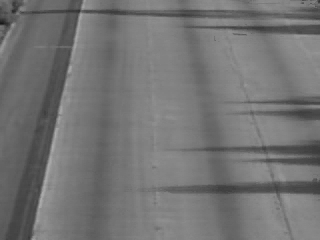}&
\includegraphics[width=0.163\textwidth]{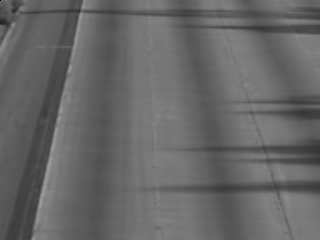}&
\includegraphics[width=0.163\textwidth]{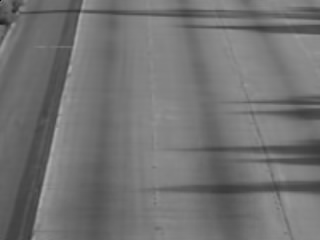}&
\includegraphics[width=0.163\textwidth]{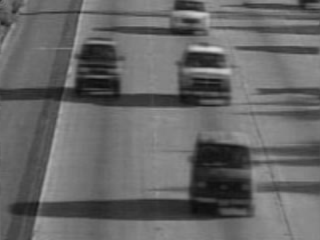}&
\includegraphics[width=0.163\textwidth]{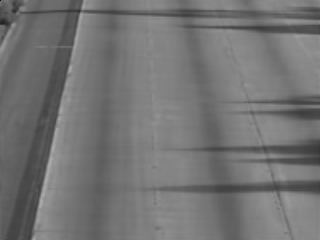}\\
\includegraphics[width=0.163\textwidth]{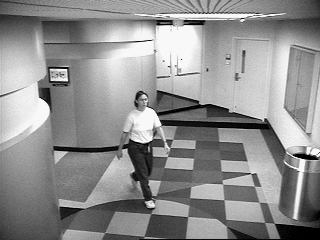}&
\includegraphics[width=0.163\textwidth]{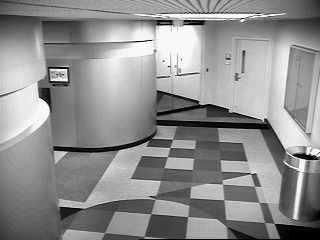}&
\includegraphics[width=0.163\textwidth]{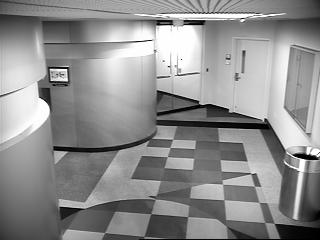}&
\includegraphics[width=0.163\textwidth]{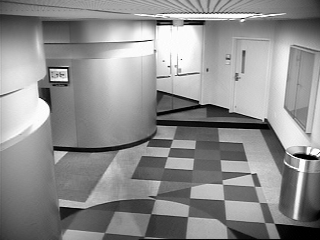}&
\includegraphics[width=0.163\textwidth]{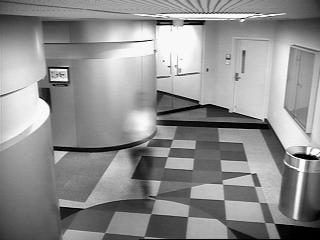}&
\includegraphics[width=0.163\textwidth]{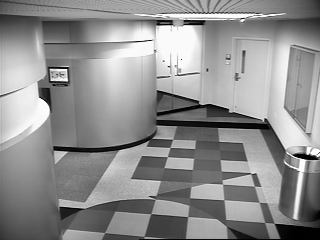}\\
\includegraphics[width=0.163\textwidth]{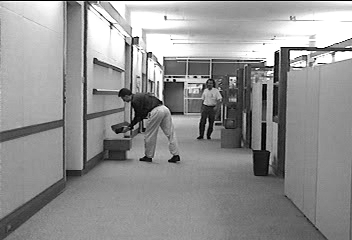}&
\includegraphics[width=0.163\textwidth]{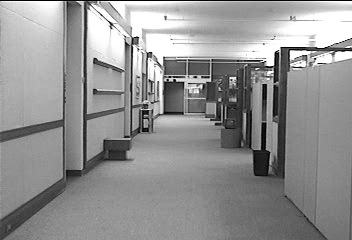}&
\includegraphics[width=0.163\textwidth]{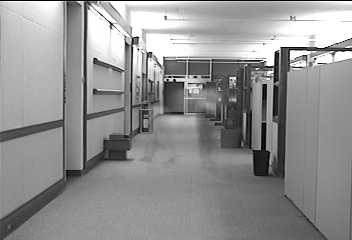}&
\includegraphics[width=0.163\textwidth]{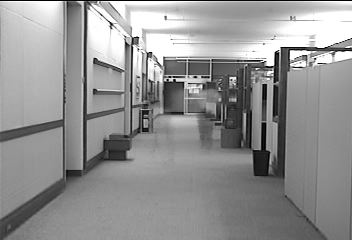}&
\includegraphics[width=0.163\textwidth]{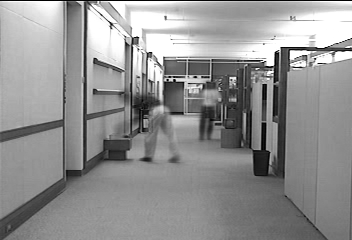}&
\includegraphics[width=0.163\textwidth]{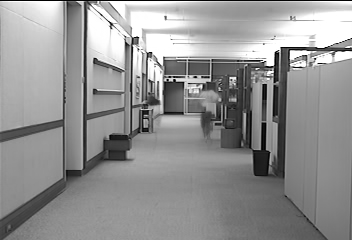}\\
Observed & GT & RPCA-SGD\cite{tong2021accelerating}  & Tucker-SGD\cite{dong2023fast} & RTPCA-TNN\cite{lu2019tensor}  &   RTPCA-SGD
  \end{tabular}
  \caption{Background initialization results by different methods on 3 scenes including the $182$-th frame of ``HighwayI'', the $62$-th frame of ``IBMtest2'', and the $132$-th frame of ``HallAndMonitor''.}
		\label{fig: BM}
  \end{center}\vspace{-0.6cm}
\end{figure*}

\subsection{Effectiveness of the deep unfolding}
We also evaluate the effectiveness of the deep unfolding method on real-world datasets by comparing the performance of RTPCA-SGD and RTPCA-LSGD, as shown in Fig. \ref{fig: discussion}.
In the video denoising experiment, RTPCA-LSGD performs comparably to RTPCA-SGD under a low noise rate (10\%) but significantly outperforms it under a high noise rate (30\%). This demonstrates the strength of the deep unfolding method in parameter optimization, enabling superior recovery performance in challenging scenarios.
In addition, RTPCA-LSGD achieves slight yet consistent improvements over RTPCA-SGD in the background initialization task. These results further validate the practical effectiveness of the deep unfolding method across different tasks.

\begin{figure*}[htbp]
% \scriptsize
% \small
\footnotesize
\setlength{\tabcolsep}{1pt}
\begin{center}
\begin{tabular}{cccccc} 
10 \% noise & 41.8706 dB & 41.8831 dB & HighwayI& 37.1574 dB& 37.2183 dB\\
\includegraphics[width=0.152\textwidth]{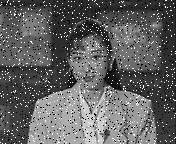}&
\includegraphics[width=0.152\textwidth]{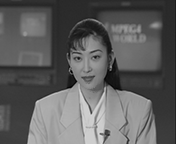}&
\includegraphics[width=0.152\textwidth]{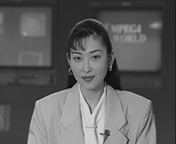}&\quad
\includegraphics[width=0.166\textwidth]{images/background/imshow/HighwayI_frame182/Observed.png}&
\includegraphics[width=0.166\textwidth]{images/background/imshow/HighwayI_frame182/RTPCA-SGD.png}
&\includegraphics[width=0.166\textwidth]{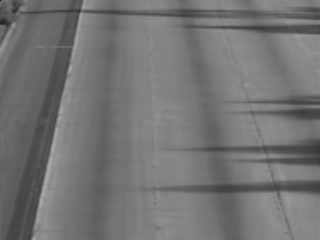}\\
 20 \% noise & 40.2838 dB & 41.1209 dB & IBMtest2 & 35.6996 dB& 35.8065 dB\\
\includegraphics[width=0.152\textwidth]{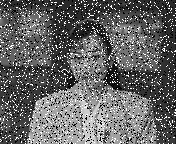}&
\includegraphics[width=0.152\textwidth]{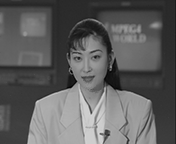}&
\includegraphics[width=0.152\textwidth]{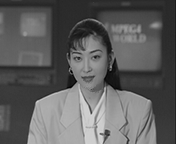}&\quad
\includegraphics[width=0.166\textwidth]{images/background/imshow/IBMtest2_frame62/Observed.png}&
\includegraphics[width=0.166\textwidth]{images/background/imshow/IBMtest2_frame62/RTPCA-SGD.png}
&\includegraphics[width=0.166\textwidth]{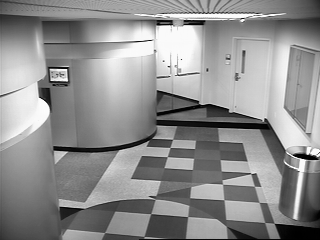}\\
 30 \% noise & 34.8449 dB & 39.4644 dB &  HallAndMonitor & 29.2700 dB& 29.3871 dB \\
\includegraphics[width=0.152\textwidth]{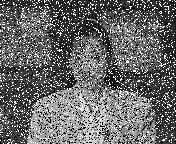}&
\includegraphics[width=0.152\textwidth]{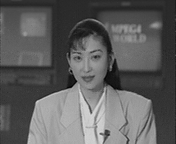}&
\includegraphics[width=0.152\textwidth]{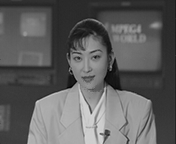}&\quad
\includegraphics[width=0.166\textwidth, height=2.22cm]{images/background/imshow/HallAndMonitor_frame105/Observed.png}&     
\includegraphics[width=0.166\textwidth, height=2.22cm]{images/background/imshow/HallAndMonitor_frame105/RTPCA-SGD.png}&
\includegraphics[width=0.166\textwidth, height=2.22cm]{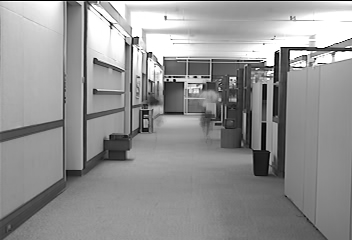}\\
Observed &  RTPCA-SGD&  RTPCA-LSGD & Observed &  RTPCA-SGD&  RTPCA-LSGD \vspace{4pt}\\
 \multicolumn{3}{c}{(a) Video denoising} & \multicolumn{3}{c}{(b) Background initialization} 
  \end{tabular}
  \caption{Comparison reslults of RTPCA-SGD and RTPCA-LSGD on real-world datasets. The avarage PSNR values are shown on the top of the subfigure. }
		\label{fig: discussion}
  \end{center}\vspace{-0.6cm}
\end{figure*}

\section{Conclusions} \label{sec: conclusions}
This paper proposes the RTPCA-SGD model, introducing an efficient ScaledGD method within the t-SVD framework for the first time, along with rigorous theoretical guarantees for exact recovery. 
Additionally, we develop a learnable self-supervised deep unfolding model termed RTPCA-LSGD, enabling effective parameter learning to achieve further performance improvement. 
Synthetic and real-world experiments validate our theoretical findings and demonstrate the effectiveness and efficiency of the proposed methods.
Future work could extend the proposed method to any invertible transformations\cite{qin2022low}, high-order versions \cite{feng2023multiplex},  and other tensor analysis tasks, such as tensor completion and tensor regression~\cite{liu2022tensor}. In addition, exploring recovery guarantees for cases with unknown tubal rank remains a promising direction~\cite{giampouras2024guarantees}.

\ifCLASSOPTIONcaptionsoff
\newpage
\fi

\bibliographystyle{IEEEtran}
\bibliography{reference}
\label{sec:refs}

\begin{IEEEbiography}[{\includegraphics[width=1in,height=1.25in,clip,keepaspectratio]{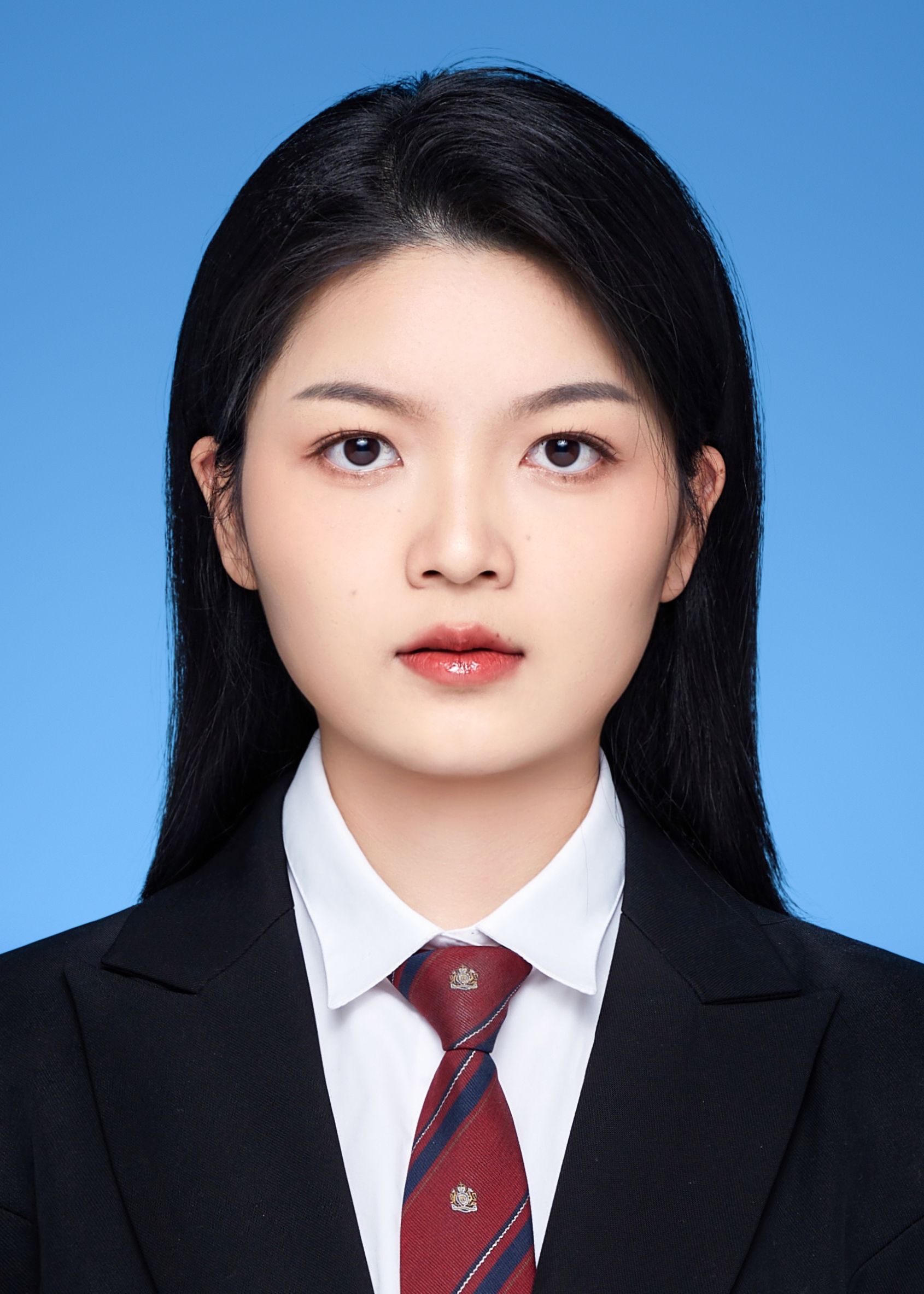}}]{LanlanFeng}
received the B.S. degree from the
University of Electronic Science and Technology of
China (UESTC), Chengdu, China, in 2017, where
she is currently pursuing the Ph.D. degree. Her
research interests include tensor principal component analysis, tensor completion, and tensor neural networks.	
\end{IEEEbiography}

\begin{IEEEbiography}[{\includegraphics[width=1in,height=1.25in,clip,keepaspectratio]{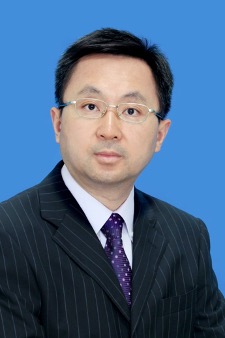}}]
{Ce Zhu} (Fellow, IEEE) received the B.S. degree in electronic and information engineering from Sichuan University,  Chengdu, China, in 1989, and the M.Eng. and Ph.D. degrees from Southeast University, Nanjing, China, in 1992 and 1994, respectively, all in electronic and information engineering. He held a postdoctoral research position with the Chinese University of Hong Kong in 1995, the City University of Hong Kong, and the University of Melbourne, Australia, from 1996 to 1998. For 14 years from 1998 to 2012, he was with Nanyang Technological University, Singapore, where he was a research fellow, a program manager, an assistant professor, and then promoted to an associate professor in 2005. He has been with the University of Electronic Science and Technology of China (UESTC), Chengdu, China, as a professor since 2012, and is the Dean of Glasgow College, a joint school between the University of Glasgow, Glasgow, UK and the University of Electronic Science and Technology of China, Chengdu, China. His research interests include video coding and communications, video analysis and processing, 3D video, and visual perception and applications. He has served on the editorial boards of a few journals, including as an Associate Editor for IEEE Transactions on Image Processing, IEEE Transactions on Circuits and Systems for Video Technology, IEEE Transactions on Broadcasting, IEEE Signal Processing Letters, the Editor of IEEE Communications Surveys and Tutorials, and the Area Editor of Signal Processing: Image Communication. He was also the Guest Editor of a few special issues in international journals, including as the Guest Editor of the IEEE Journal of Selected Topics in Signal Processing.
He was an APSIPA Distinguished Lecturer during 2021-2022, and also an IEEE Distinguished Lecturer of Circuits and Systems Society during 2019-2020. He is the Chair of IEEE ICME Steering Committee during 2024-2025. He was the co-recipient of multiple paper awards at international conferences, including the most recent Best Demo Award in IEEE MMSP 2022, and the Best Paper Runner Up Award in IEEE ICME 2020.
\end{IEEEbiography}

\begin{IEEEbiography}[{\includegraphics[width=1in,height=1.25in,clip,keepaspectratio]{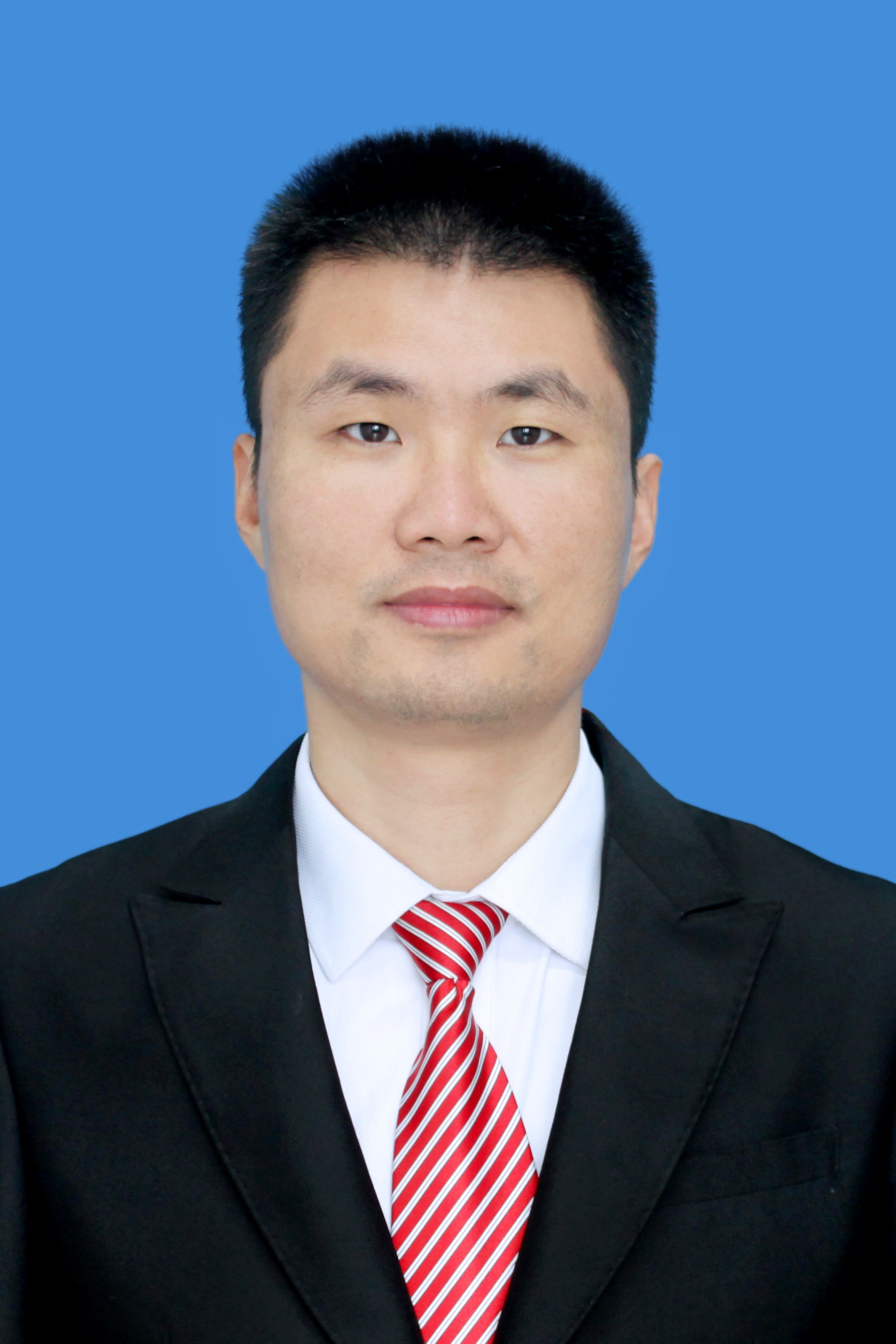}}]{Yipeng Liu} (Senior Member, IEEE) received the B.S. degree in biomedical engineering and the Ph.D. degree in information and communication engineering from University of Electronic Science and Technology of China (UESTC), Chengdu, China, in 2006 and 2011, respectively. In 2011, he was a Research Engineer with Huawei Technologies. From 2011 to 2014, he was a Research Fellow with the University of Leuven, Leuven, Belgium. In 2014, he joined as an Associate Professor at the UESTC, Chengdu, China, where he has been a Full Professor, since 2023. 
His research interest is tensor for data processing. He has published over 100 papers, co-authored two books “Tensor Computation for Data Analysis” by Springer and “Tensor Regression” by Foundations and Trends\textregistered \  in Machine Learning of NOW Publishers, and edited one book “Tensors for Data Processing” by Elsevier. He has served as an associate editor for IEEE Signal Processing Letters and the lead guest editor for Signal Processing: Image Communication. He has given tutorials for 10 international conferences, such as IJCAI 2022, ICME 2022, MLSP 2022, APSIPSA ASC 2022, VCIP 2021, ICIP 2020, SSCI 2020, ISCAS 2019, SiPS 2019, and APSIPA ASC 2019. He is the APSIPA Distinguished Lecturer 2022-2023. 
\end{IEEEbiography}

\begin{IEEEbiography}[{\includegraphics[width=1in,height=1.25in,clip,keepaspectratio]{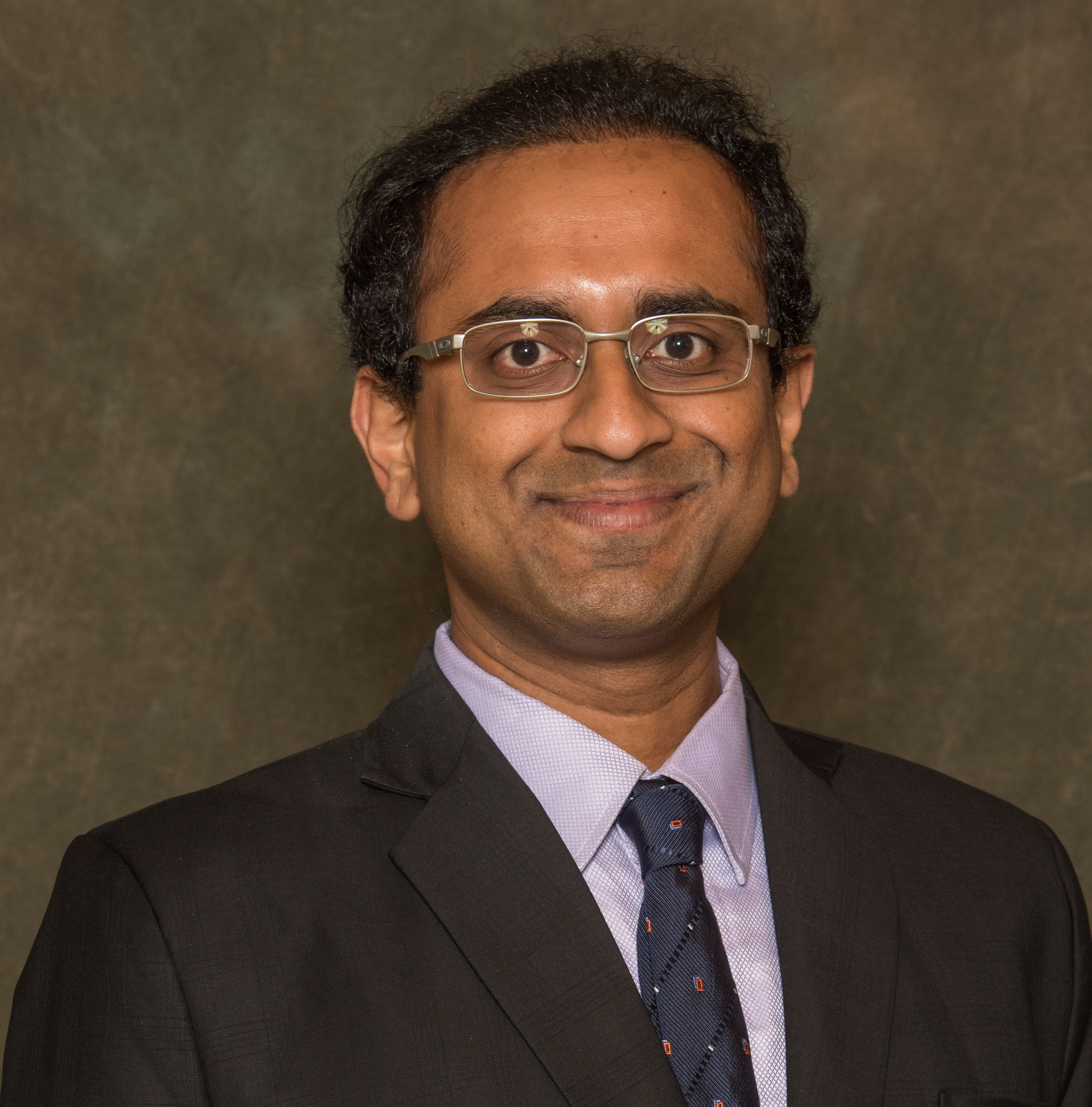}}]{Saiprasad Ravishankar} (Senior Member, IEEE) received the B.Tech. degree in Electrical Engineering from the Indian Institute of Technology Madras, Chennai, India, in 2008, and the M.S. and Ph.D. degrees in Electrical and Computer Engineering from the University of Illinois at Urbana-Champaign, Urbana, IL, USA, in 2010 and 2014, respectively. He is currently an Assistant Professor with the Departments of Computational Mathematics, Science and Engineering, and Biomedical Engineering, Michigan State University, Michigan, USA. He was an Adjunct Lecturer and a Postdoctoral Research Associate with the University of Illinois at Urbana-Champaign from February to August, 2015. Since August 2015, he was a Postdoc with the Department of Electrical Engineering and Computer Science at the University of Michigan, Ann Arbor, MI, USA, and then a Postdoc Research Associate with the Theoretical Division at Los Alamos National Laboratory, Los Alamos, NM, USA, from August 2018 to February 2019. His research interests include biomedical and computational imaging, signal and image processing, machine learning, inverse problems, large-scale data processing, optimization, and neuroscience. He was the recipient of the IEEE Signal Processing Society Young Author Best Paper Award in 2016. A paper he co-authored won a Best Student Paper Award at the IEEE International Symposium on Biomedical Imaging (ISBI) 2018 and other papers were award finalists at the IEEE International Workshop on Machine Learning for Signal Processing (MLSP) 2017, ISBI 2020, and Optica Imaging Congress, 2023. He is currently a member of the IEEE Machine Learning for Signal Processing Technical Committee (MLSP TC) and the IEEE Bio Image and Signal Processing (BISP) TC. He has organized several special sessions or workshops on computational imaging and machine learning themes including at the Institute for Mathematics and its Applications (IMA), IEEE Image, Video, and Multidimensional Signal Processing (IVMSP) Workshop 2016, MLSP 2017, ISBI 2018, the International Conference on Computer Vision (ICCV) 2019 and 2021, etc.
\end{IEEEbiography}

\begin{IEEEbiography}[{\includegraphics[width=1in,height=1.25in,clip,keepaspectratio]{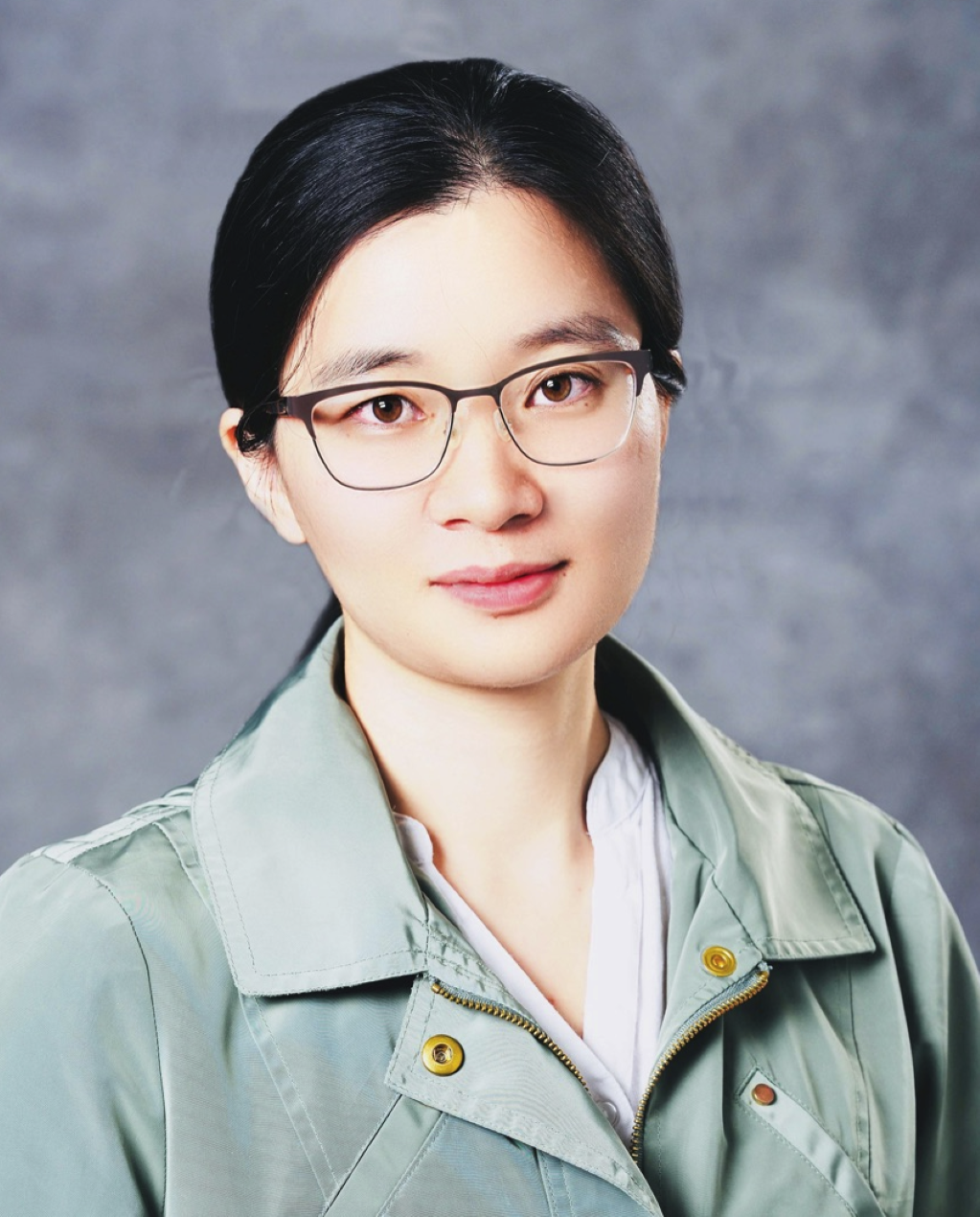}}]{Longxiu Huang} received her Ph.D.~degree in Applied Mathematics from Vanderbilt University, Tennessee, USA, in 2019. She is currently an Assistant Professor in both the Department of Computational Mathematics, Science, and Engineering and the Department of Mathematics at Michigan State University, Michigan, USA. Prior to this, she served as an Assistant Adjunct Professor in the Department of Mathematics at the University of California, Los Angeles (UCLA), California, USA, from July 2019 to June 2022.  
Her research interests include applied harmonic analysis, machine learning, data science, tensor analysis, numerical linear algebra and approximation theory. 
\end{IEEEbiography}

\clearpage
%{\twocolumn[
%	\begin{center}
%		\Huge Supplementary Material for ``Learnable Scaled Gradient Descent for Guaranteed Robust Tensor PCA''
%		\vspace{0.4in} 
%	\end{center}]}
\onecolumn
\begin{center}
	\Huge Supplementary Material for ``Learnable Scaled Gradient Descent for Guaranteed Robust Tensor PCA'' \\
	\vspace{8pt}
    \large Lanlan~Feng,
    Ce~Zhu,~\IEEEmembership{Fellow,~IEEE,} % <-this % stops a space
	Yipeng~Liu,~\IEEEmembership{Senior Member,~IEEE,}
	Saiprasad Ravishankar,\\~\IEEEmembership{Senior Member,~IEEE,}
	Longxiu~Huang
	\vspace{0.3in}
\end{center}
\setcounter{section}{0}

We first introduce some key properties about tensor operations in Section\ref{subsection: Some properties about tensor operations}. Following this, the proof details for Lemma \ref{lm: Q_existence}, \ref{lemma:matrix2factor}, \ref{lm:sparity} are provided in Section\ref{appendix B}.
In Section\ref{appendix C}, we then establish several important auxiliary lemmas, which are necessary for proving Theorem \ref{thm:local convergence} and Theorem \ref{thm:initial}.
Finally, we present the detailed proofs of Theorem \ref{thm:local convergence} (local linear convergence) in Section\ref{sec:local conv} and Theorem \ref{thm:initial} (guaranteed initialization) in Section\ref{sec:guaranteed initialization}.

\subsection{Some properties about tensor operations}
\label{subsection: Some properties about tensor operations}
We summarize some properties of tensor operations.
% used in the proof process.
\begin{itemize}
	\item $\norm{\AC}_{\fro} = \frac{1}{\sqrt{I_3}}\norm{\widehat{\AC}}_{\fro} = \frac{1}{\sqrt{I_3}}\norm{\widehat{\Am}}_{\fro}$
	\item $\norm{\AC}_{\op} = \norm{\widehat{\AC}}_{\op} = \norm{\widehat{\Am}}_{\op}$
	\item $\|\mathcal{A} * \mathcal{B}\|_{\fro}\ge\|\mathcal{A}\|_{\fro}\sigma_{\min}(\mathcal{B})$
	\item $\|\mathcal{A}* \mathcal{B}\|_{2}\ge\|\mathcal{A}\|_{2}\sigma_{\min}(\mathcal{B})$
	\item $\|\mathcal{A}* \mathcal{B}\|_{2,\infty}\ge\|\mathcal{A}\|_{2,\infty}\sigma_{\min}(\mathcal{B})$
	\item $ \left\Vert \AC \right\Vert _{\op}\le {\sqrt{I_3}} \left\Vert \AC \right\Vert _{\fro}$
	\item $|\sigma_{\min}(\mathcal{A})-\sigma_{\min}(\mathcal{B})| \le \|\mathcal{A}-\mathcal{B}\|_{\op}$
	\item $|\sigma_{\max}(\mathcal{A})-\sigma_{\max}(\mathcal{B})| \le \|\mathcal{A}-\mathcal{B}\|_{\op}$
	\item $\sigma_{\min}(\mathcal{A})\sigma_{1}(\mathcal{B}) \le \sigma_{1}(\mathcal{A}*\mathcal{B})$
	\item If $\mathcal{Q}$ is an orthogonal tensor, then$\norm{\AC }_{\fro} = \norm{\AC * \QC}_{\fro} = \norm{\QC * \AC}_{\fro}$; $\norm{\AC }_{\op} = \norm{\AC * \QC}_{\op} = \norm{\QC * \AC}_{\op}$
	\item $\norm{\AC + \BC }_{\op} \leq \norm{\AC }_{\op} + \norm{\BC}_{\op}$
	\item   $\sigma_{1}(\mathcal{A}) = \frac{1}{\sigma_{\min}(\mathcal{A}^{-\To})}$
	\item $\operatorname{tr}(\AC * \BC) = \operatorname{tr}(\BC * \AC)$
	\item $\operatorname{tr}(\AC * \BC) \le \norm{\AC}_{\op} \operatorname{tr}(\BC)$
	\item $\|\mathcal{A}* \mathcal{B}\|_{\op}\le \|\mathcal{A}\|_{\op}\|\mathcal{B}\|_{\op}$
	\item $ \tr(\mathcal{A}*\mathcal{B})  \le \|\mathcal{A}\|_{2}\|\mathcal{B}\|_{*}$
	\item If the tubal rank of $\mathcal{A}$ is $R$, we have $\|\mathcal{A}\|_{*} \le \sqrt{R}\|\mathcal{A}\|_{\fro} $ because of the Cauchy-Schwaz inequality $ \left( \sum_{i=1}^{n} a_i \right)^2 \leq n \sum_{i=1}^{n} a_i^2$. Moreover, $\|\mathcal{A}\|_{\fro}^{2} = \frac{1}{I_3}\|\widehat{\mathbf{A}}\|_{\fro}^{2} \le \frac{1}{I_3} R I_3\|\widehat{\mathbf{A}}\|_{2}^{2} \le R\|\mathcal{A}\|_{2}^{2}$.
	\item $\|\mathcal{A}* \mathcal{B}\|_{\fro}\le \|\mathcal{A}\|_{2}\|\mathcal{B}\|_{\fro}$
	\item $\|\mathcal{A}* \mathcal{B}\|_{2,\infty}\le \|\mathcal{A}\|_{2}\|\mathcal{B}\|_{2,\infty}$
	\item $\|\mathcal{A}* \mathcal{B}^{\To}\|_{\infty}\le \|\mathcal{A}\|_{2,\infty}\|\mathcal{B}\|_{2,\infty}$
	\item $\|\mathcal{A}*\mathcal{B}\|_{2,\infty}\leq\|\mathcal{A}\|_{1,\infty}\|\mathcal{B}\|_{2,\infty}$ 
\end{itemize} 

\subsection{Proof of Lemma \ref{lm: Q_existence}, \ref{lemma:matrix2factor},  \ref{lm:sparity} } \label{appendix B}
\begin{proof}[\textbf{Proof of Lemma \ref{lm: Q_existence}}]
	Given the condition in Inequation \eqref{eq:Q_existence_condition} and Definition \ref{de: em} of $\distk$  in Equation \eqref{eq: em}, one knows that there exists a tensor $\QC_{\diamond} \in \GL(R)$ such that
	\begin{equation}\label{eq: Q_existenceas}
		\sqrt{\left\Vert \left(\LC*\QC_{\diamond}-\LC_{\star}\right)*\SCigma_{\star}^{1/2}\right\Vert _{\fro}^{2}+\left\Vert \left(\RC*\QC_{\diamond}^{-\top}-\RC_{\star}\right)*\SCigma_{\star}^{1/2}\right\Vert _{\fro}^{2}} \le \frac{\epsilon }{\sqrt{I_3}}\sigma_{\min}\left(\mathcal{X}_{\star}\right)
		% \frac{\epsilon }{\|\mathcal{X}^{\dagger}\|_{2}}\,
	\end{equation}
	for some $\epsilon$ satisfing  $0<\epsilon<1$. 
	In light of the relation $\|\mathcal{A}*\mathcal{B}\|_{\fro}\ge\|\mathcal{A}\|_{\fro}\sigma_{\min}(\mathcal{B})$, 
	\begin{equation}
		\begin{split}
			& \sqrt{\left\Vert \left(\LC*\QC_{\diamond}-\LC_{\star}\right)*\mathit{\Sigma}_{\star}^{-1/2}\right\Vert _{\fro}^{2}+\left\Vert \left(\RC*\QC_{\diamond}^{-\top}-\RC_{\star}\right)*\mathit{\Sigma}_{\star}^{-1/2}\right\Vert _{\fro}^{2}}\\
			\le &\frac{1}{\sigma_{\min}( \mathcal{X}_{\star})}  
			\sqrt{ \left\Vert \left(\LC*\QC_{\diamond}-\LC_{\star}\right)*\mathit{\Sigma}_{\star}^{1/2}\right\Vert _{\fro}^{2}+\left\Vert \left(\RC*\QC_{\diamond}^{-\top}-\RC_{\star}\right)*\mathit{\Sigma}_{\star}^{1/2}\right\Vert _{\fro}^{2} }\\
			\le &\frac{\epsilon}{\sqrt{I_3}}.
		\end{split}
	\end{equation}
	Utilize the relationship $I_3\left\Vert \mathcal{A} \right\Vert _{\fro}^{2} \ge \left\Vert \mathcal{A} \right\Vert _{\op}^{2}$, 
	\begin{equation}
		\begin{split}
			&\sqrt{ \left\Vert \left(\LC*\QC_{\diamond}-\LC_{\star}\right)*\mathit{\Sigma}_{\star}^{-1/2}\right\Vert _{\fro}^{2}+\left\Vert \left(\RC*\QC_{\diamond}^{-\top}-\RC_{\star}\right)*\mathit{\Sigma}_{\star}^{-1/2}\right\Vert _{\fro}^{2}}\\
			\ge& \frac{1}{\sqrt{I_3}}\sqrt{ \left\Vert \left(\LC*\QC_{\diamond}-\LC_{\star}\right)*\mathit{\Sigma}_{\star}^{-1/2}\right\Vert _{\op}^{2}+\left\Vert \left(\RC*\QC_{\diamond}^{-\top}-\RC_{\star}\right)*\mathit{\Sigma}_{\star}^{-1/2}\right\Vert _{\op}^{2}}\\
			\ge &\frac{1}{\sqrt{I_3}}\left\Vert \left(\LC*\QC_{\diamond}-\LC_{\star}\right)*\mathit{\Sigma}_{\star}^{-1/2}\right\Vert _{\op}  \vee
			\left\Vert \left(\RC*\QC_{\diamond}^{-\top}-\RC_{\star}\right)*\mathit{\Sigma}_{\star}^{-1/2}\right\Vert _{\op},
		\end{split}
	\end{equation}
	%  then according to the property of the spectral norm $\norm{\AC}_2 = \norm{\widehat{\mathbf{A}}}_2$,  it means that 
	% \begin{align}
		% \left\Vert \left(\LC*\QC_{\diamond}-\LC_{\star}\right)*\SCigma_{\star}^{-1/2}\right\Vert _{\op} \vee \left\Vert \left(\RC*\QC_{\diamond}^{-\top}-\RC_{\star}\right)*\SCigma_{\star}^{-1/2}\right\Vert _{\op} \le \epsilon.
		% \end{align}
	it means that 
	\begin{align} \left\Vert \left(\LC*\QC_{\diamond}-\LC_{\star}\right)*\mathit{\Sigma}_{\star}^{-1/2}\right\Vert _{\text{op}} \vee \left\Vert \left(\RC*\QC_{\diamond}^{-\top}-\RC_{\star}\right)*\mathit{\Sigma}_{\star}^{-1/2}\right\Vert _{\text{op}} \le \epsilon. \label{eq: op_ineq}
	\end{align}
	Then Invoke Weyl's inequality $|\sigma_{\min}(\mathcal{A})-\sigma_{\min}(\mathcal{B})| \le \|\mathcal{A}-\mathcal{B}\|_{\op}$, and $\sigma_{\min}(\LC_{\star}*\mathit{\Sigma}_{\star}^{-1/2})=\sigma_{\min}(\widehat{\mathcal{U}}_{\star})=1$ are used to obtain
	\begin{align}
		\sigma_{\min}(\LC*\QC_{\diamond}*\mathit{\Sigma}_{\star}^{-1/2}) \ge 1 - \left\Vert \left(\LC*\QC_{\diamond}-\LC_{\star}\right)*\mathit{\Sigma}_{\star}^{-1/2}\right\Vert _{\op} \ge 1-\epsilon.
		\label{eq:first_sigma}
	\end{align}
	% \begin{align}
		% \small \sigma_{\min}(\LC*\QC_{\diamond}*\mathit{\Sigma}_{\star}^{-1/2}) \ge \sigma_{\min}(\LC_{\star}*\mathit{\Sigma}_{\star}^{-1/2}) - \left\Vert \left(\LC*\QC_{\diamond}-\LC_{\star}\right)*\mathit{\Sigma}_{\star}^{-1/2}\right\Vert _{\op} \ge 1-\epsilon.\label{eq:first_sigma}
		% \end{align}
	In addition, it is straightforward to verify that 
	\begin{align}
		& \inf_{\QC\in \GL(R)}\left\Vert \left(\LC*\QC-\LC_{\star}\right)*\mathit{\Sigma}_{\star}^{1/2}\right\Vert _{\fro}^{2}+\left\Vert \left(\RC*\QC^{-\top}-\RC_{\star}\right)*\mathit{\Sigma}_{\star}^{1/2}\right\Vert _{\fro}^{2}\label{eq:first_inf}\\
		=&\inf_{\HC\in \GL(R)}\left\Vert \left(\LC*\QC_{\diamond}*\HC-\LC_{\star}\right)*\mathit{\Sigma}_{\star}^{1/2}\right\Vert _{\fro}^{2}+\left\Vert \left(\RC*\QC_{\diamond}^{-\top}*\HC^{-\top}-\RC_{\star}\right)*\mathit{\Sigma}_{\star}^{1/2}\right\Vert _{\fro}^{2}.\label{eq:second_inf}
	\end{align}
	If the minimizer of the optimization problem~\eqref{eq:second_inf} exists at some $\HC$, then $\QC_{\diamond}*\HC$ must be the minimizer of the optimization problem~\eqref{eq:first_inf}.
	Consequently, our focus now shifts to proving the existence of a minimizer for the problem~\eqref{eq:second_inf} at some $\HC$. One has
	\begin{align*}
		& \inf_{\HC\in \GL(R)}\;\left\Vert \left(\LC*\QC_{\diamond}*\HC-\LC_{\star}\right)*\mathit{\Sigma}_{\star}^{1/2}\right\Vert _{\fro}^{2}+\left\Vert \left(\RC*\QC_{\diamond}^{-\top}*\HC^{-\top}-\RC_{\star}\right)*\mathit{\Sigma}_{\star}^{1/2}\right\Vert _{\fro}^{2}\nonumber \\
		& \quad\le\left\Vert \left(\LC*\QC_{\diamond}-\LC_{\star}\right)*\mathit{\Sigma}_{\star}^{1/2}\right\Vert _{\fro}^{2}+\left\Vert \left(\RC*\QC_{\diamond}^{-\top}-\RC_{\star}\right)*\mathit{\Sigma}_{\star}^{1/2}\right\Vert _{\fro}^{2},
	\end{align*}
	Clearly, for any $\QC_{\diamond}*\HC$ to yield a smaller distance than $\QC_{\diamond}$ and recall \eqref{eq: Q_existenceas},$\HC$ must obey
	\begin{align*}
		\sqrt{\left\Vert \left(\LC*\QC_{\diamond}*\HC-\LC_{\star}\right)*\mathit{\Sigma}_{\star}^{1/2}\right\Vert _{\fro}^{2}+\left\Vert \left(\RC*\QC_{\diamond}^{-\top}*\HC^{-\top}-\RC_{\star}\right)*\mathit{\Sigma}_{\star}^{1/2}\right\Vert _{\fro}^{2}} 
		\le \frac{\epsilon }{\sqrt{I_3}}\sigma_{\min}(\XC_{\star}).
	\end{align*}
	which implies that
	\begin{align*}
		\left\Vert \left(\LC*\QC_{\diamond}*\HC-\LC_{\star}\right)*\mathit{\Sigma}_{\star}^{-1/2}\right\Vert _{\op} \vee \left\Vert \left(\RC*\QC_{\diamond}^{-\top}*\HC^{-\top}-\RC_{\star}\right)*\mathit{\Sigma}_{\star}^{-1/2}\right\Vert _{\op} \le \epsilon.
	\end{align*}
	Then Invoke Weyl's inequality $|\sigma_{1}(\mathcal{A})-\sigma_{1}(\mathcal{B})| \le \|\mathcal{A}-\mathcal{B}\|_{\op}$ and $\sigma_{1}(\LC_{\star}*\mathit{\Sigma}_{\star}^{-1/2})=\sigma_{1}(\UC_{\star})=1$ are used to obtain
	\begin{align}
		\sigma_{1}(\LC*\QC_{\diamond}*\HC*\mathit{\Sigma}_{\star}^{-1/2}) \le 1 + \left\Vert\left(\LC*\QC_{\diamond}*\HC-\LC_{\star}\right)*\mathit{\Sigma}_{\star}^{-1/2}\right\Vert _{\op} \le 1+\epsilon.\label{eq:second_sigma}
	\end{align}
	% \begin{align}
		% \sigma_{1}(\LC*\QC_{\diamond}*\HC*\mathit{\Sigma}_{\star}^{-1/2}) \le \sigma_{1}(\LC_{\star}*\mathit{\Sigma}_{\star}^{-1/2}) + \left\Vert\left(\LC*\QC_{\diamond}*\HC-\LC_{\star}\right)*\mathit{\Sigma}_{\star}^{-1/2}\right\Vert _{\op} \le 1+\epsilon.\label{eq:second_sigma}
		% \end{align}
	% \normalsize
	% \setstretch{1.3}
	Combine \eqref{eq:first_sigma} and \eqref{eq:second_sigma}, and use the relation $\sigma_{\min}(\mathcal{A})\sigma_{1}(\mathcal{B}) \le \sigma_{1}(\mathcal{A}*\mathcal{B})$, we can get
	\begin{align*}
		\sigma_{\min}(\LC*\QC_{\diamond}*\mathit{\Sigma}_{\star}^{-1/2})\sigma_{1}(\mathit{\Sigma}_{\star}^{1/2}*\HC*\mathit{\Sigma}_{\star}^{-1/2}) 
		\le \sigma_{1}(\LC*\QC_{\diamond}*\HC*\mathit{\Sigma}_{\star}^{-1/2}) \le \frac{1+\epsilon}{1-\epsilon}\sigma_{\min}(\LC*\QC_{\diamond}*\mathit{\Sigma}_{\star}^{-1/2}).
	\end{align*}
	Therefore, we have $\sigma_{1}(\mathit{\Sigma}_{\star}^{1/2}*\HC*\mathit{\Sigma}_{\star}^{-1/2}) \le \frac{1+\epsilon}{1-\epsilon}$.
	
	Similarly, we also derive $\sigma_{1}(\mathit{\Sigma}_{\star}^{1/2}*\HC^{-\top}*\mathit{\Sigma}_{\star}^{-1/2}) \le \frac{1+\epsilon}{1-\epsilon}$, which is equivalent to $\sigma_{\min}(\mathit{\Sigma}_{\star}^{1/2}*\HC*\mathit{\Sigma}_{\star}^{-1/2}) \ge \frac{1-\epsilon}{1+\epsilon}$ because of the property $\sigma_{1}(\mathcal{A}) = \frac{1}{\sigma_{\min}(\mathcal{A}^{-\To})}$. 
	With the above two arguments, the minimization problem \eqref{eq:second_inf} is equivalent to the following problem: 
	\begin{align*}
		\min_{\mathcal{H}\in \GL(R)} & \left\Vert \left(\LC*\QC_{\diamond}*\HC-\LC_{\star}\right)*\mathit{\Sigma}_{\star}^{1/2}\right\Vert _{\fro}^{2}+\left\Vert \left(\RC*\QC_{\diamond}^{-\top}*\HC^{-\top}-\RC_{\star}\right)*\mathit{\Sigma}_{\star}^{1/2}\right\Vert _{\fro}^{2}\\
		\mbox{s.t.}\quad & \frac{1-\epsilon}{1+\epsilon} \le \sigma_{\min}(\mathit{\Sigma}_{\star}^{1/2}*\HC*\mathit{\Sigma}_{\star}^{-1/2}) \le \sigma_{1}(\mathit{\Sigma}_{\star}^{1/2}*\HC*\mathit{\Sigma}_{\star}^{-1/2})\le \frac{1+\epsilon}{1-\epsilon}.
	\end{align*}
	Note that this is a continuous optimization problem defined over a compact set. By applying the Weierstrass Extreme Value Theorem, we can complete the proof.
\end{proof}

\begin{definition} \label{def: shorthand notation}
	For convenience, we define the following notations: $\mathcal{Q}_k$ denotes the optimal alignment tensor between $(\mathcal{L}_k, \mathcal{R}_k)$ and $(\mathcal{L}_{\star}, \mathcal{R}_\star)$, $\mathcal{L}_\natural:=\mathcal{L}_k*\mathcal{Q}_k$, $\mathcal{R}_\natural:=\mathcal{R}_k*\mathcal{Q}_k^{-\top}$, $\Delta_{\mathcal{L}}:=\mathcal{L}_\natural-\mathcal{L}_\star$, $\Delta_{\mathcal{R}}:=\mathcal{R}_\natural-\mathcal{R}_\star$, and $\Delta_{\mathcal{S}}:=\mathcal{S}_{k+1}-\mathcal{S}_\star$.
	Then we can have the following properties: $\mathcal{L}_\natural * \mathcal{R}_\natural^{\top} = \mathcal{L}_k *\mathcal{R}_k^{\top}$.
	% Moreover, we denote $\Delta_{\widehat{\mathbf{L}}}= \widehat{\mathbf{L}}_\natural-\widehat{\mathbf{L}}_\star =\widehat{\mathbf{L}}_{k}\widehat{\mathbf{Q}}_{k}-\widehat{\mathbf{L}}_\star; \Delta_{\widehat{\mathbf{R}}}:=\widehat{\mathbf{R}}_\natural-\widehat{\mathbf{R}}_\star = \widehat{\mathbf{R}}_k\widehat{\mathbf{Q}}_k^{-\top}-\widehat{\mathbf{R}}_\star.$
\end{definition}

\begin{proof}[\textbf{Proof of Lemma \ref{lemma:matrix2factor}}] 
	Given the decomposition $\mathcal{L}_{\natural}*\mathcal{R}_{\natural}^{\top}-\mathcal{X}_{\star}=\Delta_{\mathcal{L}}*\mathcal{R}_{\star}^{\top}+\mathcal{L}_{\star}*\Delta_{\mathcal{R}}^{\top}+\Delta_{\mathcal{L}}*\Delta_{\mathcal{R}}^{\top}$ and leveraging the triangle inequality, we can derive the following:
	\begin{align*}
		\|\mathcal{L}_{\natural}*\mathcal{R}_{\natural}^{\top}-\mathcal{X}_{\star}\|_{\fro} & \le\|\Delta_{\mathcal{L}}*\mathcal{R}_{\star}^{\top}\|_{\fro}+\|\mathcal{L}_{\star}*\Delta_{\mathcal{R}}^{\top}\|_{\fro}+\|\Delta_{\mathcal{L}}*\Delta_{\mathcal{R}}^{\top}\|_{\fro}\\
		& =\|\Delta_{\mathcal{L}}*\mathit{\Sigma}_{\star}^{1/2}\|_{\fro}+\|\Delta_{\mathcal{R}}*\mathit{\Sigma}_{\star}^{1/2}\|_{\fro}+\|\Delta_{\mathcal{L}}*\Delta_{\mathcal{R}}^{\top}\|_{\fro},
	\end{align*}
	where the following facts have been used
	\begin{align*}
		\|\Delta_{\mathcal{L}}*\mathcal{R}_{\star}^{\top}\|_{\fro}=\|\Delta_{\mathcal{L}}*\mathit{\Sigma}_{\star}^{1/2}*\mathcal{V}_{\star}^{\top}\|_{\fro}=\|\Delta_{\mathcal{L}}*\mathit{\Sigma}_{\star}^{1/2}\|_{\fro}, \\
		\quad\|\mathcal{L}_{\star}*\Delta_{\mathcal{R}}^{\top}\|_{\fro}=\|\mathcal{U}_{\star}*\mathit{\Sigma}_{\star}^{1/2}*\Delta_{\mathcal{R}}^{\top}\|_{\fro}=\|\Delta_{\mathcal{R}}*\mathit{\Sigma}_{\star}^{1/2}\|_{\fro}.
	\end{align*}
	We then derive a simple bound for the third term
	\begin{align*}
		&\|\Delta_{\mathcal{L}}*\Delta_{\mathcal{R}}^{\top}\|_{\fro} \\
		=& \frac{1}{2}\|\Delta_{\mathcal{L}}*\mathit{\Sigma}_{\star}^{1/2}*(\Delta_{\mathcal{R}}*\mathit{\Sigma}_{\star}^{-1/2})^{\top}\|_{\fro}+\frac{1}{2}\|\Delta_{\mathcal{L}}*\mathit{\Sigma}_{\star}^{-1/2}*(\Delta_{\mathcal{R}}*\mathit{\Sigma}_{\star}^{1/2})^{\top}\|_{\fro}\\
		\le & \frac{1}{2}\|\Delta_{\mathcal{L}}*\mathit{\Sigma}_{\star}^{1/2}\|_{\fro}\|\Delta_{\mathcal{R}}*\mathit{\Sigma}_{\star}^{-1/2}\|_{\op}+\frac{1}{2}\|\Delta_{\mathcal{L}}*\mathit{\Sigma}_{\star}^{-1/2}\|_{\op}\|\Delta_{\mathcal{R}}*\mathit{\Sigma}_{\star}^{1/2}\|_{\fro}\\
		\le &\frac{1}{2}(\|\Delta_{\mathcal{L}}*\mathit{\Sigma}_{\star}^{-1/2}\|_{\op}\vee\|\Delta_{\mathcal{R}}*\mathit{\Sigma}_{\star}^{-1/2}\|_{\op})\left(\|\Delta_{\mathcal{L}}*\mathit{\Sigma}_{\star}^{1/2}\|_{\fro}+\|\Delta_{\mathcal{R}}*\mathit{\Sigma}_{\star}^{1/2}\|_{\fro}\right)
	\end{align*}  
	% Here we finish the proof of the first claim.
	If 
	$\distk\leq\frac{\epsilon}{\sqrt{I_3}} \sigma_{\min}\left(\mathcal{X}_{\star}\right)$, the following inequality for the tensor operator norm holds, analogous to the derivation in \eqref{eq: op_ineq}
	\begin{align}
		\|\Delta_{\mathcal{L}}*\mathit{\Sigma}_{\star}^{-1/2}\|_{\op}\vee\|\Delta_{\mathcal{R}}*\mathit{\Sigma}_{\star}^{-1/2}\|_{\op}\le\epsilon,
		\label{eq:cond_RPCA_op}
	\end{align}
	Then we can obtain an easy consequence of Lemma~\ref{lemma:matrix2factor} as
	\begin{equation}
		\begin{aligned}\label{eq:dist_matrix}
			% \begin{split}
				& \left\Vert\mathcal{L}_{k}*\mathcal{R}_{k}^{\top}-\mathcal{X}_{\star}\right\Vert _{\fro} \\
				\le& \left(1+\frac{\epsilon}{2}\right)\left(\|\Delta_{\mathcal{L}}*\mathit{\Sigma}_{\star}^{1/2}\|_{\fro}+\|\Delta_{\mathcal{R}}*\mathit{\Sigma}_{\star}^{1/2}\|_{\fro}\right)\\
				\le&\left(1+\frac{\epsilon}{2}\right)\sqrt{2}\distk
				% & \le \frac{3\sqrt{2}}{2} \distk.
				% \end{split}
		\end{aligned}
	\end{equation}
	Here the last line follows from the elementary inequality $a+b\le\sqrt{2(a^{2}+b^{2})}$ and Definition \ref{de: em} of $\distk$ in Eq.~\eqref{eq: em}.
	The proof is now completed.
	% Here $0<\epsilon<1$, the second line follows from the elementary inequality $a+b\le\sqrt{2(a^{2}+b^{2})}$ and Definition \ref{de: em} of $\distk$ in Eq.~\eqref{eq: em}.
	% The proof is now completed.
\end{proof}

\begin{proof}[\textbf{Proof of Lemma \ref{lm:sparity}}]
	Let \(\Omega_\star := \supp(\mathcal{S}_\star)\) and \(\Omega_{k+1} := \supp(\mathcal{S}_{k+1})\). Recall the update about the sparse tensor that \(\mathcal{S}_{k+1} = \operatorname{T}_{\zeta_{k+1}}(\mathcal{Y} - \mathcal{X}_{k}) = \operatorname{T}_{\zeta_{k+1}}(\mathcal{S}_\star + \mathcal{X}_\star - \mathcal{X}_{k})\), where $\operatorname{T}_{\zeta_{k+1}}$ denotes the soft-thresholding operator with threshold parameter $\zeta_{k+1}$. Since \([\mathcal{S}_\star]_{i_1,i_2,i_3} = 0\) outside its support, it follows that \([\mathcal{Y} - \mathcal{X}_{k}]_{i_1,i_2,i_3} = [\mathcal{X}_\star - \mathcal{X}_{k}]_{i_1,i_2,i_3}\) for all entries \((i_1, i_2, i_3) \in \Omega_\star^c\). By applying the thresholding value \(\zeta_{k+1} \geq \|\mathcal{X}_\star - \mathcal{X}_{k}\|_\infty\), we have \([\mathcal{S}_{k+1}]_{i_1,i_2,i_3} = 0\) for all \((i_1, i_2, i_3) \in \Omega_\star^c\). Consequently, the support of \(\mathcal{S}_{k+1}\) must belongs to the support of \(\mathcal{S}_\star\), \textit{i.e.},
	\begin{equation*}
		\supp(\mathcal{S}_{k+1}) = \Omega_{k+1} \subseteq \Omega_\star = \supp(\mathcal{S}_\star).
	\end{equation*}
	This establishes our first claim.
	
	Clearly, $[\mathcal{S}_\star-\mathcal{S}_{k+1}]_{i_1,i_2,i_3}=0$ for all $(i_1,i_2,i_3)\in\Omega_\star^c$. Furthermore, the entries in \(\Omega_\star\) can be divided into two groups:
	\begin{align*}
		\Omega_{k+1} &= \{(i_1,i_2,i_3)~|~ |[\mathcal{Y}-\mathcal{X}_k]_{i_1,i_2,i_3}|>\zeta_{k+1} \textnormal{ and } [\mathcal{S}_\star]_{i_1,i_2,i_3}\neq 0\},\cr
		\Omega_\star\backslash\Omega_{k+1} &= \{(i_1,i_2,i_3)~|~ |[\mathcal{Y}-\mathcal{X}_k]_{i_1,i_2,i_3}|\leq\zeta_{k+1} \textnormal{ and } [\mathcal{S}_\star]_{i_1,i_2,i_3}\neq 0\},
	\end{align*}
	and it holds
	\begin{align*}
		&~|[\mathcal{S}_\star-\mathcal{S}_{k+1}]_{i_1,i_2,i_3}|\cr
		= &~
		\begin{cases}
			|[\mathcal{X}_{k}-\mathcal{X}_\star]_{i_1,i_2,i_3} - \mathrm{sign}([\mathcal{Y}-\mathcal{X}_{k}]_{i_1,i_2,i_3})\zeta_{k+1} |&   \cr
			|[\mathcal{S}_\star]_{i_1,i_2,i_3}|         & \cr
		\end{cases}  \cr
		\leq&~
		\begin{cases}
			|[\mathcal{X}_{k}-\mathcal{X}_\star]_{i_1,i_2,i_3}|+\zeta_{k+1}      &  \cr
			|[\mathcal{X}_\star-\mathcal{X}_{k}]_{i_1,i_2,i_3}|+\zeta_{k+1}  & \cr
		\end{cases} \cr
		\leq&~
		\begin{cases}
			2\zeta_{k+1}     & \quad (i_1,i_2,i_3)\in \Omega_{k+1},   \cr
			2\zeta_{k+1}     & \quad (i_1,i_2,i_3)\in \Omega_\star\backslash\Omega_{k+1}. \cr
		\end{cases}
		% \begin{cases}
			% 2\|\mathcal{X}_\star-\mathcal{X}_{k}\|_\infty      & \qquad (i_1,i_2,i_3)\in \Omega_{k+1},   \cr
			% 2\|\mathcal{X}_\star-\mathcal{X}_{k}\|_\infty      & \qquad (i_1,i_2,i_3)\in \Omega_\star\backslash\Omega_{k+1}. \cr
			% \end{cases}
	\end{align*}
	Thus, we have \(\|\mathcal{S}_\star - \mathcal{S}_{k+1}\|_\infty \leq 2 \|\mathcal{X}_\star - \mathcal{X}_k\|_\infty + \zeta_{k+1} \leq 2 \zeta_{k+1}\), which proves our second claim.
	% Now we will prove the last claim. By the arguments above, we have
	% \begin{align*}
		%     |\langle \mathcal{S}_\star-\mathcal{S}_{k+1},\BM\rangle|&\leq \sum_{(i_1,i_2,i_3)\in\Omega_\star} |[\mathcal{S}_\star-\mathcal{S}_{k+1}]_{i_1,i_2,i_3}|~|[\BM]_{i_1,i_2,i_3}| \cr
		%     &\leq \sum_{(i_1,i_2,i_3)\in\Omega_\star} 2\|\mathcal{X}_\star-\mathcal{X}_{k}\|_\infty~|[\BM]_{i_1,i_2,i_3}|
		% \end{align*}
\end{proof}

% Now, we are already to prove Theorem~\ref{thm:main_theorem}.

% \begin{proof}[\textbf{Proof of Theorem~~\ref{thm:main_theorem}}]
	% Take $c_0=10^{-4}$ in Theorem~\ref{thm:initial}. Thus, the results of Theorem~\ref{thm:initial} satisfy the condition of Theorem~\ref{thm:local convergence}, and gives 
	% \begin{gather*}
		%     \distk\leq \frac{0.02}{\sqrt{I_3}} (1-0.6\eta)^k \sigma_{\min}(\mathcal{X}_\star), \cr
		%     \sqrt{I_1}\|(\mathcal{L}_k*\mathcal{Q}_k-\mathcal{L}_\star)*\mathit{\Sigma}_\star^{1/2}\|_{2,\infty}  \lor \sqrt{I_2}\|(\mathcal{R}_k*\mathcal{Q}_k^{-\top}-\mathcal{R}_\star)*\mathit{\Sigma}_\star^{1/2}\|_{2,\infty} \cr
		%     \leq\sqrt{\frac{\mu R}{I_3}} (1-0.6\eta)^k \sigma_{\min}(\mathcal{X}_\star)
		% \end{gather*}
	% for all $k\geq0$. Moreover, Lemma~\ref{lemma:matrix2factor} states that
	% \begin{gather*}
		%     \| \mathcal{L}_k\mathcal{R}_k^{\top} - \mathcal{X}_\star \|_{\fro} \leq \frac{3\sqrt{2}}{2}~\distk
		% \end{gather*}
	% as long as $\distk\leq\frac{\epsilon}{\sqrt{I_3}} \sigma_{\min}\left(\mathcal{X}_{\star}\right)$. See \eqref{eq:dist_matrix} for a detailed argument.
	% Hence, our first claim is proved.
	
	% When $k\geq 1$, the second claim is directly followed by Lemma~\ref{lm:sparity}. When $k=0$, take $\mathcal{X}_{-1}=\bm{0}$, then one can see $\mathcal{S}_0=\operatorname{T}_{\zeta_0}(\mathcal{Y})=\operatorname{T}_{\zeta_0}(\mathcal{Y}-\mathcal{X}_{-1})$ where $\zeta_0=\|\mathcal{X}_\star\|_\infty=\|\mathcal{X}_\star-\mathcal{X}_{-1}\|_\infty$. Applying Lemma~\ref{lm:sparity} again, we have the second claim for all $k\geq0$.
	
	% This finishes the proof.
	% \end{proof}

\subsection{Auxiliary Lemma}\label{appendix C}
Before we can present the detailed proofs for Theorem~\ref{thm:local convergence} and \ref{thm:initial}, several important auxiliary Lemmas must be established.

\begin{lemma}\label{lemma:Q_criterion} Suppose that the optimal alignment tensor
	\begin{align*}
		\mathcal{Q}_{k} =\argmin_{\mathcal{Q}_{k}\in\GL(R)}\;\left\Vert (\mathcal{L}_{k}*\mathcal{Q}_{k}-\mathcal{L}_{\star})*{\mathit{\Sigma}}_{\star}^{1/2}\right\Vert _{\fro}^{2}+\left\Vert (\mathcal{R}_{k}*\mathcal{Q}_{k}^{-\top}-\mathcal{R}_{\star})*{\mathit{\Sigma}}_{\star}^{1/2}\right\Vert _{\fro}^{2}
	\end{align*}  
	between $[\mathcal{L}_{k},\mathcal{R}_{k}]$ and $[\mathcal{L}_{\star},\mathcal{R}_{\star}]$ exists, then $\mathcal{Q}_{k}$ obeys the following optimal alignment criterion
	\begin{align}
		\mathcal{L}_{\natural}^{\top} *\Delta_{\mathcal{L}} *{\mathit{\Sigma}}_{\star}={\mathit{\Sigma}}_{\star}*\Delta_{\mathcal{R}}^{\top}*\mathcal{R}_{\natural}.\label{eq:Q_criterion}
	\end{align}
\end{lemma}

\begin{proof} The first order necessary condition (\textit{i.e.}~the gradient is zero) yields 
	\begin{align*}
		& 2\mathcal{L}_{k}^{\top}(\mathcal{L}_{k}*\mathcal{Q}_{k}-\mathcal{L}_{\star})*{\mathit{\Sigma}}_{\star}-2\mathcal{Q}_{k}^{-\top}{\mathit{\Sigma}}_{\star}*(\mathcal{R}_{k}*\mathcal{Q}_{k}^{-\top}-\mathcal{R}_{\star})^{\top}*\mathcal{R}_{k}*\mathcal{Q}_{k}^{-\top}= 0,
	\end{align*}
	which means 
	\begin{align*}
		(\mathcal{L}_{k}*\mathcal{Q}_{k})^{\top}(\mathcal{L}_{k}*\mathcal{Q}_{k}-\mathcal{L}_{\star})*{\mathit{\Sigma}}_{\star}={\mathit{\Sigma}}_{\star}*(\mathcal{R}_{k}*\mathcal{Q}_{k}^{-\top}-\mathcal{R}_{\star})^{\top}*\mathcal{R}_{k}*\mathcal{Q}_{k}^{-\top}.
	\end{align*}
	which completes the proof, in conjunction with Definition~\ref{def: shorthand notation}.
\end{proof}

\begin{lemma} \label{lm:bound of sparse tensor}
	For any $\alpha$-sparse tensor $\mathcal{S}\in \mathbb{R}^{I_1 \times I_2\times I_3}$, the following inequalities hold:
	\begin{align*}
		\|\mathcal{S}\|_2\leq\alpha I_3\sqrt{I_1I_2} \|\mathcal{S}\|_\infty, \cr
		\|\mathcal{S}\|_{2,\infty}\leq \sqrt{\alpha I_2I_3} \|\mathcal{S}\|_\infty, \cr
		\|\mathcal{S}\|_{1,\infty}\leq \alpha I_2 I_3 \|\mathcal{S}\|_\infty.
		% \|\mathcal{S}\|_2\leq\alpha  \sqrt{I_1I_2} \|\mathcal{S}\|_\infty, \cr
		% \|\mathcal{S}\|_{2,\infty}\leq \alpha\sqrt{ I_2I_3} \|\mathcal{S}\|_\infty.
		% \cr
		% \|\mathcal{S}\|_{1,\infty}&\leq \alpha n \|\mathcal{S}\|_\infty .
		%\cr
		%\|\mathcal{S}\|_{\fro}&\leq \sqrt{\alpha} n \|\mathcal{S}\|_\infty.
	\end{align*}
\end{lemma}
\begin{proof}
	If $\mathbf{S}\size^{I_1\times I_2}$ is an $\alpha$-sparse matrix, we have $ \| \mathbf{S} \|_2 \leq\alpha \sqrt{I_1I_2 }\|\mathbf{S}\|_\infty$ as shown in {\cite[Lemma~1]{yi2016fast}}. 
	% When $\mathcal{S}$ is an $\alpha$-sparse tensor, its diagonal matrix $\widehat{\mathbf{S}}\size^{I_1I_3\times I_2I_3}$ will be an $\frac{\alpha}{I_3}$-sparse matrix. Use the fact that $\norm{\mathcal{S}}_{2} = \norm{\widehat{\mathbf{S}}}_{2}$ and $\norm{\mathcal{S}}_{\infty} \neq \norm{\widehat{\mathbf{S}}}_{\infty}$, we proof the first claim.
	When $\mathcal{S}$ is an $\alpha$-sparse tensor, its block circulant matrix $\operatorname{bcirc}(\mathcal{S})\size^{I_1I_3\times I_2I_3}$ will be an $\alpha$-sparse matrix. Use the fact that $\norm{\mathcal{S}}_{2} = \norm{\operatorname{bcirc}(\mathcal{S})}_{2}$ and $\norm{\mathcal{S}}_{\infty} = \norm{\operatorname{bcirc}(\mathcal{S})}_{\infty}$, we proof the first claim.
	
	Moreover, considering the fact that the $\alpha$-sparse tensor $\mathcal{S}$ has at most $\alpha I_2 I_3$ non-zero elements in each slice $\mathcal{S}(i_1,:,:)$, the second claim follows directly from Definition~\ref{def: l2inf} of the $\ell_{2,\infty}$ norm. Similarly, by applying Definition~\ref{def: l1inf} of the $\ell_{1,\infty}$ norm, we immediately obtain the third claim.
	% as the fact that the $\alpha$-sparse tensor $\mathcal{S}$ has at most $\alpha I_1$, $\alpha I_2$ and $\alpha I_3$ non-zero element in each column fiber, row fiber, and tube fiber, 
\end{proof}

\begin{lemma} \label{lm:Delta_F_norm}
	If 
	\begin{gather*} 
		\distk \leq \frac{\epsilon }{\sqrt{I_3}} \tau^k \sigma_{\min}(\mathcal{X}_\star),
	\end{gather*}
	then the following inequalities hold
	\begin{align*}
		\| \Delta_{\mathcal{L}}*\mathit{\Sigma}_\star^{1/2} \|_{\fro} \lor \| \Delta_{\mathcal{R}}*\mathit{\Sigma}_\star^{1/2} \|_{\fro} &\leq \frac{\epsilon }{\sqrt{I_3}} \tau^{k}\sigma_{\min}(\mathcal{X}_\star) \cr
		\| \Delta_{\mathcal{L}}*\mathit{\Sigma}_\star^{1/2} \|_2 \lor \| \Delta_{\mathcal{R}}*\mathit{\Sigma}_\star^{1/2} \|_2 &\leq \varepsilon\tau^{k}\sigma_{\min}(\mathcal{X}_\star).
	\end{align*}
\end{lemma}
\begin{proof}
	Recall the definition that $\distk= \sqrt{\| \Delta_{\mathcal{L}}*\mathit{\Sigma}_\star^{1/2} \|_{\fro}^{2} + \| \Delta_{\mathcal{R}}*\mathit{\Sigma}_\star^{1/2} \|_{\fro}^{2}} \geq \| \Delta_{\mathcal{L}}*\mathit{\Sigma}_\star^{1/2} \|_{\fro} \lor \| \Delta_{\mathcal{R}}*\mathit{\Sigma}_\star^{1/2} \|_{\fro}$.
	The first claim is proofed.
	
	By the fact that $\|\mathcal{A}\|_2\leq \sqrt{I_3}\|\mathcal{A}\|_{\fro}$ for any tensor, we further deduct the second claim from the first claim.
\end{proof}

% \begin{lemma} \label{lm:L_R_2_inf_norm}  %(61)
	% If 
	% \begin{gather*}
		% \distk \leq\varepsilon \tau^k \sigma_{\min}(\mathcal{X}_\star), \cr
		% \|\Delta_{\mathcal{L}}*\mathit{\Sigma}_\star^{1/2}\|_{2,\infty}  \lor \|\Delta_{\mathcal{R}}*\mathit{\Sigma}_\star^{1/2}\|_{2,\infty}\leq\sqrt{\mu R} \tau^k \sigma_{\min}(\mathcal{X}_\star),
		% \end{gather*}
	% then 
	% \begin{gather*}
		%     \|\mathcal{L}_\natural*\mathit{\Sigma}_\star^{1/2}\|_{2,\infty} \lor \|\mathcal{R}_\natural*\mathit{\Sigma}_\star^{1/2}\|_{2,\infty} \leq
		% \end{gather*}
	% \end{lemma}

\begin{lemma}\label{lemma:Weyl} 
	For any $\mathcal{L}_{\natural} \in\mathbb{R}^{I_{1}\times R \times I_{3}},\mathcal{R}_{\natural} \in\mathbb{R}^{ R \times I_{2}\times I_{3}}$, denote $\Delta_{\mathcal{L}}:=\mathcal{L}_{\natural}-\mathcal{L}_{\star}$ and $\Delta_{\mathcal{R}}:=\mathcal{R}_{\natural}-\mathcal{R}_{\star}$ as shown in Definition~\ref{def: shorthand notation}. 
	Suppose that $\|\Delta_{\mathcal{L}}*\SCigma_{\star}^{-1/2}\|_{\op}\vee\|\Delta_{\mathcal{R}}*\SCigma_{\star}^{-1/2}\|_{\op}<1$, then one has 
	\begin{subequations}
		\begin{align}
			\left\Vert \mathcal{L}_{\natural}*(\mathcal{L}_{\natural}^{\top}*\mathcal{L}_{\natural})^{-1}*\SCigma_{\star}^{1/2}\right\Vert _{\op} & \le \frac{1}{1-\|\Delta_{\mathcal{L}}*\SCigma_{\star}^{-1/2}\|_{\op}}; \label{eq:Weyl-1L} \\
			\left\Vert \mathcal{R}_{\natural}*(\mathcal{R}_{\natural}^{\top}*\mathcal{R}_{\natural})^{-1}*\SCigma_{\star}^{1/2}\right\Vert _{\op} & \le \frac{1}{1-\|\Delta_{\mathcal{R}}*\SCigma_{\star}^{-1/2}\|_{\op}}.\label{eq:Weyl-1R} 
			% \left\Vert \mathcal{L}(\mathcal{L}^{\top}*\mathcal{L})^{-1}*\SCigma_{\star}^{1/2}-\mathcal{U}_{\star}\right\Vert _{\op}& \le\frac{\sqrt{2}\|\Delta_{\mathcal{L}}*\SCigma_{\star}^{-1/2}\|_{\op}}{1-\|\Delta_{\mathcal{L}}*\SCigma_{\star}^{-1/2}\|_{\op}}; \label{eq:Weyl-2L} \\
			% \left\Vert \mathcal{R}(\mathcal{R}^{\top}*\mathcal{R})^{-1}*\SCigma_{\star}^{1/2}-\bV_{\star}\right\Vert _{\op}& \le\frac{\sqrt{2}\|\Delta_{\mathcal{R}}*\SCigma_{\star}^{-1/2}\|_{\op}}{1-\|\Delta_{\mathcal{R}}*\SCigma_{\star}^{-1/2}\|_{\op}}. \label{eq:Weyl-2R}
		\end{align}
	\end{subequations}
\end{lemma}

\begin{proof} We only prove claims \eqref{eq:Weyl-1L} on the factor $\mathcal{L}$, while the corresponding claims on the factor $\mathcal{R}$ follow from a similar argument. Use the fact that $\norm{\mathcal{A}}_{\op} =  \sigma_{\max}(\mathcal{A}) = \frac{1}{\sigma_{\min}(\mathcal{A}^{-\To})}$,  we then have
	\begin{align*}
		\left\Vert \mathcal{L}_{\natural}*(\mathcal{L}_{\natural}^{\top}*\mathcal{L}_{\natural})^{-1}*\SCigma_{\star}^{1/2}\right\Vert _{\op}=\frac{1}{\sigma_{\min}(\mathcal{L}_{\natural}*\SCigma_{\star}^{-1/2})}.
	\end{align*}
	In addition, we invoke Weyl's inequality $|\sigma_{\min}(\mathcal{A})-\sigma_{\min}(\mathcal{B})| \le \|\mathcal{A}-\mathcal{B}\|_{\op}$ to obtain 
	\begin{align*}
		\sigma_{\min}(\mathcal{L}_{\natural}*\SCigma_{\star}^{-1/2})
		\ge \sigma_{\min}(\mathcal{L}_{\star}*\SCigma_{\star}^{-1/2})-\|\Delta_{\mathcal{L}}*\SCigma_{\star}^{-1/2}\|_{\op}
		=1-\|\Delta_{\mathcal{L}}*\SCigma_{\star}^{-1/2}\|_{\op},
	\end{align*}
	where we use the facts that $\mathcal{U}_{\star}=\mathcal{L}_{\star}*\SCigma_{\star}^{-1/2}$ and $\sigma_{\min}(\mathcal{U}_{\star})=1$. By combining these two relations, we can complete the proof of \eqref{eq:Weyl-1L}.

\end{proof}

\begin{lemma} \label{lm:L_R_scale_sigma_half}  %(64)
	If 
	\begin{gather*} 
		\distk \leq \frac{\varepsilon}{\sqrt{I_3}}\tau^k \sigma_{\min}(\mathcal{X}_\star),
	\end{gather*}
	then it holds
	\begin{gather*}
		\|\mathcal{L}_\natural*(\mathcal{L}_\natural^\top*\mathcal{L}_\natural)^{-1}*\mathit{\Sigma}_\star^{1/2} \|_2 \lor \|\mathcal{R}_\natural*(\mathcal{R}_\natural^\top*\mathcal{R}_\natural)^{-1}*\mathit{\Sigma}_\star^{1/2} \|_2 
		%\leq \frac{1}{1-\varepsilon\tau^k} 
		\leq \frac{1}{1-\varepsilon} .
	\end{gather*}
\end{lemma}
\begin{proof}
	Lemma~\ref{lemma:Weyl}  provides the following inequalities:
	\begin{gather*}
		\|\mathcal{L}_\natural*(\mathcal{L}_\natural^\top*\mathcal{L}_\natural)^{-1}*\mathit{\Sigma}_\star^{1/2} \|_2 \leq \frac{1}{1-\|\Delta_{\mathcal{L}}*\mathit{\Sigma}_\star^{-1/2}\|_2}, \cr
		\|\mathcal{R}_\natural*(\mathcal{R}_\natural^\top*\mathcal{R}_\natural)^{-1}*\mathit{\Sigma}_\star^{1/2} \|_2 \leq \frac{1}{1-\|\Delta_{\mathcal{R}}*\mathit{\Sigma}_\star^{-1/2}\|_2},
	\end{gather*}
	as long as $\|\Delta_{\mathcal{L}}*\mathit{\Sigma}_\star^{-1/2}\|_2 \lor \|\Delta_{\mathcal{R}}*\mathit{\Sigma}_\star^{-1/2}\|_2<1$. 
	
	By Lemma~\ref{lm:Delta_F_norm}, we have $\|\Delta_{\mathcal{L}}*\mathit{\Sigma}_\star^{-1/2}\|_2 \lor \|\Delta_{\mathcal{R}}*\mathit{\Sigma}_\star^{-1/2}\|_2 \leq \varepsilon \tau^k \leq\varepsilon$, given $\tau=1-0.6\eta<1$. The proof is finished since $\varepsilon=0.02<1$.
\end{proof}

\begin{lemma}\label{lemma:Q_perturbation} Given any $\LC\in\mathbb{R}^{I_{1}\times R \times I_3},\RC\in\mathbb{R}^{I_{2}\times R\times I_3}$ and any invertible matrices $\QC,\ \widetilde{\QC}\in\GL(R)$, assume that $\|(\LC*\QC-\RC_{\star})*\mathit{\Sigma}_{\star}^{-1/2}\|_{\op}\vee\|(\RC*\QC^{-\top}-\RC_{\star})*\mathit{\Sigma}_{\star}^{-1/2}\|_{\op}<1$, then we have
	\begin{align*}
		\left\Vert \mathit{\Sigma}_{\star}^{1/2}*\widetilde{\QC}^{-1}*\QC*\mathit{\Sigma}_{\star}^{1/2}-\mathit{\Sigma}_{\star}\right\Vert _{\op} &\le\frac{\|\RC*(\widetilde{\QC}^{-\top}-\QC^{-\top})*\mathit{\Sigma}_{\star}^{1/2}\|_{\op}}{1-\|(\RC*\QC^{-\top}-\RC_{\star})*\mathit{\Sigma}_{\star}^{-1/2}\|_{\op}}; \\
		\left\Vert \mathit{\Sigma}_{\star}^{1/2}*\widetilde{\QC}^{\top}*\QC^{-\top}*\mathit{\Sigma}_{\star}^{1/2}-\mathit{\Sigma}_{\star}\right\Vert _{\op}  &\le\frac{\|\LC*(\widetilde{\QC}-\QC)*\mathit{\Sigma}_{\star}^{1/2}\|_{\op}}{1-\|(\LC*\QC-\RC_{\star})*\mathit{\Sigma}_{\star}^{-1/2}\|_{\op}}.
	\end{align*}
\end{lemma}
\begin{proof}Inserting $\RC^{\top}*\RC*(\RC^{\top}*\RC)^{-1}$, applying the inequality $\|\AC*\BC\|_{\op} \le \|\AC\|_{\op}\|\BC\|_{\op}$, and Lemma~\ref{lemma:Weyl}, we have
	\begin{align*}
		& \left\Vert \mathit{\Sigma}_{\star}^{1/2}*\widetilde{\QC}^{-1}*\QC*\mathit{\Sigma}_{\star}^{1/2}-\mathit{\Sigma}_{\star}\right\Vert _{\op} \\ 
		=  &\left\Vert \mathit{\Sigma}_{\star}^{1/2}*(\widetilde{\QC}^{-1}-\QC^{-1})*\RC^{\top}*\RC*(\RC^{\top}*\RC)^{-1}*\QC*\mathit{\Sigma}_{\star}^{1/2}\right\Vert _{\op}\\
		\le&  \left\Vert \RC*(\widetilde{\QC}^{-\top}-\QC^{-\top})*\mathit{\Sigma}_{\star}^{1/2}\right\Vert _{\op}\left\Vert \RC*(\RC^{\top}*\RC)^{-1}*\QC*\mathit{\Sigma}_{\star}^{1/2}\right\Vert _{\op}\\
		= & \left\Vert \RC*(\widetilde{\QC}^{-\top}-\QC^{-\top})*\mathit{\Sigma}_{\star}^{1/2}\right\Vert _{\op}\left\Vert \RC*\QC^{-\top}((\RC*\QC^{-\top})^{\top}*\RC*\QC^{-\top})^{-1}\mathit{\Sigma}_{\star}^{1/2}\right\Vert _{\op}\\
		\le &  \frac{\|\RC*(\widetilde{\QC}^{-\top}-\QC^{-\top})*\mathit{\Sigma}_{\star}^{1/2}\|_{\op}}{1-\|(\RC*\QC^{-\top}-\RC_{\star})*\mathit{\Sigma}_{\star}^{-1/2}\|_{\op}}.
	\end{align*}
	
	Similarly, we insert $\LC^{\top}*\LC*(\LC^{\top}*\LC)^{-1}$,  ultilize the relation $\|\AC* \BC\|_{\op} \le \|\AC\|_{\op}\|\BC\|_{\op}$  and  Lemma~\ref{lemma:Weyl}  to get
	\begin{align*}
		&\left\Vert \mathit{\Sigma}_{\star}^{1/2}*\widetilde{\QC}^{\top}*\QC^{-\top}*\mathit{\Sigma}_{\star}^{1/2}-\mathit{\Sigma}_{\star}\right\Vert _{\op}\\
		=&\left\Vert \mathit{\Sigma}_{\star}^{1/2}*(\widetilde{\QC}^{\top}-\QC^{\top})*\LC^{\top}*\LC*(\LC^{\top}*\LC)^{-1}*\QC^{-\top}*\mathit{\Sigma}_{\star}^{1/2}\right\Vert _{\op}\\
		\le  &\left\Vert \LC*(\widetilde{\QC}-\QC)*\mathit{\Sigma}_{\star}^{1/2}\right\Vert _{\op}\left\Vert \LC*(\LC^{\top}*\LC)^{-1}*\QC^{-\top}*\mathit{\Sigma}_{\star}^{1/2}\right\Vert _{\op}\\
		= &\left\Vert \LC*(\widetilde{\QC}-\QC)*\mathit{\Sigma}_{\star}^{1/2}\right\Vert _{\op}\left\Vert \LC*\QC((\LC*\QC)^{\top}*\LC*\QC)^{-1}\mathit{\Sigma}_{\star}^{1/2}\right\Vert _{\op}\\
		\le&\frac{\|\LC*(\widetilde{\QC}-\QC)*\mathit{\Sigma}_{\star}^{1/2}\|_{\op}}{1-\|(\LC*\QC-\RC_{\star})*\mathit{\Sigma}_{\star}^{-1/2}\|_{\op}}.
	\end{align*}
	The proof is finished.
\end{proof}

\begin{lemma} \label{lm:sigma_half_Q-Q_sigma_half}  %under (66), pg 42
	If 
	\begin{gather*}
		\|(\mathcal{L}_{k+1}*\mathcal{Q}_k-\mathcal{L}_\star)*\mathit{\Sigma}_\star^{1/2}\|_2 \lor \|(\mathcal{R}_{k+1}*\mathcal{Q}_k^{-\top}-\mathcal{R}_\star)*\mathit{\Sigma}_\star^{1/2}\|_2 
		\leq\varepsilon \tau^{k+1} \sigma_{\min}(\mathcal{X}_\star) ,
	\end{gather*}
	then
	\begin{gather*}
		\|\mathit{\Sigma}_\star^{1/2}*\mathcal{Q}_k^{-1}*(\mathcal{Q}_{k+1}-\mathcal{Q}_k)*\mathit{\Sigma}_\star^{1/2}\|_2 \lor  \|\mathit{\Sigma}_\star^{1/2}*\mathcal{Q}_k^\top*(\mathcal{Q}_{k+1}-\mathcal{Q}_k)^{-\top}*\mathit{\Sigma}_\star^{1/2}\|_2
		\leq \frac{2\varepsilon}{1-\varepsilon}\sigma_{\min}(\mathcal{X}_\star) .
	\end{gather*}
\end{lemma}
\begin{proof}
	We only examine the first term and obtain the second term by similar steps. By the assumption and the definition of  $\mathcal{Q}_{k+1}$, we have
	\begin{align*}
		\|(\mathcal{R}_{k+1}*\mathcal{Q}_k^{-\top}-\mathcal{R}_\star)*\mathit{\Sigma}_\star^{1/2}\|_2 &\leq \varepsilon \tau^{k+1} \sigma_{\min}(\mathcal{X}_\star) ,\cr
		\|(\mathcal{R}_{k+1}*\mathcal{Q}_{k+1}^{-\top}-\mathcal{R}_\star)*\mathit{\Sigma}_\star^{1/2}\|_2 &\leq \varepsilon \tau^{k+1} \sigma_{\min}(\mathcal{X}_\star) ,\cr
		\|(\mathcal{R}_{k+1}*\mathcal{Q}_{k+1}^{-\top}-\mathcal{R}_\star)*\mathit{\Sigma}_\star^{-1/2}\|_2 &\leq \varepsilon \tau^{k+1} .
	\end{align*}
	Thus, by applying Lemma~\ref{lemma:Q_perturbation} with $\mathcal{R}=\mathcal{R}_{k+1}$, $\widetilde{\mathcal{Q}}=\mathcal{Q}_k$, and $\mathcal{Q}=\mathcal{Q}_{k+1}$, we get
	\begin{align*}
		&\|\mathit{\Sigma}_\star^{1/2}*\mathcal{Q}_k^{-1}*(\mathcal{Q}_{k+1}-\mathcal{Q}_k)*\mathit{\Sigma}_\star^{1/2}\|_2 \cr
		= & \|\mathit{\Sigma}_\star^{1/2}*\mathcal{Q}_k^{-1}*\mathcal{Q}_{k+1}*\mathit{\Sigma}_\star^{1/2}-\mathit{\Sigma}_\star\|_2  \cr
		\leq& \frac{\|\mathcal{R}_{k+1}*(\mathcal{Q}_k^{-\top}-\mathcal{Q}_{k+1}^{-\top})* \mathit{\Sigma}_\star^{1/2}\|_2}{1-\|(\mathcal{R}_{k+1}*\mathcal{Q}_{k+1}^{-\top}-\mathcal{R}_\star)*\mathit{\Sigma}_\star^{-1/2} \|_2} \cr
		\leq& \frac{\|(\mathcal{R}_{k+1}*\mathcal{Q}_k^{-\top}-\mathcal{R}_\star)*\mathit{\Sigma}_\star^{1/2}\|_2 + \|(\mathcal{R}_{k+1}*\mathcal{Q}_{k+1}^{-\top}-\mathcal{R}_\star)*\mathit{\Sigma}_\star^{1/2}\|_2}{1 - \|(\mathcal{R}_{k+1}*\mathcal{Q}_{k+1}^{-\top}-\mathcal{R}_\star)*\mathit{\Sigma}_\star^{-1/2}\|_2} \cr
		\leq& \frac{2\varepsilon \tau^{k+1}}{1-\varepsilon\tau^{k+1}} \sigma_{\min}(\mathcal{X}_\star) 
		\leq  \frac{2\varepsilon }{1-\varepsilon} \sigma_{\min}(\mathcal{X}_\star),
	\end{align*}
	where the last line uses $\tau=1-0.6\eta<1$. 
	% By similar steps, we can obtain that
	% \begin{align*}
		%     \|\mathit{\Sigma}_\star^{1/2}*\mathcal{Q}_k^\top*(\mathcal{Q}_{k+1}-\mathcal{Q}_k)^{-\top}*\mathit{\Sigma}_\star^{1/2}\|_2 \leq \frac{2\varepsilon }{1-\varepsilon} \sigma_{\min}(\mathcal{X}_\star).
		% \end{align*}
	This completes the proof. 
\end{proof}

Notice that Lemma~\ref{lm:sigma_half_Q-Q_sigma_half} will be used solely in the proof of Lemma~\ref{lm:convergence_incoher}. The assumption of Lemma~\ref{lm:sigma_half_Q-Q_sigma_half} is validated in \eqref{eq:dist_k+1_with_Q_k}, as shown in the proof of Lemma~\ref{lm:convergence_dist}.

\subsection{Proof of Theorem~\ref{thm:local convergence} (local linear convergence)} \label{sec:local conv}
We begin by establishing the local linear convergence with respect to the $\ell_\infty$ norm at $k$-th iteration.
Then, we demonstrate the local convergence of the proposed algorithm by proving that the claims hold at the $(k+1)$-th iteration, assuming they hold at the $k$-th iteration.
\begin{lemma} \label{lm:X-X_K_inf_norm}
	Suppose that $\mathcal{X}_\star=\mathcal{L}_\star*\mathcal{R}_\star^\top$ is a tubal rank-$R$ tensor with $\mu$-incoherence. If 
	\begin{gather*}
		\distk \leq \frac{\varepsilon}{\sqrt{I_3}} \tau^k \sigma_{\min}(\mathcal{X}_\star), \cr
		\sqrt{I_1}\|\Delta_{\mathcal{L}}*\mathit{\Sigma}_\star^{1/2}\|_{2,\infty}  \lor \sqrt{I_2}\|\Delta_{\mathcal{R}}*\mathit{\Sigma}_\star^{1/2}\|_{2,\infty}\leq\sqrt{\mu R} \tau^k \sigma_{\min}(\mathcal{X}_\star),
	\end{gather*}
	then
	\begin{gather*}
		\|\mathcal{X}_\star-\mathcal{X}_k\|_\infty \leq \frac{3}{\sqrt{I_1I_2}} \mu R \tau^k \sigma_{\min}(\mathcal{X}_\star).
	\end{gather*}
\end{lemma}
\begin{proof}
	Firstly, by the fact that $\norm{\AC+\BC}_{2,\infty}\le \norm{\AC}_{2,\infty} +\norm{\BC}_{2,\infty}$, $\|\mathcal{A}* \mathcal{B}\|_{2,\infty}\le \|\mathcal{A}\|_{2}\|\mathcal{B}\|_{2,\infty}$, and the assumption of this lemma, we have
	\begin{align*}
		\|\mathcal{R}_\natural*\mathit{\Sigma}_\star^{-1/2}\|_{2,\infty} 
		\leq & \|\Delta_{\mathcal{R}}*\mathit{\Sigma}_\star^{1/2}\|_{2,\infty} \|\mathit{\Sigma}_\star^{-1}\|_2 + \|\mathcal{R}_\star*\mathit{\Sigma}_\star^{-1/2} \|_{2,\infty} \cr
		\leq & \sqrt{\frac{\mu R}{I_2}}\tau^k \sigma_{\min}(\mathcal{X}_\star) \frac{1}{\sigma_{\min}(\mathcal{X}_\star)} + \|\mathcal{V}_\star\|_{2,\infty} \cr
		\leq & \tau^k \sqrt{\frac{\mu R}{I_2}} + \sqrt{\frac{\mu R}{I_2}} \cr
		\leq &2\sqrt{\frac{\mu R}{I_2}},
		% \left(\tau^k+1\right)\sqrt{\mu R} \quad \leq 2\sqrt{\mu R},
	\end{align*}
	given $\tau=1-0.6\eta<1$. Moreover, one can see
	\begin{align*}
		\|\mathcal{X}_\star-\mathcal{X}_k\|_\infty  
		=& \| \Delta_{\mathcal{L}}*\mathcal{R}_\natural^\top+\mathcal{L}_\star*\Delta_{\mathcal{R}}^\top \|_\infty \cr
		\leq& \| \Delta_{\mathcal{L}}*\mathcal{R}_\natural^\top\|_\infty+\|\mathcal{L}_\star*\Delta_{\mathcal{R}}^\top \|_\infty \cr
		\leq &\| \Delta_{\mathcal{L}}*\mathit{\Sigma}_\star^{1/2}\|_{2,\infty}\|\mathcal{R}_\natural*\mathit{\Sigma}_\star^{-1/2}\|_{2,\infty}+\|\mathcal{L}_\star*\mathit{\Sigma}_\star^{-1/2}\|_{2,\infty}\|\Delta_{\mathcal{R}}*\mathit{\Sigma}_\star^{1/2} \|_{2,\infty}\cr
		\leq & \sqrt{\frac{\mu R}{I_1}} \tau^{k} \sigma_{\min}(\mathcal{X}_\star)
		2\sqrt{\frac{\mu R}{I_2}} + \sqrt{\frac{\mu R}{I_1}}
		\sqrt{\frac{\mu R}{I_2}} \tau^k \sigma_{\min}(\mathcal{X}_\star)\cr
		= & \frac{3}{\sqrt{I_1I_2}} \mu R \tau^k \sigma_{\min}(\mathcal{X}_\star).
		% = & 3 \frac{\mu R}{n} \tau^k \sigma_{\min}(\mathcal{X}_\star).
	\end{align*}
	This finishes the proof.
	%\HQ{Need Lemma~\ref{lm:L_R_2_inf_norm} to show $\|\mathcal{R}_\natural*\mathit{\Sigma}_\star^{1/2}\|_{2,\infty}$}
\end{proof}

\begin{lemma} \label{lm:convergence_dist}
	If 
	\begin{gather*}
		\distk \leq \frac{\varepsilon }{\sqrt{I_3}} \tau^k \sigma_{\min}(\mathcal{X}_\star), \cr
		\sqrt{I_1}\|\Delta_{\mathcal{L}}*\mathit{\Sigma}_\star^{1/2}\|_{2,\infty}  \lor \sqrt{I_2}\|\Delta_{\mathcal{R}}*\mathit{\Sigma}_\star^{1/2}\|_{2,\infty}\leq\sqrt{\mu R} \tau^k \sigma_{\min}(\mathcal{X}_\star),
	\end{gather*}
	then
	\begin{gather*}
		\distkplusone \leq \frac{\varepsilon }{\sqrt{I_3}}  \tau^{k+1} \sigma_{\min}(\mathcal{X}_\star).
	\end{gather*}
\end{lemma}

\begin{proof}
	Since $\mathcal{Q}_{k+1}$ is the optimal alignment tensor between $(\mathcal{L}_{k+1},\mathcal{R}_{k+1})$ and $(\mathcal{L}_\star,\mathcal{R}_\star)$, we have
	\begin{align*}
		% \distsquarekplusone := 
		&~\|(\mathcal{L}_{k+1}*\mathcal{Q}_{k+1}-\mathcal{L}_\star)*\mathit{\Sigma}_\star^{1/2}\|_{\fro}^{2} + \|(\mathcal{R}_{k+1}*\mathcal{Q}_{k+1}^{-\top}-\mathcal{R}_\star)*\mathit{\Sigma}_\star^{1/2}\|_{\fro}^{2} \cr
		\leq &~\|(\mathcal{L}_{k+1}*\mathcal{Q}_k-\mathcal{L}_\star)*\mathit{\Sigma}_\star^{1/2}\|_{\fro}^{2} + \|(\mathcal{R}_{k+1}*\mathcal{Q}_k^{-\top}-\mathcal{R}_\star)*\mathit{\Sigma}_\star^{1/2}\|_{\fro}^{2}
	\end{align*}
	We will concentrate on bounding the first term in this proof, and the second term can be bounded in a similar manner.
	
	Note that $\mathcal{Q}_{k}$ is inverible, $\mathcal{L}_\natural*\mathcal{R}_\natural^\top-\mathcal{X}_\star=\Delta_{\mathcal{L}}*\mathcal{R}_\natural^\top+\mathcal{L}_\star*\Delta_{\mathcal{R}}^\top$, $\mathcal{R}_{k}*(\mathcal{R}_{k}^\top*\mathcal{R}_{k})^{-1}*\mathcal{Q}_{k} = \mathcal{R}_\natural*(\mathcal{R}_\natural^\top*\mathcal{R}_\natural)^{-1}$. 
	We have
	\begin{equation}
		\label{eq:LQ-L}
		\begin{split}
			& \mathcal{L}_{k+1}*\mathcal{Q}_k-\mathcal{L}_\star \cr
			=~& \mathcal{L}_\natural-\eta(\mathcal{L}_\natural*\mathcal{R}_\natural^\top-\mathcal{X}_\star+\mathcal{S}_{k+1}-\mathcal{S}_\star)*\mathcal{R}_{k}*(\mathcal{R}_{k}^\top*\mathcal{R}_{k})^{-1}*\mathcal{Q}_{k}-\mathcal{L}_\star \cr
			=~& \mathcal{L}_\natural-\eta(\mathcal{L}_\natural*\mathcal{R}_\natural^\top-\mathcal{X}_\star+\mathcal{S}_{k+1}-\mathcal{S}_\star)*\mathcal{R}_\natural*(\mathcal{R}_\natural^\top*\mathcal{R}_\natural)^{-1}-\mathcal{L}_\star \cr
			=~& \Delta_{\mathcal{L}}-\eta(\mathcal{L}_\natural*\mathcal{R}_\natural^\top-\mathcal{X}_\star)*\mathcal{R}_\natural*(\mathcal{R}_\natural^\top \mathcal{R}_\natural)^{-1}-\eta\Delta_{\mathcal{S}}*\mathcal{R}_\natural*(\mathcal{R}_\natural^\top*\mathcal{R}_\natural)^{-1} \cr
			=~& (1-\eta)\Delta_{\mathcal{L}}-\eta\mathcal{L}_\star*\Delta_{\mathcal{R}}^\top*\mathcal{R}_\natural*(\mathcal{R}_\natural^\top*\mathcal{R}_\natural)^{-1}-\eta\Delta_{\mathcal{S}}*\mathcal{R}_\natural*(\mathcal{R}_\natural^\top*\mathcal{R}_\natural)^{-1}.
		\end{split}
	\end{equation}
	Thus,
	\begin{align*}
		&~\|(\mathcal{L}_{k+1}*\mathcal{Q}_k-\mathcal{L}_\star)*\mathit{\Sigma}_\star^{1/2}\|_{\fro}^{2} \cr =
		&~ \underbrace{\| (1-\eta)\Delta_{\mathcal{L}}*\mathit{\Sigma}_\star^{1/2}-\eta\mathcal{L}_\star*\Delta_{\mathcal{R}}^\top*\mathcal{R}_\natural*(\mathcal{R}_\natural^\top*\mathcal{R}_\natural)^{-1}*\mathit{\Sigma}_\star^{1/2}  \|_{\fro}^{2}}_{\mathfrak{R}_1} \
		\underbrace{- 2\eta(1-\eta) \operatorname{tr}(\Delta_{\mathcal{S}}*\mathcal{R}_\natural*(\mathcal{R}_\natural^\top*\mathcal{R}_\natural)^{-1}*\mathit{\Sigma}_\star*\Delta_{\mathcal{L}}^\top)}_{\mathfrak{R}_2} \cr
		&~ \underbrace{+2\eta^2 \operatorname{tr}(\Delta_{\mathcal{S}}*\mathcal{R}_\natural*(\mathcal{R}_\natural^\top*\mathcal{R}_\natural)^{-1}*\mathit{\Sigma}_\star*(\mathcal{R}_\natural^\top*\mathcal{R}_\natural)^{-1}*\mathcal{R}_\natural^\top*\Delta_{\mathcal{R}}*\mathcal{L}_\star^\top)}_{\mathfrak{R}_3} \ \underbrace{+\eta^2 \|\Delta_{\mathcal{S}}*\mathcal{R}_\natural*(\mathcal{R}_\natural^\top*\mathcal{R}_\natural)^{-1}*\mathit{\Sigma}_\star^{1/2} \|_{\fro}^{2}}_{\mathfrak{R}_4} 
		% :=&~ \mathfrak{R}_1 - \mathfrak{R}_2 + \mathfrak{R}_3 + \mathfrak{R}_4
	\end{align*}
	\begin{enumerate}
		\item\textbf{Bound of $\mathfrak{R}_1$:}
		We can expand $\mathfrak{R}_1$ as
		\begin{equation}
			\label{eq:MF_Lt}
			\begin{split}
				\mathfrak{R}_1 = & (1-\eta)^{2}\tr\left(\Delta_{\mathcal{L}}*{\mathit{\Sigma}}_{\star}*\Delta_{\mathcal{L}}^{\top}\right)\cr
				&-2\eta(1-\eta)\underbrace{\tr\left(\mathcal{L}_{\star}*\Delta_{\mathcal{R}}^{\top}*\mathcal{R}_\natural*(\mathcal{R}_\natural^{\top}*\mathcal{R}_\natural)^{-1}*{\mathit{\Sigma}}_{\star}*\Delta_{\mathcal{L}}^{\top}\right)}_{\mfk{M}_{1}}\nonumber \cr
				&+\eta^{2}\underbrace{\left\Vert \mathcal{L}_{\star}*\Delta_{\mathcal{R}}^{\top}*\mathcal{R}_\natural*(\mathcal{R}_\natural^{\top}*\mathcal{R}_\natural)^{-1}*{\mathit{\Sigma}}_{\star}^{1/2}\right\Vert _{\fro}^{2}}_{\mfk{M}_{2}}.
			\end{split}
		\end{equation}
		The first term $\tr(\Delta_{\mathcal{L}}*{\mathit{\Sigma}}_{\star}*\Delta_{\mathcal{L}}^{\top})$ is closely related to $\distk$, so we will focus on relating $\mfk{M}_{1}$ and $\mfk{M}_{2}$ to $\distk$.
		We begin with the term $\mfk{M}_{1}$. Since $\mathcal{Q}_{k}$ is the optimal alignment tensor between $(\mathcal{L}_{k},\mathcal{R}_{k})$ and $(\mathcal{L}_\star,\mathcal{R}_\star)$, Lemma~\ref{lemma:Q_criterion} states that ${\mathit{\Sigma}}_{\star}*\Delta_{\mathcal{L}}^{\top}*\mathcal{L}_{\natural}=\mathcal{R}_{\natural}^{\top}*\Delta_{\mathcal{R}}*{\mathit{\Sigma}}_{\star}$, and the definition $\mathcal{L}_{\star}=\mathcal{L}_{\natural}-\Delta_{\mathcal{L}}$, we can rewrite $\mfk{M}_{1}$ as
		\begin{align*}
			\mfk{M}_{1}  = ~& \tr\left(\mathcal{R}_{\natural}*(\mathcal{R}_{\natural}^{\top}*\mathcal{R}_{\natural}^{-1}*{\mathit{\Sigma}}_{\star}*\Delta_{\mathcal{L}}^{\top}*\mathcal{L}_{\star}*\Delta_{\mathcal{R}}^{\top}\right)\\
			= ~& \tr\left(\mathcal{R}_{\natural}*(\mathcal{R}_{\natural}^{\top}*\mathcal{R}_{\natural}^{-1}*{\mathit{\Sigma}}_{\star}*\Delta_{\mathcal{L}}^{\top}*\mathcal{L}_{\natural}*\Delta_{\mathcal{R}}^{\top}\right)-\tr\left(\mathcal{R}_{\natural}*(\mathcal{R}_{\natural}^{\top}*\mathcal{R}_{\natural}^{-1}*{\mathit{\Sigma}}_{\star}*\Delta_{\mathcal{L}}^{\top}*\Delta_{\mathcal{L}}*\Delta_{\mathcal{R}}^{\top}\right)\\
			= ~& \tr\left(\mathcal{R}_{\natural}*(\mathcal{R}_{\natural}^{\top}*\mathcal{R}_{\natural}^{-1}*\mathcal{R}_{\natural}^{\top}*\Delta_{\mathcal{R}}*{\mathit{\Sigma}}_{\star}*\Delta_{\mathcal{R}}^{\top}\right)-\tr\left(\mathcal{R}_{\natural}*(\mathcal{R}_{\natural}^{\top}*\mathcal{R}_{\natural}^{-1}*{\mathit{\Sigma}}_{\star}*\Delta_{\mathcal{L}}^{\top}*\Delta_{\mathcal{L}}*\Delta_{\mathcal{R}}^{\top}\right).
		\end{align*}
		As for $\mfk{M}_{2}$, we can utilize the fact $\mathcal{L}_{\star}^{\top}*\mathcal{L}_{\star}={\mathit{\Sigma}}_{\star}$ and the decomposition ${\mathit{\Sigma}}_{\star}=\mathcal{R}_{\natural}^{\top}*\mathcal{R}_{\natural}-(\mathcal{R}_{\natural}^{\top}*\mathcal{R}_{\natural}-{\mathit{\Sigma}}_{\star})$ to obtain 
		\begin{align*}
			\mfk{M}_{2}&= 
			\tr\left(\mathcal{R}_{\natural}*(\mathcal{R}_{\natural}^{\top}*\mathcal{R}_{\natural}^{-1}*{\mathit{\Sigma}}_{\star}*(\mathcal{R}_{\natural}^{\top}*\mathcal{R}_{\natural}^{-1}*\mathcal{R}_{\natural}^{\top}*\Delta_{\mathcal{R}}*{\mathit{\Sigma}}_{\star}*\Delta_{\mathcal{R}}^{\top}\right)\\
			& = \tr\left(\mathcal{R}_{\natural}*(\mathcal{R}_{\natural}^{\top}*\mathcal{R}_{\natural}^{-1}*\mathcal{R}_{\natural}^{\top}*\Delta_{\mathcal{R}}*{\mathit{\Sigma}}_{\star}*\Delta_{\mathcal{R}}^{\top}\right) \\
			& \quad-\tr\left(\mathcal{R}_{\natural}*(\mathcal{R}_{\natural}^{\top}*\mathcal{R}_{\natural}^{-1}*(\mathcal{R}_{\natural}^{\top}*\mathcal{R}_{\natural}-{\mathit{\Sigma}}_{\star})*(\mathcal{R}_{\natural}^{\top}*\mathcal{R}_{\natural}^{-1}*\mathcal{R}_{\natural}^{\top}*\Delta_{\mathcal{R}}*{\mathit{\Sigma}}_{\star}*\Delta_{\mathcal{R}}^{\top}\right).
		\end{align*}
		Substituting  $\mfk{M}_{1}$ and $\mfk{M}_{2}$ back into \eqref{eq:MF_Lt} yields 
		\begin{align*}
			\mathfrak{R}_1 = 
			&  (1-\eta)^{2}\tr\left(\Delta_{\mathcal{L}}*{\mathit{\Sigma}}_{\star}*\Delta_{\mathcal{L}}^{\top}\right)-\eta(2-3\eta)\underbrace{\tr\left(\mathcal{R}_{\natural}*(\mathcal{R}_{\natural}^{\top}*\mathcal{R}_{\natural}^{-1}*\mathcal{R}_{\natural}^{\top}*\Delta_{\mathcal{R}}*{\mathit{\Sigma}}_{\star}*\Delta_{\mathcal{R}}^{\top}\right)}_{\mfk{F}_{1}}\\
			&+2\eta(1-\eta)\underbrace{\tr\left(\mathcal{R}_{\natural}*(\mathcal{R}_{\natural}^{\top}*\mathcal{R}_{\natural}^{-1}*{\mathit{\Sigma}}_{\star}*\Delta_{\mathcal{L}}^{\top}*\Delta_{\mathcal{L}}*\Delta_{\mathcal{R}}^{\top}\right)}_{\mfk{F}_{2}}\\
			& -\eta^{2}\underbrace{\tr\left(\mathcal{R}_{\natural}*(\mathcal{R}_{\natural}^{\top}*\mathcal{R}_{\natural}^{-1}*(\mathcal{R}_{\natural}^{\top}*\mathcal{R}_{\natural}-{\mathit{\Sigma}}_{\star})*(\mathcal{R}_{\natural}^{\top}*\mathcal{R}_{\natural}^{-1}*\mathcal{R}_{\natural}^{\top}*\Delta_{\mathcal{R}}*{\mathit{\Sigma}}_{\star}*\Delta_{\mathcal{R}}^{\top}\right)}_{\mfk{F}_{3}}.
		\end{align*}
		Then we separately control three terms $\mfk{F}_{1},\mfk{F}_{2}$ and $\mfk{F}_{3}$. 
		\begin{enumerate}
			\item Notice that $\mfk{F}_{1}$ is the inner product of two positive semi-definite tensors $\mathcal{R}_{\natural}*(\mathcal{R}_{\natural}^{\top}*\mathcal{R}_{\natural}^{-1}*\mathcal{R}_{\natural}^{\top}$ and $\Delta_{\mathcal{R}}*{\mathit{\Sigma}}_{\star}*\Delta_{\mathcal{R}}^{\top}$. Consequently, we have $\mfk{F}_{1}\ge0$.
			\item For $\mfk{F}_{2}$, we need control on $\|\Delta_{\mathcal{L}}*{\mathit{\Sigma}}_{\star}^{-1/2}\|_{\op}$ and $\|\Delta_{\mathcal{R}}*{\mathit{\Sigma}}_{\star}^{-1/2}\|_{\op}$. Similar to the derivation in \eqref{eq:cond_RPCA_op}, we have
			% he first induction hypothesis 
			% \begin{align*}
				% \distk=\sqrt{\|\Delta_{\mathcal{L}}*{\mathit{\Sigma}}_{\star}^{-1/2}{\mathit{\Sigma}}_{\star}\|_{\fro}^{2}+\|\Delta_{\mathcal{R}}*{\mathit{\Sigma}}_{\star}^{-1/2}{\mathit{\Sigma}}_{\star}\|_{\fro}^{2}} & \le\epsilon\sigma_{\min}(\mathcal{X}_{\star})
				% \end{align*}
			% together with the relation $\|\mathcal{A}*\mathcal{B}\|_{\fro}\ge\|\mathcal{A}\|_{\fro}\sigma_{\min}(\mathcal{B})$ tells that 
			% \begin{align*}
				% \sqrt{\|\Delta_{\mathcal{L}}*{\mathit{\Sigma}}_{\star}^{-1/2}\|_{\fro}^{2}+\|\Delta_{\mathcal{R}}*{\mathit{\Sigma}}_{\star}^{-1/2}\|_{\fro}^{2}}\;\sigma_{\min}(\mathcal{X}_{\star})\le\epsilon\sigma_{\min}(\mathcal{X}_{\star}).
				% \end{align*}
			% In light of the relation $\|\mathcal{A}\|_{\op}\le\|\mathcal{A}\|_{\fro}$, this further implies 
			\begin{align}
				\|\Delta_{\mathcal{L}}*{\mathit{\Sigma}}_{\star}^{-1/2}\|_{\op}\vee\|\Delta_{\mathcal{R}}*{\mathit{\Sigma}}_{\star}^{-1/2}\|_{\op}\le\epsilon.\label{eq:cond_MF}
			\end{align}
			Moreover, by invoking Lemma~\ref{lm:L_R_scale_sigma_half} with the underlying condition $\tau\leq1$, we obtain             
			\begin{align*}
				\left\Vert \mathcal{R}_{\natural}*(\mathcal{R}_{\natural}^{\top}*\mathcal{R}_{\natural}^{-1}*{\mathit{\Sigma}}_{\star}^{1/2}\right\Vert _{\op} & \le \frac{1}{1-\epsilon}.
			\end{align*}
			Given these results and the fact that $\operatorname{tr}(\AC * \BC) \le \norm{\AC}_{\op} \operatorname{tr}(\BC),\|\mathcal{A}* \mathcal{B}\|_{\op}\le \|\mathcal{A}\|_{\op}\|\mathcal{B}\|_{\op}$, we can have
			\begin{align*}
				|\mfk{F}_{2}| 
				=&|\tr\left({\mathit{\Sigma}}_{\star}^{-1/2}\Delta_{\mathcal{R}}^{\top}*\mathcal{R}_{\natural}*(\mathcal{R}_{\natural}^{\top}*\mathcal{R}_{\natural}^{-1}*{\mathit{\Sigma}}_{\star}*\Delta_{\mathcal{L}}^{\top}*\Delta_{\mathcal{L}}*{\mathit{\Sigma}}_{\star}^{1/2}\right)| \\
				\le & \left\Vert {\mathit{\Sigma}}_{\star}^{-1/2}\Delta_{\mathcal{R}}^{\top}*\mathcal{R}_{\natural}*(\mathcal{R}_{\natural}^{\top}*\mathcal{R}_{\natural}^{-1}*{\mathit{\Sigma}}_{\star}^{1/2}\right\Vert _{\op}\tr\left({\mathit{\Sigma}}_{\star}^{1/2}\Delta_{\mathcal{L}}^{\top}*\Delta_{\mathcal{L}}*{\mathit{\Sigma}}_{\star}^{1/2}\right)\\
				\le & \|\Delta_{\mathcal{R}}*{\mathit{\Sigma}}_{\star}^{-1/2}\|_{\op}\left\Vert \mathcal{R}_{\natural}*(\mathcal{R}_{\natural}^{\top}*\mathcal{R}_{\natural}^{-1}*{\mathit{\Sigma}}_{\star}^{1/2}\right\Vert _{\op}\tr\left(\Delta_{\mathcal{L}}*{\mathit{\Sigma}}_{\star}*\Delta_{\mathcal{L}}^{\top}\right)\\
				\le & \frac{\epsilon}{1-\epsilon}\tr\left(\Delta_{\mathcal{L}}*{\mathit{\Sigma}}_{\star}*\Delta_{\mathcal{L}}^{\top}\right).
			\end{align*}
			
			\item Similarly, one can bound $|\mfk{F}_{3}|$ by 
			\begin{align*}
				|\mfk{F}_{3}| \le &  \left\Vert \mathcal{R}_{\natural}*(\mathcal{R}_{\natural}^{\top}*\mathcal{R}_{\natural}^{-1}*(\mathcal{R}_{\natural}^{\top}*\mathcal{R}_{\natural}-{\mathit{\Sigma}}_{\star})*(\mathcal{R}_{\natural}^{\top}*\mathcal{R}_{\natural}^{-1}*\mathcal{R}_{\natural}^{\top}\right\Vert _{\op} \tr\left(\Delta_{\mathcal{R}}*{\mathit{\Sigma}}_{\star}*\Delta_{\mathcal{R}}^{\top}\right)\\
				\le & \left\Vert \mathcal{R}_{\natural}*(\mathcal{R}_{\natural}^{\top}*\mathcal{R}_{\natural}^{-1}*{\mathit{\Sigma}}_{\star}^{1/2}\right\Vert _{\op}^{2}\left\Vert {\mathit{\Sigma}}_{\star}^{-1/2}*(\mathcal{R}_{\natural}^{\top}*\mathcal{R}_{\natural}-{\mathit{\Sigma}}_{\star})*{\mathit{\Sigma}}_{\star}^{-1/2}\right\Vert _{\op}\tr\left(\Delta_{\mathcal{R}}*{\mathit{\Sigma}}_{\star}*\Delta_{\mathcal{R}}^{\top}\right)\\
				\le &  \frac{1}{(1-\epsilon)^{2}}\left\Vert {\mathit{\Sigma}}_{\star}^{-1/2}*(\mathcal{R}_{\natural}^{\top}*\mathcal{R}_{\natural}-{\mathit{\Sigma}}_{\star})*{\mathit{\Sigma}}_{\star}^{-1/2}\right\Vert _{\op}\tr\left(\Delta_{\mathcal{R}}*{\mathit{\Sigma}}_{\star}*\Delta_{\mathcal{R}}^{\top}\right).
			\end{align*}
			According to the fact that $\mathcal{R}_{\natural} = \Delta_{\mathcal{R}} + \mathcal{R}_{\star}$, $\mathcal{R}_{\star}^{\top}*\mathcal{R}_{\star} =  {\mathit{\Sigma}}_{\star}$, $\RC_\star =\VC_\star*\mathit{\Sigma}_\star^{1/2}$, $\norm{\AC + \BC }_{\op} \leq \norm{\AC }_{\op} + \norm{\BC}_{\op}$, and if $\mathcal{V}$ is an orthogonal tensor then $\norm{\AC }_{\op} = \norm{\AC * \VC}_{\op} = \norm{\VC * \AC}_{\op}$, one have
			\begin{align*}
				&\left\Vert {\mathit{\Sigma}}_{\star}^{-1/2}*(\mathcal{R}_{\natural}^{\top}*\mathcal{R}_{\natural}-{\mathit{\Sigma}}_{\star})*{\mathit{\Sigma}}_{\star}^{-1/2}\right\Vert _{\op} \\
				=~& \left\Vert {\mathit{\Sigma}}_{\star}^{-1/2}*(\mathcal{R}_{\star}^{\top}*\Delta_{\mathcal{R}}+\Delta_{\mathcal{R}}^{\top}*\mathcal{R}_{\star}+\Delta_{\mathcal{R}}^{\top}*\Delta_{\mathcal{R}})*{\mathit{\Sigma}}_{\star}^{-1/2}\right\Vert _{\op}\\
				\le~&  2\|\Delta_{\mathcal{R}}*{\mathit{\Sigma}}_{\star}^{-1/2}\|_{\op}+\|\Delta_{\mathcal{R}}*{\mathit{\Sigma}}_{\star}^{-1/2}\|_{\op}^{2}
				\le 2\epsilon+\epsilon^{2}.
			\end{align*}
			Then we take the preceding two bounds together to obtain
			\begin{align*}
				|\mfk{F}_{3}|\le\frac{2\epsilon+\epsilon^{2}}{(1-\epsilon)^{2}}\tr\left(\Delta_{\mathcal{R}}*{\mathit{\Sigma}}_{\star}*\Delta_{\mathcal{R}}^{\top}\right).
			\end{align*}
		\end{enumerate}
		
		Combining the bounds for $\mfk{F}_{1},\mfk{F}_{2},\mfk{F}_{3}$  with the underlying condition $0<\eta\leq\frac{2}{3}, \tau\leq1$ , one has 
		\begin{align*}
			\mathfrak{R}_1 
			=~& \left\Vert (1-\eta)\Delta_{\mathcal{L}}*{\mathit{\Sigma}}_{\star}^{1/2}-\eta\mathcal{L}_{\star}*\Delta_{\mathcal{R}}^{\top}*\mathcal{R}_{\natural}*(\mathcal{R}_{\natural}^{\top}*\mathcal{R}_{\natural}^{-1}*{\mathit{\Sigma}}_{\star}^{1/2}\right\Vert _{\fro}^{2}\nonumber \\
			\le~&\left((1-\eta)^{2}+\frac{2\epsilon}{1-\epsilon}\eta(1-\eta)\right)\tr\left(\Delta_{\mathcal{L}}*{\mathit{\Sigma}}_{\star}*\Delta_{\mathcal{L}}^{\top}\right)+\frac{2\epsilon+\epsilon^{2}}{(1-\epsilon)^{2}}\eta^{2}\tr\left(\Delta_{\mathcal{R}}*{\mathit{\Sigma}}_{\star}*\Delta_{\mathcal{R}}^{\top}\right) \\
			=~&\left((1-\eta)^{2}+\frac{2\epsilon}{1-\epsilon}\eta(1-\eta)\right)\|\Delta_{\mathcal{L}}*\mathit{\Sigma}_\star^{1/2}\|_{\fro}^{2} +\frac{2\epsilon+\epsilon^{2}}{(1-\epsilon)^{2}}\eta^{2}\|\Delta_{\mathcal{R}}*\mathit{\Sigma}_\star^{1/2}\|_{\fro}^{2}.\label{eq:MF_Lt_bound}
		\end{align*}

		\item \textbf{Bound of $\mathfrak{R}_2$:} Lemma~\ref{lm:sparity} implies that $\Delta_{\mathcal{S}}=\mathcal{S}_{k+1}-\mathcal{S}_\star$ is an $\alpha$-sparse tensor. Thus, by the properties $ \tr(\mathcal{A}*\mathcal{B})  \le \|\mathcal{A}\|_{2}\|\mathcal{B}\|_{*}$; $\|\mathcal{A}\|_{*} \le \sqrt{R}\|\mathcal{A}\|_{\fro} $ for $\mathcal{A}$ with the tubal rank $R$ and Lemma~\ref{lm:bound of sparse tensor} ;
		Lemma~\ref{lm:sparity} and $\|\mathcal{A}* \mathcal{B}\|_{\fro}\le \|\mathcal{A}\|_{2}\|\mathcal{B}\|_{\fro}$; 
		Lemma~\ref{lm:X-X_K_inf_norm}, \ref{lm:L_R_scale_sigma_half} and \ref{lm:Delta_F_norm} and denote $ \zeta_{k+1} = \zeta_{k} \tau = \zeta_{1} \tau^k := \frac{3}{\sqrt{I_1I_2}} \mu R \tau^k \sigma_{\min}(\mathcal{X}_\star)$, we have
		\begin{align*}
			&| \operatorname{tr}(\Delta_{\mathcal{S}}*\mathcal{R}_\natural*(\mathcal{R}_\natural^\top*\mathcal{R}_\natural)^{-1}*\mathit{\Sigma}_\star*\Delta_{\mathcal{L}}^\top)| \cr
			\leq & \|\Delta_{\mathcal{S}}\|_2 \|\mathcal{R}_\natural*(\mathcal{R}_\natural^\top*\mathcal{R}_\natural)^{-1}*\mathit{\Sigma}_\star*\Delta_{\mathcal{L}}^\top\|_* \cr
			\leq &\alpha I_3 \sqrt{I_1I_2} \sqrt{R} \|\Delta_{\mathcal{S}}\|_\infty  \|\mathcal{R}_\natural*(\mathcal{R}_\natural^\top*\mathcal{R}_\natural)^{-1}*\mathit{\Sigma}_\star*\Delta_{\mathcal{L}}^\top\|_{\fro} \cr
			\leq & 2\alpha  I_3 \sqrt{I_1I_2R} \zeta_{k+1} \|\mathcal{R}_\natural*(\mathcal{R}_\natural^\top*\mathcal{R}_\natural)^{-1}*\mathit{\Sigma}_\star^{1/2}\|_2 \|\Delta_{\mathcal{L}}*\mathit{\Sigma}_\star^{1/2}\|_{\fro}\cr
			\leq & 6 I_3 \alpha  \mu R^{1.5} \frac{1}{1-\varepsilon}  \tau^{k} \sigma_{\min}(\mathcal{X}_\star)\|\Delta_{\mathcal{L}}*\mathit{\Sigma}_\star^{1/2}\|_{\fro} \cr
			\leq & 6 \alpha  \mu R^{1.5} \sqrt{I_{3}} \frac{\varepsilon}{(1-\varepsilon)}  \tau^{2k} \sigma_{\min}^2(\mathcal{X}_\star)
			% 6\alpha  \mu R^{1.5} \tau^{2k} \frac{\varepsilon}{1-\varepsilon}  \sigma_{\min}^2(\mathcal{X}_\star) .
		\end{align*}
		Hence,
		\begin{equation*}
			|\mathfrak{R}_2|\leq 12 I_{3}\eta(1-\eta)\alpha  \mu R^{1.5}  \frac{1}{1-\varepsilon}  \tau^{k} \sigma_{\min}(\mathcal{X}_\star) \|\Delta_{\mathcal{L}}*\mathit{\Sigma}_\star^{1/2}\|_{\fro}.
		\end{equation*}
		% \begin{equation*}
			%     |\mathfrak{R}_2|\leq 12\eta(1-\eta)\alpha  \mu R^{1.5} \tau^{2k} \frac{\varepsilon}{1-\varepsilon}  \sigma_{\min}^2(\mathcal{X}_\star).
			% \end{equation*}
		%\HQ{WIll need a Lemma~\ref{lm:Delta_F_norm} to show $\|\Delta_{\mathcal{L}}*\mathit{\Sigma}_\star^{1/2}\|_{\fro}\leq \varepsilon\tau^k\sigma_{\min}(\mathcal{X}_\star)$.}

		\item \textbf{Bound of $\mathfrak{R}_3$:} Similar to $\mathfrak{R}_2$, we have
		\begin{align*}
			&~ | \operatorname{tr}(\Delta_{\mathcal{S}}*\mathcal{R}_\natural*(\mathcal{R}_\natural^\top*\mathcal{R}_\natural)^{-1}*\mathit{\Sigma}_\star*(\mathcal{R}_\natural^\top*\mathcal{R}_\natural)^{-1}*\mathcal{R}_\natural^\top*\Delta_{\mathcal{R}}*\mathcal{L}_\star^\top)| \cr
			\leq&~ \|\Delta_{\mathcal{S}}\|_2 \|\mathcal{R}_\natural*(\mathcal{R}_\natural^\top*\mathcal{R}_\natural)^{-1}*\mathit{\Sigma}_\star*(\mathcal{R}_\natural^\top*\mathcal{R}_\natural)^{-1}*\mathcal{R}_\natural^\top*\Delta_{\mathcal{R}}*\mathcal{L}_\star^\top\|_* \cr
			\leq&~ \alpha I_{3} \sqrt{I_1I_2} \sqrt{R} \|\Delta_{\mathcal{S}}\|_\infty \|\mathcal{R}_\natural*(\mathcal{R}_\natural^\top*\mathcal{R}_\natural)^{-1}*\mathit{\Sigma}_\star*(\mathcal{R}_\natural^\top*\mathcal{R}_\natural)^{-1}*\mathcal{R}_\natural^\top*\Delta_{\mathcal{R}}*\mathcal{L}_\star^\top\|_{\fro} \cr
			\leq&~ \alpha I_{3} \sqrt{I_1I_2R} \|\Delta_{\mathcal{S}}\|_\infty \|\mathcal{R}_\natural*(\mathcal{R}_\natural^\top*\mathcal{R}_\natural)^{-1}*\mathit{\Sigma}_\star^{1/2}\|_2^2\|\Delta_{\mathcal{R}}*\mathcal{L}_\star^\top\|_{\fro} \cr
			\leq&~ 2\alpha I_{3} \sqrt{I_1I_2R}  \zeta_{k+1}\|\mathcal{R}_\natural*(\mathcal{R}_\natural^\top*\mathcal{R}_\natural)^{-1}*\mathit{\Sigma}_\star^{1/2}\|_2^2\|\Delta_{\mathcal{R}}*\mathit{\Sigma}_\star^{1/2}\|_{\fro} \|\mathcal{U}_\star\|_2 \cr
			\leq&~ 6 I_{3} \alpha  \mu R^{1.5} \frac{1}{ (1-\varepsilon)^2}  \tau^{k} \sigma_{\min}(\mathcal{X}_\star)\|\Delta_{\mathcal{R}}*\mathit{\Sigma}_\star^{1/2}\|_{\fro} \cr
			\leq&~ 6  \alpha  \mu R^{1.5} \sqrt{I_{3}} \frac{\varepsilon}{ (1-\varepsilon)^2}  \tau^{2k} \sigma_{\min}^2(\mathcal{X}_\star).
			% 6\alpha  \mu R^{1.5} \tau^{2k} \frac{\varepsilon}{(1-\varepsilon)^2}  \sigma_{\min}^2(\mathcal{X}_\star) .
		\end{align*}
		% where we used the fact $\|\Delta_{\mathcal{R}}*\mathcal{L}_\star^\top\|_{\fro}=\|\Delta_{\mathcal{R}}*\mathit{\Sigma}_\star^{1/2}*\mathcal{U}_\star^\top\|_{\fro}\leq\|\Delta_{\mathcal{R}}*\mathit{\Sigma}_\star^{1/2}\|_{\fro}$
		Hence, 
		\begin{equation*}
			|\mathfrak{R}_3| \leq 12 I_{3} \eta^2\alpha  \mu R^{1.5}  \frac{1}{(1-\varepsilon)^2}  \tau^{k} \sigma_{\min}(\mathcal{X}_\star) \|\Delta_{\mathcal{R}}*\mathit{\Sigma}_\star^{1/2}\|_{\fro}.
			% \leq 12 \eta^2\alpha  \mu R^{1.5}  \frac{\varepsilon}{\sqrt{I_{3}}(1-\varepsilon)^2}  \tau^{2k} \sigma_{\min}^2(\mathcal{X}_\star).
		\end{equation*}
		
		\item \textbf{Bound of $\mathfrak{R}_4$:}
		\begin{align*} &~\|\Delta_{\mathcal{S}}*\mathcal{R}_\natural*(\mathcal{R}_\natural^\top*\mathcal{R}_\natural)^{-1}*\mathit{\Sigma}_\star^{1/2} \|_{\fro}^{2}\cr
			\leq&~  R  \|\Delta_{\mathcal{S}}*\mathcal{R}_\natural*(\mathcal{R}_\natural^\top*\mathcal{R}_\natural)^{-1}*\mathit{\Sigma}_\star^{1/2} \|_2^2 \cr 
			\leq&~  R\|\Delta_{\mathcal{S}}\|_2^2 \|\mathcal{R}_\natural*(\mathcal{R}_\natural^\top*\mathcal{R}_\natural)^{-1}*\mathit{\Sigma}_\star^{1/2} \|_2^2 \cr 
			\leq&~  4 {I_{3}}^2 \alpha^2 I_1 I_2 R  \zeta_{k+1}^2 \|\mathcal{R}_\natural*(\mathcal{R}_\natural^\top*\mathcal{R}_\natural)^{-1}*\mathit{\Sigma}_\star^{1/2} \|_2^2 \cr 
			\leq&~  36 {I_{3}}^2 \alpha^2 \mu^2 R^3 \tau^{2k} \frac{1}{(1-\varepsilon)^2}  \sigma_{\min}^2(\mathcal{X}_\star).
		\end{align*}
		Hence,
		\begin{equation*}
			\mathfrak{R}_4\leq 36 {I_{3}}^2 \eta^2  \alpha^2 \mu^2 R^3 \tau^{2k} \frac{1}{(1-\varepsilon)^2}  \sigma_{\min}^2(\mathcal{X}_\star)
		\end{equation*}
		
		Combining all four bounds, we have
		\begin{align*}
			&~\|(\mathcal{L}_{k+1}*\mathcal{Q}_k-\mathcal{L}_\star)*\mathit{\Sigma}_\star^{1/2}\|_{\fro}^{2} \cr 
			\leq &~ \left((1-\eta)^{2}+\frac{2\epsilon}{1-\epsilon}\eta(1-\eta)\right)\|\Delta_{\mathcal{L}}*\mathit{\Sigma}_\star^{1/2}\|_{\fro}^{2} +\frac{2\epsilon+\epsilon^{2}}{(1-\epsilon)^{2}}\eta^{2}\|\Delta_{\mathcal{R}}*\mathit{\Sigma}_\star^{1/2}\|_{\fro}^{2}\cr
			&~ + 12I_{3}\eta(1-\eta)\alpha  \mu R^{1.5}  \frac{1}{1-\varepsilon}  \tau^{k} \sigma_{\min}(\mathcal{X}_\star) \|\Delta_{\mathcal{L}}*\mathit{\Sigma}_\star^{1/2}\|_{\fro} \cr
			&~ + 12 I_{3} \eta^2\alpha  \mu R^{1.5}  \frac{1}{(1-\varepsilon)^2}  \tau^{k} \sigma_{\min}(\mathcal{X}_\star) \|\Delta_{\mathcal{R}}*\mathit{\Sigma}_\star^{1/2}\|_{\fro} \cr
			&~ + 36{I_{3}}^2 \eta^2 \alpha^2 \mu^2 R^3 \tau^{2k} \frac{1}{(1-\varepsilon)^2}  \sigma_{\min}^2(\mathcal{X}_\star),
		\end{align*}
		% \begin{align*}
			%     &~\|(\mathcal{L}_{k+1}*\mathcal{Q}_k-\mathcal{L}_\star)*\mathit{\Sigma}_\star^{1/2}\|_{\fro}^{2} \cr 
			%     \leq&~  (1-\eta)^2\|\Delta_{\mathcal{L}}*\mathit{\Sigma}_\star^{1/2}\|_{\fro}^{2} \cr
			%     &~ + \left((1-\eta) \frac{2\varepsilon^3}{1-\varepsilon}  +\eta\frac{2\varepsilon^3+\varepsilon^4}{(1-\varepsilon)^2}\right)\eta\tau^{2k}\sigma_{\min}^2(\mathcal{X}_\star) \cr
			%     &~ + 12\eta(1-\eta)\alpha  \mu R^{1.5} \tau^{2k} \frac{\varepsilon}{1-\varepsilon}  \sigma_{\min}^2(\mathcal{X}_\star) \cr
			%     &~ + 12\eta^2\alpha  \mu R^{1.5} \tau^{2k} \frac{\varepsilon}{(1-\varepsilon)^2}  \sigma_{\min}^2(\mathcal{X}_\star) \cr
			%     &~ + 36\alpha^2 \mu^2 r^3 \tau^{2k} \frac{1}{(1-\varepsilon)^2}  \sigma_{\min}^2(\mathcal{X}_\star),
			% \end{align*}
		and a similar bound can be derived for $\|(\mathcal{R}_{k+1}*\mathcal{Q}_k^{-\top}-\mathcal{R}_\star)*\mathit{\Sigma}_\star^{1/2}\|_{\fro}^{2}$. Add together, we have
		\begin{align} \label{eq:dist_k+1_with_Q_k}
			&~\distsquarekplusone \cr
			\leq&~ \|(\mathcal{L}_{k+1}*\mathcal{Q}_k-\mathcal{L}_\star)*\mathit{\Sigma}_\star^{1/2}\|_{\fro}^{2} + \|(\mathcal{R}_{k+1}*\mathcal{Q}_k^{-\top}-\mathcal{R}_\star)*\mathit{\Sigma}_\star^{1/2}\|_{\fro}^{2} \cr
			\leq&~  \left((1-\eta)^{2}+\frac{2\epsilon}{1-\epsilon}\eta(1-\eta)+\frac{2\epsilon+\epsilon^{2}}{(1-\epsilon)^{2}}\eta^{2}\right) \distsquarek \cr
			&~ + 12I_{3}\left( \frac{\eta(1-\eta)}{1-\varepsilon}  +\frac{\eta^2}{(1-\varepsilon)^2}\right)\alpha  \mu R^{1.5} \tau^{k}\sigma_{\min}(\mathcal{X}_\star) \left(\|\Delta_{\mathcal{L}}*\mathit{\Sigma}_\star^{1/2}\|_{\fro}+\|\Delta_{\mathcal{R}}*\mathit{\Sigma}_\star^{1/2}\|_{\fro}\right) \cr 
			&~ + 72{I_{3}}^2 \eta^2\alpha^2 \mu^2 R^3 \tau^{2k} \frac{1}{(1-\varepsilon)^2}  \sigma_{\min}^2(\mathcal{X}_\star) \cr
			\leq&~  \Bigg((1-\eta)^{2}+\frac{2\epsilon}{1-\epsilon}\eta(1-\eta)+\frac{2\epsilon+\epsilon^{2}}{(1-\epsilon)^{2}}\eta^{2}+ 12\sqrt{2} I_3^{1.5} \alpha \mu R^{1.5}\left( \frac{\eta(1-\eta)}{\varepsilon(1-\varepsilon)}  +\frac{\eta^2}{\varepsilon(1-\varepsilon)^2}\right) \cr
			&~ + 72{I_{3}}^3 \eta^2 \alpha^2 \mu^2 R^3 \frac{1}{\varepsilon^2(1-\varepsilon)^2} \Bigg)  \frac{1}{I_3}\varepsilon^2\tau^{2k} \sigma_{\min}^2(\mathcal{X}_\star)\cr
			\leq&~  \Bigg((1-\eta)^{2}+\frac{2\epsilon}{1-\epsilon}\eta(1-\eta)+\frac{2\epsilon+\epsilon^{2}}{(1-\epsilon)^{2}}\eta^{2} + 12\sqrt{2}/10^{4} \left( \frac{\eta(1-\eta)}{\varepsilon(1-\varepsilon)}  +\frac{\eta^2}{\varepsilon(1-\varepsilon)^2}\right)  + \frac{72/10^{8} \eta^2  }{\varepsilon^2(1-\varepsilon)^2} \Bigg)  \frac{1}{I_3}\varepsilon^2\tau^{2k} \sigma_{\min}^2(\mathcal{X}_\star)\cr
			\leq&~ (1-0.8\eta)^2 \frac{1}{I_3}\varepsilon^2\tau^{2k} \sigma_{\min}^2(\mathcal{X}_\star) ,
		\end{align}
		where we use the fact $\|\Delta_{\mathcal{L}}*\mathit{\Sigma}_\star^{1/2}\|_{\fro}^{2}+\|\Delta_{\mathcal{R}}*\mathit{\Sigma}_\star^{1/2}\|_{\fro}^{2} =: \distsquarek \leq \frac{\varepsilon^2}{I_3}\tau^{2k}\sigma_{\min}^2(\mathcal{X}_\star)$ and apply the Cauchy-Schwarz inequality $a+b\leq \sqrt{2(a^2+b^2)}$ in the 3rd step,  use  $\alpha\leq\frac{1}{10^4\mu R^{1.5} {I_3}^{1.5}}$ in the 4th step, and the last step use $\varepsilon=0.02$, $0<\eta\leq \frac{2}{3}$. The proof is finished by substituting $\tau=1-0.8\eta \leq 1-0.6\eta $.

	\end{enumerate}
\end{proof}

\begin{lemma} \label{lm:convergence_incoher}
	If 
	\begin{equation}
		\begin{gathered}
			\distk \leq \frac{\varepsilon }{\sqrt{I_3}}\tau^k \sigma_{\min}(\mathcal{X}_\star), \\
			\sqrt{I_1} \|\Delta_{\mathcal{L}}*\mathit{\Sigma}_\star^{1/2}\|_{2,\infty}  \lor \sqrt{I_2} \|\Delta_{\mathcal{R}}*\mathit{\Sigma}_\star^{1/2}\|_{2,\infty} \leq \sqrt{\mu R} \tau^k \sigma_{\min}(\mathcal{X}_\star),
		\end{gathered}
	\end{equation}
	then
	\begin{equation}
		\begin{gathered}
			\sqrt{I_1}\|(\mathcal{L}_{k+1}*\mathcal{Q}_{k+1}-\mathcal{L}_\star)*\mathit{\Sigma}_\star^{1/2}\|_{2,\infty}  \lor \sqrt{I_2}\|(\mathcal{R}_{k+1}*\mathcal{Q}_{k+1}^{-\top}-\mathcal{R}_\star)*\mathit{\Sigma}_\star^{1/2}\|_{2,\infty}
			\leq\sqrt{\mu R} \tau^{k+1} \sigma_{\min}(\mathcal{X}_\star).
		\end{gathered}
	\end{equation}
\end{lemma}
\begin{proof}
	Using \eqref{eq:LQ-L} again, we have
	\begin{align*}
		&~\|(\mathcal{L}_{k+1}*\mathcal{Q}_k-\mathcal{L}_\star)*\mathit{\Sigma}_\star^{1/2}\|_{2,\infty}\\
		\leq&~ (1-\eta)\|\Delta_{\mathcal{L}}*\mathit{\Sigma}_\star^{1/2}\|_{2,\infty} +\eta\|\mathcal{L}_\star*\Delta_{\mathcal{R}}^\top*\mathcal{R}_\natural*(\mathcal{R}_\natural^\top*\mathcal{R}_\natural)^{-1}*\mathit{\Sigma}_\star^{1/2}\|_{2,\infty} +\eta\|\Delta_{\mathcal{S}}*\mathcal{R}_\natural*(\mathcal{R}_\natural^\top*\mathcal{R}_\natural)^{-1}*\mathit{\Sigma}_\star^{1/2}\|_{2,\infty} \cr
		:=&~ \mathfrak{T}_1 + \mathfrak{T}_2 +\mathfrak{T}_3.
	\end{align*}
	
	\begin{enumerate}
		\item \textbf{Bound of $\mathfrak{T}_1$:} $\mathfrak{T}_1\leq(1-\eta)\sqrt{\frac{\mu R}{I_1}}\tau^k\sigma_{\min}(\mathcal{X}_\star)$ directly follows from the assumption of this lemma.

		\item \textbf{Bound of $\mathfrak{T}_2$:} Assumption~\ref{as:incoherence} implies $\|\mathcal{L}_\star*\mathit{\Sigma}_\star^{-1/2}\|_{2,\infty}\leq\sqrt{\frac{\mu R}{I_1}}$, Lemma~\ref{lm:Delta_F_norm} implies $\|\Delta_{\mathcal{R}}*\mathit{\Sigma}_\star^{1/2}\|_2\leq \varepsilon \tau^k  \sigma_{\min}(\mathcal{X}_\star)$, 
		and  Lemma~\ref{lm:L_R_scale_sigma_half} implies $\|\mathcal{R}_\natural*(\mathcal{R}_\natural^\top*\mathcal{R}_\natural)^{-1}*\mathit{\Sigma}_\star^{1/2}\|_2\leq\frac{1}{1-\varepsilon}$. %(64)
		Together, we have
		\begin{align*}
			\mathfrak{T}_2
			&\leq \eta\|\mathcal{L}_\star*\mathit{\Sigma}_\star^{-1/2}\|_{2,\infty} \|\Delta_{\mathcal{R}}*\mathit{\Sigma}_\star^{1/2}\|_2\|\mathcal{R}_\natural*(\mathcal{R}_\natural^\top*\mathcal{R}_\natural)^{-1}*\mathit{\Sigma}_\star^{1/2}\|_2  \leq \eta \frac{\varepsilon}{1-\varepsilon}\sqrt{\frac{\mu R}{I_1}}\tau^k\sigma_{\min}(\mathcal{X}_\star) .
		\end{align*}

		\item \textbf{Bound of $\mathfrak{T}_3$:} By Lemma~\ref{lm:sparity}, $\supp(\Delta_{\mathcal{S}})\subseteq\supp(\mathcal{S}_\star)$, which implies that $\Delta_{\mathcal{S}}$ is an $\alpha$-sparse tensor. 
		%Hence, $\|\Delta_{\mathcal{S}}\|_{2,\infty}\leq\sqrt{\alpha I_3 \sqrt{I_1I_2}}\|\Delta_{\mathcal{S}}\|_\infty$. 
		Thus, by Lemma~\ref{lm:bound of sparse tensor}, Lemma~\ref{lm:sparity} and \ref{lm:X-X_K_inf_norm}, we get
		\begin{align*}
			\mathfrak{T}_3 
			&\leq \eta\|\Delta_{\mathcal{S}}\|_{2,\infty} \|\mathcal{R}_\natural*(\mathcal{R}_\natural^\top*\mathcal{R}_\natural)^{-1}*\mathit{\Sigma}_\star^{1/2}\|_2  \leq \eta\frac{\sqrt{\alpha I_2I_3}}{1-\varepsilon}\|\Delta_{\mathcal{S}}\|_\infty  \leq 2\eta\frac{\sqrt{\alpha I_2I_3}}{1-\varepsilon}\zeta_{k+1}  \leq 6\eta\frac{\sqrt{\alpha I_3}}{1-\varepsilon} \frac{1}{\sqrt{I_1}}\mu R \tau^k \sigma_{\min}(\mathcal{X}_\star).
		\end{align*}
		Putting  three terms together, we obtain
		\begin{align} \label{eq:LQ-L_sigma_half_2_inf}
			&~\sqrt{I_1}\|(\mathcal{L}_{k+1}*\mathcal{Q}_k-\mathcal{L}_\star)*\mathit{\Sigma}_\star^{1/2}\|_{2,\infty}\cr
			%\leq&~ (1-\eta)\|\Delta_{\mathcal{L}}*\mathit{\Sigma}_\star^{1/2}\|_{2,\infty}
			%+\eta\|\mathcal{L}_\star*\Delta_{\mathcal{R}}^\top*\mathcal{R}_\natural*(\mathcal{R}_\natural^\top*\mathcal{R}_\natural)^{-1}*\mathit{\Sigma}_\star^{1/2}\|_{2,\infty}
			%+\eta\|\Delta_{\mathcal{S}}*\mathcal{R}_\natural*(\mathcal{R}_\natural^\top*\mathcal{R}_\natural)^{-1}*\mathit{\Sigma}_\star^{1/2}\|_{2,\infty} \cr
			\leq&~\sqrt{I_1}( \mathfrak{T}_1 + \mathfrak{T}_2 +\mathfrak{T}_3) \cr
			\leq&~ \left(  
			1-\eta
			+ \eta \frac{\varepsilon}{1-\varepsilon}
			+ 6\eta\frac{\sqrt{\alpha \mu RI_3}}{1-\varepsilon} \right)
			\sqrt{\mu R} \tau^k \sigma_{\min}(\mathcal{X}_\star) \cr
			\leq&~ \left(  
			1-\eta\left(1
			- \frac{\varepsilon}{1-\varepsilon}
			- 6\frac{\sqrt{\alpha \mu R I_3}}{1-\varepsilon}\right) \right)
			\sqrt{\mu R} \tau^k \sigma_{\min}(\mathcal{X}_\star).
		\end{align}
		Moreover, we also have 
		\begin{align} \label{eq:LQ-L_sigma_-half_2_inf}
			&~ \sqrt{I_1}\|(\mathcal{L}_{k+1}*\mathcal{Q}_k-\mathcal{L}_\star)*\mathit{\Sigma}_\star^{-1/2}\|_{2,\infty}  
			\leq
			\left(  
			1-\eta\left(1
			- \frac{\varepsilon}{1-\varepsilon}
			- 6\frac{\sqrt{\alpha \mu R I_3}}{1-\varepsilon}\right) \right)
			\sqrt{\mu R} \tau^k.
		\end{align}
		
		\item \textbf{Bound with $\mathcal{Q}_{k+1}$:} 
		% Note that $\mathcal{Q}$'s are the best align matrices under Frobenius norm but this is not necessary true under $\ell_{2,\infty}$ norm. So we must show the bound of $\|(\mathcal{L}_{k+1}*\mathcal{Q}_{k+1}-\mathcal{L}_\star)*\mathit{\Sigma}_\star^{1/2}\|_{2,\infty}$ directly. 
		Note that $\mathcal{Q}_{k+1}$ does exist, according to Lemma~\ref{lm:convergence_dist} and \ref{lm: Q_existence}. 
		Applying \eqref{eq:LQ-L_sigma_half_2_inf}, \eqref{eq:LQ-L_sigma_-half_2_inf} and Lemma~\ref{lm:sigma_half_Q-Q_sigma_half}, we have 
		\begin{align*}
			&~ \|(\mathcal{L}_{k+1}*\mathcal{Q}_{k+1}-\mathcal{L}_\star)*\mathit{\Sigma}_\star^{1/2}\|_{2,\infty} \cr 
			\leq&~ \|(\mathcal{L}_{k+1}*\mathcal{Q}_k-\mathcal{L}_\star)*\mathit{\Sigma}_\star^{1/2}\|_{2,\infty} + \|\mathcal{L}_{k+1}(\mathcal{Q}_{k+1}-\mathcal{Q}_k)*\mathit{\Sigma}_\star^{1/2}\|_{2,\infty} \cr
			=&~ \|(\mathcal{L}_{k+1}*\mathcal{Q}_k-\mathcal{L}_\star)*\mathit{\Sigma}_\star^{1/2}\|_{2,\infty}  + \|\mathcal{L}_{k+1}*\mathcal{Q}_k\mathit{\Sigma}_\star^{-1/2}\mathit{\Sigma}_\star^{1/2}*\mathcal{Q}_k^{-1}*(\mathcal{Q}_{k+1}-\mathcal{Q}_k)*\mathit{\Sigma}_\star^{1/2}\|_{2,\infty}  \cr
			\leq&~ \|(\mathcal{L}_{k+1}*\mathcal{Q}_k-\mathcal{L}_\star)*\mathit{\Sigma}_\star^{1/2}\|_{2,\infty} + \|\mathcal{L}_{k+1}*\mathcal{Q}_k\mathit{\Sigma}_\star^{-1/2}\|_{2,\infty} \|\mathit{\Sigma}_\star^{1/2}*\mathcal{Q}_k^{-1}*(\mathcal{Q}_{k+1}-\mathcal{Q}_k)*\mathit{\Sigma}_\star^{1/2}\|_2 \cr
			\leq&~ \|(\mathcal{L}_{k+1}*\mathcal{Q}_k-\mathcal{L}_\star)*\mathit{\Sigma}_\star^{1/2}\|_{2,\infty}  \cr
			&~  +\left(\|(\mathcal{L}_{k+1}*\mathcal{Q}_k-\mathcal{L}_\star)*\mathit{\Sigma}_\star^{-1/2}\|_{2,\infty} +\|\mathcal{L}_\star*\mathit{\Sigma}_\star^{-1/2} \|_{2,\infty}\right)\|\mathit{\Sigma}_\star^{1/2}*\mathcal{Q}_k^{-1}*(\mathcal{Q}_{k+1}-\mathcal{Q}_k)*\mathit{\Sigma}_\star^{1/2}\|_2 \cr
			\leq&~\Bigg( 1-\eta\left(1
			- \frac{\varepsilon}{1-\varepsilon}
			- 6\frac{\sqrt{\alpha \mu R I_3}}{1-\varepsilon}\right)   + \frac{2\varepsilon}{1-\varepsilon} \left(  2-\eta\left(1
			- \frac{\varepsilon}{1-\varepsilon}
			- 6\frac{\sqrt{\alpha \mu R I_3}}{1-\varepsilon}\right) \right)\Bigg) \sqrt{\frac{\mu R}{I_1}} \tau^k \sigma_{\min}(\mathcal{X}_\star)\cr
			\leq&~\Bigg( 1-\eta\left(1
			- \frac{\varepsilon}{1-\varepsilon}
			- \frac{6\times10^{-2}}{1-\varepsilon}\right)  + \frac{2\varepsilon}{1-\varepsilon} \left(  2-\eta\left(1
			- \frac{\varepsilon}{1-\varepsilon}
			- \frac{6\times10^{-2}}{1-\varepsilon}\right) \right)\Bigg) \sqrt{\frac{\mu R}{I_1}} \tau^k \sigma_{\min}(\mathcal{X}_\star)\cr
			\leq&~ (1-0.6\eta)\sqrt{\frac{\mu R}{I_1}} \tau^k \sigma_{\min}(\mathcal{X}_\star) ,
		\end{align*}
		In the last step, we use $\varepsilon=0.02$, $\alpha\leq\frac{1}{10^4\mu R^{1.5}{I_3}^{1.5}}\leq\frac{1}{10^4\mu RI_3}$, and $\frac{1}{4}\leq\eta\leq \frac{2}{3}$. A similar result can be derived for $\|(\mathcal{R}_{k+1}*\mathcal{Q}_{k+1}^{-\top}-\mathcal{R}_\star)*\mathit{\Sigma}_\star^{1/2}\|_{2,\infty}$. 
		The proof is concluded by substituting $\tau=1-0.6\eta$. 
	\end{enumerate}
\end{proof}

Now we have all the necessary ingredients, we proceed to prove Theorem~\ref{thm:local convergence} (local linear convergence).

\begin{proof}[\textbf{Proof of Theorem~\ref{thm:local convergence}}]
	This proof is done by induction.
	
	\paragraph{Base case} Since $\tau^0=1$, the assumed initial conditions satisfy the base case at $k=0$.
	
	\paragraph{Induction step} At the $k$-th iteration, we assume the following holds: 
	\begin{gather*}
		\distk \leq \frac{\varepsilon}{\sqrt{I_3}}  \tau^k \sigma_{\min}(\mathcal{X}_\star), \cr
		\sqrt{I_1}\|(\mathcal{L}_k*\mathcal{Q}_k-\mathcal{L}_\star)*\mathit{\Sigma}_\star^{1/2}\|_{2,\infty}  \lor \sqrt{I_2}\|(\mathcal{R}_k*\mathcal{Q}_k^{-\top}-\mathcal{R}_\star)*\mathit{\Sigma}_\star^{1/2}\|_{2,\infty} 
		\leq\sqrt{\mu R} \tau^k \sigma_{\min}(\mathcal{X}_\star)
	\end{gather*}
	By applying  Lemma~\ref{lm:X-X_K_inf_norm}, Lemma~\ref{lm:convergence_dist} and \ref{lm:convergence_incoher}, we obtain
	\begin{gather*}
		\|\mathcal{X}_\star-\mathcal{X}_k\|_\infty \leq \frac{3}{\sqrt{I_1I_2}} \mu R \tau^k \sigma_{\min}(\mathcal{X}_\star), \cr
		\distkplusone \leq \frac{\varepsilon}{\sqrt{I_3}} \tau^{k+1} \sigma_{\min}(\mathcal{X}_\star), \cr
		\sqrt{I_1}\|(\mathcal{L}_{k+1}*\mathcal{Q}_{k+1}-\mathcal{L}_\star)*\mathit{\Sigma}_\star^{1/2}\|_{2,\infty}  \lor \sqrt{I_2} \|(\mathcal{R}_{k+1}*\mathcal{Q}_{k+1}^{-\top}-\mathcal{R}_\star)*\mathit{\Sigma}_\star^{1/2}\|_{2,\infty}  
		\leq\sqrt{\mu R} \tau^{k+1} \sigma_{\min}(\mathcal{X}_\star)
	\end{gather*}
	This completes the induction step and finishes the proof.
\end{proof}

\subsection{Proof of Theorem~\ref{thm:initial} (guaranteed initialization) } \label{sec:guaranteed initialization}
Now we show the demonstrate of the initialization step in Algorithm~\ref{Alg: LGRTPCA1} satisfy the initial conditions required by Theorem~\ref{thm:local convergence}.
\begin{proof}[\textbf{Proof of Theorem~\ref{thm:initial}}]
	Firstly,  using the facts that  $\|\mathcal{A}* \mathcal{B}^{\To}\|_{\infty}\le \|\mathcal{A}\|_{2,\infty}\|\mathcal{B}\|_{2,\infty}$ and $\|\mathcal{A}* \mathcal{B}\|_{2,\infty}\le \|\mathcal{A}\|_{2}\|\mathcal{B}\|_{2,\infty}$  and by Assumption~\ref{as:incoherence}, we obtain
	\begin{align*}
		\|\mathcal{X}_\star\|_\infty\leq\|\mathcal{U}_\star\|_{2,\infty}\|\mathit{\Sigma}_\star\|_2\|\mathcal{V}_\star\|_{2,\infty}\leq\frac{\mu R}{\sqrt{I_1I_2}}\sigma_1(\mathcal{X}_\star).
	\end{align*}
	By invoking Lemma~\ref{lm:sparity} with $\mathcal{X}_{-1}=\bm{0}$, and $\left\|\mathcal{X}_{\star}\right\|_{\infty} \leq \zeta_0 \leq 2\left\|\mathcal{X}_{\star}\right\|_{\infty}$, we have 
	\begin{align} \label{eq:init_S_inf}
		\|\mathcal{S}_\star-\mathcal{S}_0\|_\infty\leq 3\frac{\mu R}{\sqrt{I_1I_2}}\sigma_1(\mathcal{X}_\star) \ \textnormal{and}\ \supp(\mathcal{S}_0)\subseteq\supp(\mathcal{S}_\star).
	\end{align}
	This implies $\mathcal{S}_\star-\mathcal{S}_0$ is an $\alpha$-sparse tensor. By applying Lemma~\ref{lm:bound of sparse tensor}, we have
	\begin{align*}
		&\|\mathcal{S}_\star-\mathcal{S}_0\|_2\leq \alpha I_3 \sqrt{I_1I_2}\|\mathcal{S}_\star-\mathcal{S}_0\|_\infty 
		\leq  3 \alpha \mu R I_3  \sigma_1(\mathcal{X}_\star)=3 \alpha \mu R I_3 \kappa \sigma_{\min}(\mathcal{X}_\star).
	\end{align*}
	According to the initialization step in Algorithm~\ref{Alg: LGRTPCA1}, $\mathcal{X}_0=\mathcal{L}_0\mathcal{R}_0^\top$ is the top-$R$  tensor approximation (see Definition~\ref{def: top-R}) of $\mathcal{Y}-\mathcal{S}_0$.
	Then we have
	\begin{align}\label{lm: X_norm2}
		\|\mathcal{X}_\star-\mathcal{X}_0\|_2
		&\leq \|\mathcal{X}_\star-(\mathcal{Y}-\mathcal{S}_0)\|_2 + \|(\mathcal{Y}-\mathcal{S}_0) -\mathcal{X}_0 \|_2 \cr
		&\leq 2\|\mathcal{X}_\star-(\mathcal{Y}-\mathcal{S}_0)\|_2 = 2\|\mathcal{S}_\star-\mathcal{S}_0\|_2 \leq 6 \alpha \mu R I_3\kappa \sigma_{\min}(\mathcal{X}_\star),
	\end{align}
	where the equality follows from the definition $\mathcal{Y}=\mathcal{X}_\star+\mathcal{S}_\star$. 
	
	In addition, by [Lemma~24,\cite{tong2021accelerating}],  we have
	\begin{align*}
		\dist(\widehat{\mathbf{L}}_0^{(i_3)},\widehat{\mathbf{R}}_0^{(i_3)};\widehat{\mathbf{L}}_\star^{(i_3)},\widehat{\mathbf{R}}_\star^{(i_3)}) &\leq \sqrt{\sqrt{2}+1}\|\widehat{\mathbf{X}}_\star^{(i_3)}-\widehat{\mathbf{X}}_0^{(i_3)}\|_{\fro},
	\end{align*}
	where  $i_3 = 1, \cdots, I_3$. When transformed into the time domain, we have
	\begin{align*}
		\distzero &\leq \sqrt{\sqrt{2}+1}\|\mathcal{X}_\star-\mathcal{X}_0\|_{\fro},
	\end{align*}
	Then we  obtain
	\begin{align*}
		\distzero &\leq \sqrt{\sqrt{2}+1}\|\mathcal{X}_\star-\mathcal{X}_0\|_{\fro} \leq \sqrt{(\sqrt{2}+1)2R}\|\mathcal{X}_\star-\mathcal{X}_0\|_2 \leq 14  \alpha \mu R^{1.5} I_3 \kappa \sigma_{\min}(\mathcal{X}_\star),
	\end{align*}
	where we use the fact that $\mathcal{X}_\star-\mathcal{X}_0$ has at most tubal rank-$2R$, along with  \eqref{lm: X_norm2}. Given that $\varepsilon=14 c_0$ and $\alpha\leq\frac{c_0}{\mu R^{1.5} I_3^{1.5} \kappa}$, we can prove the first claim
	\begin{equation} \label{eq:init_dist}
		\distzero \leq \frac{14c_0}{\sqrt{I_3}}\sigma_{\min}(\mathcal{X}_\star).
	\end{equation}
	
	Let $\varepsilon:=14c_0$. Now, we proceed to prove the second claim:
	\begin{align*}
		% \sqrt{I_1}\|(\mathcal{L}_0\mathcal{Q}_0-\mathcal{L}_\star)*\mathit{\Sigma}_\star^{1/2}\|_{2,\infty}  \lor \sqrt{I_2} \|(\mathcal{R}_0\mathcal{Q}_0^{-\top}-\mathcal{R}_\star)*\mathit{\Sigma}_\star^{1/2}\|_{2,\infty}\cr
		%     \leq  \sqrt{\mu R} \sigma_{\min}(\mathcal{X}_\star)\cr
		\sqrt{I_1}\|\Delta_{\mathcal{L}}*\mathit{\Sigma}_\star^{1/2}\|_{2,\infty} \lor \sqrt{I_2}\|\Delta_{\mathcal{R}}*\mathit{\Sigma}_\star^{1/2}\|_{2,\infty} \leq \sqrt{\mu R} \sigma_{\min}(\mathcal{X}_\star),
	\end{align*}
	where $\Delta_{\mathcal{L}}:=\mathcal{L}_0*\mathcal{Q}_0-\mathcal{L}_\star$ and $\Delta_{\mathcal{R}}:=\mathcal{R}_0*\mathcal{Q}_0^{-\top}-\mathcal{R}_\star$. For ease of notation, we also denote $\mathcal{L}_\natural=\mathcal{L}_0*\mathcal{Q}_0$, $\mathcal{R}_\natural=\mathcal{R}_0*\mathcal{Q}_0^{-\top}$, and $\Delta_{\mathcal{S}}=\mathcal{S}_0-\mathcal{S}_\star$ in the rest of this proof. 
	
	We will first focus on bounding $\|\Delta_{\mathcal{L}}\Sigma_\star^{1/2}\|_{2,\infty}$, and $\|\Delta_{\mathcal{R}}\Sigma_\star^{1/2}\|_{2,\infty}$ can be handled in a similar manner.
	According to the initialization step in Algorithm~\ref{Alg: LGRTPCA1}, \textit{i.e.}, $\mathcal{U}_0*\mathit{\Sigma}_0*\mathcal{V}_0^{\top}=\mathcal{D}_R(\mathcal{Y}-\mathcal{S}_0)=\mathcal{D}_R(\mathcal{X}_\star-\Delta_{\mathcal{S}})$, we have 
	\begin{align*}
		\mathcal{L}_0 = \mathcal{U}_0*\mathit{\Sigma}_0^{1/2} 
		&= (\mathcal{X}_\star-\Delta_{\mathcal{S}})*\mathcal{V}_0*\mathit{\Sigma}_0^{-1/2} = (\mathcal{X}_\star-\Delta_{\mathcal{S}})*\mathcal{R}_0*\mathit{\Sigma}_0^{-1} = (\mathcal{X}_\star-\Delta_{\mathcal{S}})*\mathcal{R}_0*(\mathcal{R}_0^\top*\mathcal{R}_0)^{-1}.
	\end{align*}
	Multiplying both sides by $\mathcal{Q}_0*\Sigma_\star^{1/2}$, we have
	\begin{align*}
		\mathcal{L}_\natural*\mathit{\Sigma}_\star^{1/2} = \mathcal{L}_0*\mathcal{Q}_0*\mathit{\Sigma}_\star^{1/2} 
		&= (\mathcal{X}_\star-\Delta_{\mathcal{S}})*\mathcal{R}_0*(\mathcal{R}_0^\top*\mathcal{R}_0)^{-1}* \mathcal{Q}_0*\mathit{\Sigma}_\star^{1/2} = (\mathcal{X}_\star-\Delta_{\mathcal{S}})*\mathcal{R}_\natural*(\mathcal{R}_\natural^\top*\mathcal{R}_\natural)^{-1} \mathit{\Sigma}_\star^{1/2}.
	\end{align*}
	Subtracting $\mathcal{X}_\star*\mathcal{R}_\natural*(\mathcal{R}_\natural^\top*\mathcal{R}_\natural)^{-1}*\mathit{\Sigma}_\star^{1/2}$  from both sides and using the fact that $\mathcal{L}_\star*\mathit{\Sigma}_\star^{1/2} =\mathcal{L}_\star*\mathcal{R}_\natural^\top*\mathcal{R}_\natural*(\mathcal{R}_\natural^\top*\mathcal{R}_\natural)^{-1}*\mathit{\Sigma}_{\star}^{1/2}$, we have
	\begin{align*}
		% &\mathcal{L}_\natural*\mathit{\Sigma}_\star^{1/2} - \mathcal{L}_\star\mathcal{R}_\star^\top*\mathcal{R}_\natural*(\mathcal{R}_\natural^\top*\mathcal{R}_\natural)^{-1}*\mathit{\Sigma}_\star^{1/2}\cr
		% = &(\mathcal{X}_\star-\Delta_{\mathcal{S}})\mathcal{R}_\natural*(\mathcal{R}_\natural^\top*\mathcal{R}_\natural)^{-1} \mathit{\Sigma}_\star^{1/2} -\mathcal{X}_\star\mathcal{R}_\natural*(\mathcal{R}_\natural^\top*\mathcal{R}_\natural)^{-1}*\mathit{\Sigma}_\star^{1/2} \cr
		&\Delta_{\mathcal{L}}*\mathit{\Sigma}_\star^{1/2}+\mathcal{L}_\star*\Delta_{\mathcal{R}}^\top*  \mathcal{R}_\natural*(\mathcal{R}_\natural^\top*\mathcal{R}_\natural)^{-1}*\mathit{\Sigma}_\star^{1/2}
		= -\Delta_{\mathcal{S}}*\mathcal{R}_\natural*(\mathcal{R}_\natural^\top*\mathcal{R}_\natural)^{-1}*\mathit{\Sigma}_\star^{1/2} ,
	\end{align*}
	% where the left operand of last step uses the fact  $\mathcal{L}_\star*\mathit{\Sigma}_\star^{1/2} =\mathcal{L}_\star\mathcal{R}_\natural^\top*\mathcal{R}_\natural*(\mathcal{R}_\natural^\top*\mathcal{R}_\natural)^{-1}*\mathit{\Sigma}_{\star}^{1/2}$.
	Thus, 
	\begin{align*}
		\| \Delta_{\mathcal{L}}*\mathit{\Sigma}_\star^{1/2} \|_{2,\infty}
		\leq& \|\mathcal{L}_\star*\Delta_{\mathcal{R}}^\top  \mathcal{R}_\natural*(\mathcal{R}_\natural^\top*\mathcal{R}_\natural)^{-1} *\mathit{\Sigma}_\star^{1/2}\|_{2,\infty}
		+ \|\Delta_{\mathcal{S}}*\mathcal{R}_\natural*(\mathcal{R}_\natural^\top*\mathcal{R}_\natural)^{-1} *\mathit{\Sigma}_\star^{1/2}\|_{2,\infty} 
		:= \mathfrak{J}_1 + \mathfrak{J}_2
	\end{align*}
	
	\paragraph{\textbf{Bound of $\mathfrak{J}_1$}} By Assumption~\ref{as:incoherence},  Lemma~\ref{lm:Delta_F_norm} and  Lemma~\ref{lm:L_R_scale_sigma_half}, we get
	\begin{align*}
		\mathfrak{J}_1 &\leq \|\mathcal{L}_\star*\mathit{\Sigma}_\star^{-1/2}\|_{2,\infty} \|\Delta_{\mathcal{R}}*\mathit{\Sigma}_\star^{1/2}\|_2  \|\mathcal{R}_\natural*(\mathcal{R}_\natural^\top*\mathcal{R}_\natural)^{-1} \mathit{\Sigma}_\star^{1/2}\|_2 \leq \sqrt{\frac{\mu R}{I_1}}\frac{\varepsilon}{1-\varepsilon}\sigma_{\min}(\mathcal{X}_\star)
	\end{align*}
	% where implies $\|\Delta_{\mathcal{R}}*\mathit{\Sigma}_\star^{1/2}\|_2 \leq \varepsilon \sigma_{\min}(\mathcal{X}_\star)$, and Lemma~\ref{lm:L_R_scale_sigma_half} implies $\|\mathcal{R}_\natural*(\mathcal{R}_\natural^\top*\mathcal{R}_\natural)^{-1} \mathit{\Sigma}_\star^{1/2}\|_2 \leq\frac{1}{1-\varepsilon}$, given \eqref{eq:init_dist} holds.
	
	\paragraph{\textbf{Bound of $\mathfrak{J}_2$}} \eqref{eq:init_S_inf} implies $\Delta_{\mathcal{S}}$ is an $\alpha$-sparse tensor. Moreover, by \eqref{eq:init_dist}, Lemma~\ref{lm:bound of sparse tensor} and \ref{lm:L_R_scale_sigma_half}, we have
	
	% \begin{align*}
		%     \mathfrak{J}_2 &\leq \|\Delta_{\mathcal{S}}\|_{2,\infty}\|\mathcal{R}_\natural*(\mathcal{R}_\natural^\top*\mathcal{R}_\natural)^{-1} \mathit{\Sigma}_\star^{1/2}\|_2 \cr
		%     &\leq \sqrt{\alpha n}\|\Delta_{\mathcal{S}}\|_\infty \|\mathcal{R}_\natural*(\mathcal{R}_\natural^\top*\mathcal{R}_\natural)^{-1} \mathit{\Sigma}_\star^{1/2}\|_2 \cr
		%     &\leq \sqrt{\alpha n} \frac{2\mu R}{n(1-\varepsilon)} \sigma_1(\mathcal{X}_\star)
		% \end{align*}
	\begin{align*}
		\mathfrak{J}_2 &\leq \|\Delta_{\mathcal{S}}\|_{1,\infty}\|\mathcal{R}_\natural*\mathit{\Sigma}_\star^{-1/2}\|_{2,\infty}\|\mathit{\Sigma}_\star^{1/2}(\mathcal{R}_\natural^\top*\mathcal{R}_\natural)^{-1} \mathit{\Sigma}_\star^{1/2}\|_2 \cr
		&\leq \alpha I_2 I_3\|\Delta_{\mathcal{S}}\|_\infty\|\mathcal{R}_\natural*\mathit{\Sigma}_\star^{-1/2}\|_{2,\infty}\|\mathcal{R}_\natural*(\mathcal{R}_\natural^\top*\mathcal{R}_\natural)^{-1} \mathit{\Sigma}_\star^{1/2}\|_2^2 \cr
		&\leq \alpha I_2 I_3 \frac{3 \mu R}{ \sqrt{I_1I_2}}\sigma_1(\mathcal{X}_\star)\frac{1}{(1-\varepsilon)^2}\|\mathcal{R}_\natural*\mathit{\Sigma}_\star^{-1/2}\|_{2,\infty} \cr
		&\leq \frac{3 \alpha \mu R I_3\kappa}{(1-\varepsilon)^2} \sqrt{\frac{I_2}{I_1}}\left(\sqrt{\frac{\mu R}{I_2}}+\|\Delta_{\mathcal{R}}*\mathit{\Sigma}_\star^{-1/2}\|_{2,\infty}\right) \sigma_{\min}(\mathcal{X}_\star)
	\end{align*}
	where we use the fact that $\|\mathcal{A}*\mathcal{B}\|_{2,\infty}\leq\|\mathcal{A}\|_{1,\infty}\|\mathcal{B}\|_{2,\infty}$  in the first step. 
	Note that $\|\Delta_{\mathcal{R}}*\mathit{\Sigma}_\star^{-1/2}\|_{2,\infty}\leq \frac{\|\Delta_{\mathcal{R}}*\mathit{\Sigma}_\star^{1/2}\|_{2,\infty}}{\sigma_{\min}(\mathcal{X}_\star)}$, we have 
	\begin{align*}
		\| \Delta_{\mathcal{L}}*\mathit{\Sigma}_\star^{1/2} \|_{2,\infty}
		\leq &\left(\frac{\varepsilon}{1-\varepsilon}+\frac{3 \alpha \mu R I_3\kappa}{(1-\varepsilon)^2}\right) \sqrt{\frac{\mu R}{I_1}} \sigma_{\min}(\mathcal{X}_\star) + \frac{3 \alpha \mu R  I_3\kappa}{(1-\varepsilon)^2} \sqrt{\frac{I_2}{I_1}}\|\Delta_{\mathcal{R}}*\mathit{\Sigma}_\star^{1/2}\|_{2,\infty} .
	\end{align*}
	Similarly, one can see
	\begin{align*}
		\| \Delta_{\mathcal{R}}*\mathit{\Sigma}_\star^{1/2} \|_{2,\infty}
		\leq& \left(\frac{\varepsilon}{1-\varepsilon}+\frac{3 \alpha \mu R I_3\kappa}{(1-\varepsilon)^2}\right) \sqrt{\frac{\mu R }{I_2}} \sigma_{\min}(\mathcal{X}_\star) + \frac{3 \alpha \mu R I_3\kappa}{(1-\varepsilon)^2}\sqrt{\frac{I_1}{I_2}}\|\Delta_{\mathcal{L}}*\mathit{\Sigma}_\star^{1/2}\|_{2,\infty}  .
	\end{align*}
	Therefore, substituting $\varepsilon=14 c_0$, we have
	\begin{align*}
		&\quad~\sqrt{I_1}\|\Delta_{\mathcal{L}}*\mathit{\Sigma}_\star^{1/2} \|_{2,\infty} \lor \sqrt{I_2}\| \Delta_{\mathcal{R}}*\mathit{\Sigma}_\star^{1/2} \|_{2,\infty}  \cr
		&\leq \frac{(1-\varepsilon)^2}{(1-\varepsilon)^2-3\alpha \mu R I_3 \kappa}\left(\frac{\varepsilon}{1-\varepsilon}+\frac{3 \alpha \mu R I_3 \kappa}{(1-\varepsilon)^2}\right) \sqrt{\mu R } \sigma_{\min}(\mathcal{X}_\star) \cr
		&\leq \frac{(1-14c_0)^2}{(1-14c_0)^2-3c_0}\left(\frac{14c_0}{1-14c_0}+\frac{3c_0}{(1-14c_0)^2}\right) \sqrt{\mu R} \sigma_{\min}(\mathcal{X}_\star)  \cr
		&\leq \sqrt{\mu R} \sigma_{\min}(\mathcal{X}_\star),
	\end{align*}
	as long as $c_0\leq \frac{1}{38}$. 
	%given $\varepsilon=0.02$ and $\alpha\leq\frac{1}{500\mu R^{1.5} \kappa}$.
	
	This finishes the proof.
\end{proof}

\end{document}